\documentclass[12pt]{article} %***
\usepackage[utf8]{inputenc}

\usepackage[sectionbib]{natbib}
\usepackage{array,epsfig,fancyheadings,rotating}
\usepackage[]{hyperref}  %<----modified by Ivan
\usepackage{sectsty, secdot}
\sectionfont{\fontsize{12}{14pt plus.8pt minus .6pt}\selectfont}
\renewcommand{\theequation}{\thesection\arabic{equation}}
\subsectionfont{\fontsize{12}{14pt plus.8pt minus .6pt}\selectfont}
\usepackage[margin=1in]{geometry}
\usepackage{setspace}

\usepackage{amsmath}
\usepackage{amssymb}
\usepackage{amsfonts}
\usepackage{multirow}
\usepackage{amsthm}
\usepackage{color}

\usepackage{mycommands}

\usepackage{afterpage}

\begin{document}

%%%%%%%%%%%%%%%%%%%%%%%%%%%%%%%%%%%%%%%%%%%%%%%%%%%%%%%%%%%%%%%%%%%%%%%%%%%%%%%%%%%%%%%%%%%%%%%%%%%%%%%%%%%%%%%%%%%%%%%%%%%%
%%%%%%%%%%%%%%%%%%%%%%%%%%%%%%%%%%%%%%%%%%%%%%%%%%%%%%%%%%%%%%%%%%%%%%%%%%%%%%%%%%%%%%%%%%%%%%%%%%%%%%%%%%%%%%%%%%%%%%%%%%%%

\renewcommand{\baselinestretch}{2}

\markright{ \hbox{\footnotesize\rm
%{\footnotesize\bf 24} (201?), 000-000
}\hfill\\[-13pt]
\hbox{\footnotesize\rm
%\href{http://dx.doi.org/10.5705/ss.20??.???}{doi:http://dx.doi.org/10.5705/ss.20??.???}
}\hfill }

\markboth{\hfill{\footnotesize\rm FIRSTNAME1 LASTNAME1 AND FIRSTNAME2 LASTNAME2} \hfill}
{\hfill {\footnotesize\rm FILL IN A SHORT RUNNING TITLE} \hfill}

\renewcommand{\thefootnote}{}
$\ $\par

%%%%%%%%%%%%%%%%%%%%%%%%%%%%%%%%%%%%%%%%%%%%%%%%%%%%%%%%%%%%%%%%%%%%%%%%%%%%%%%%%%%%%%%%%%%%%%%%%%%%%%%%%%%%%%%%%%%%%%%%%%%%

\fontsize{12}{14pt plus.8pt minus .6pt}\selectfont \vspace{0.8pc}
\centerline{\large\bf Identifiability of Hierarchical Latent Attribute Models }
%\vspace{2pt} 
%\centerline{\large\bf HERE IF A SECOND LINE IS NEEDED}
\vspace{.4cm} 
\centerline{Yuqi Gu and Gongjun Xu} 
\vspace{.4cm} 
\centerline{\it Columbia University and University of Michigan}
 \vspace{.55cm} \fontsize{9}{11.5pt plus.8pt minus.6pt}\selectfont

%%%%%%%%%%%%%%%%%%%%%%%%%%%%%%%%%%%%%%%%%%%%%%%%%%%%%%%%%%%%%%%%%%%%%%%%%%%%%%%%%%%%%%%%%%%%%%%%%%%%%%%%%%%%%%%%%%%%%%%%%%%%

\begin{quotation}
\noindent {\it Abstract:}
%{\bf Contents of the Abstract.}
Hierarchical Latent Attribute Models (HLAMs) are a family of discrete latent variable models that are attracting increasing attention in educational, psychological, and behavioral sciences. 
The key ingredients of an HLAM include a binary structural matrix and a directed acyclic graph specifying hierarchical constraints on the configurations of latent attributes. 
These components encode practitioners' design information and carry important scientific meanings.
Despite the popularity of HLAMs, the fundamental identifiability issue remains unaddressed.
The existence of the attribute hierarchy graph leads to degenerate parameter space, and the potentially unknown structural matrix further complicates the identifiability problem.
This paper addresses this issue of identifying the latent structure and model parameters underlying an HLAM. 
We develop sufficient and necessary identifiability conditions. 
These results directly and sharply characterize the different impacts on identifiability cast by different attribute types in the graph.
The proposed conditions not only provide insights into diagnostic test designs under the attribute hierarchy, but also serve as tools to assess the validity of an estimated HLAM.

\vspace{9pt}
\end{quotation}\par

\def\thefigure{\arabic{figure}}
\def\thetable{\arabic{table}}

\renewcommand{\theequation}{\thesection.\arabic{equation}}

\fontsize{12}{14pt plus.8pt minus .6pt}\selectfont

\afterpage{\cfoot{\thepage}}
\section{Introduction}
Latent attribute models are a family of discrete latent variable models popular in multiple scientific disciplines, including cognitive diagnosis in educational assessments \citep{junker2001cognitive,davier2008general,HensonTemplin09, rupp2010diagnostic,dela2011,wang2018tracking}, psychiatric diagnosis of mental disorders \citep{templin2006measurement,dela2018}, and epidemiological and medical measurement studies \citep{wu2017nested,o2019causes}.  
Based on subjects' responses (often binary) to a set of items, a latent attribute model enables fine-grained inference on subjects' statuses of an underlying set of latent traits; this further allows for clustering the population into interpretable subgroups based on the inferred attribute patterns. 
In a latent attribute model, each attribute is often assumed binary and carries specific scientific meaning. For example, in  an educational assessment, the observed responses are students' correct or wrong answers to a set of test items, and the latent attributes indicate students' binary states of mastery or deficiency of certain skills measured by the assessment \citep{junker2001cognitive,davier2008general,rupp2010diagnostic}.
On top of this, the dependence among the latent attributes can be further modeled to incorporate practitioners' prior knowledge. 
A particularly popular and powerful way of modeling attribute dependence in educational and psychological studies is to enforce hard constraints on the hierarchical configurations of the attributes. 
Specifically, educational experts often postulate some prerequisite relations exist among the binary skill attributes, such that mastering some skills serve as a prerequisite for mastering some others \citep{leighton2004attribute}.
Such a family of \textit{Hierarchical Latent Attribute Models} (HLAMs) are attacting increasing attention in {cognitive diagnostic applications} in recent years; see \cite{leighton2004attribute,Gierl,templin2014hierarchical,wang2020hier}. 
Despite the popularity, the fundamental identifiability issue of HLAMs remains unaddressed.
This paper fills this gap and provides the identifiability theory for HLAMs.

\textcolor{black}{HLAMs have close connections with many other popular statistical and machine learning models. Since each possible configuration of the discrete attributes represents a pattern defining a latent subpopulation, the HLAM can be viewed as a structured {mixture model} \citep{mclachlan2004finite} and gives rises to {model-based clustering} \citep{fraley2002} of multivariate categorical data. 
%This formulation also relates HLAMs to \textit{model-based clustering} \citep{fraley2002} of multivariate categorical data.
}
HLAMs are related to several multivariate discrete latent variable models in the machine learning literature, including latent tree graphical models \citep{choi2011tree,mourad2013tree}, restricted Boltzmann machines \citep{hinton2002training, larochelle2008rbm} and restricted Boltzmann forests (RBForests) \citep{rbf2010},   latent feature models  \citep{ghahramani2006ibp}, but with the following two key differences.
First, the observed variables are assumed to have certain structured dependence on the latent attributes. This dependence is summarized by a structural matrix, the so-called $\QQ$-matrix \citep{Tatsuoka1990}, to encode scientific interpretations. 
The second key feature is that HLAMs incorporate the hierarchical structure among the latent attributes. 
For instance, in educational cognitive diagnosis, the possession of certain skill attributes are often assumed to be the prerequisite for possessing some others \citep{leighton2004attribute, templin2014hierarchical}. 
Such hierarchical structures differ from the latent tree models in that, the latter use a probabilistic graphical model to model the hierarchical tree structure among latent variables, while in an HLAM the hierarchy is a directed acyclic graph (DAG) encoding hard constraints on  allowable configurations of latent attributes. This type of hierarchical constraints in HLAMs have a similar flavor as those of RBForests proposed in \cite{rbf2010}, though the DAG-structure constraints in an HLAM are more flexible than a forest-structure (i.e., group of trees) one in an RBForest (see Example \ref{exp-rbf}).

% One major issue in the applications of HLAMs is that, the attribute hierarchy  and the structural $\QQ$-matrix often suffer from potential misspecification by domain experts in confirmatory-type applications, or even entirely unknown in exploratory-type applications. 
 {The real-world applications of HLAMs are challenged by the identifiability issues of the attribute hierarchy, the structural $\QQ$-matrix, and other model parameters.
First, in many applications, the attribute hierarchy and the structural $\QQ$-matrix are specified by the domain experts based on their understanding of the diagnostic tests. Such specification could be subjective and may not reflect the underlying truth. 
Second, the attribute hierarchy and the $\QQ$-matrix may even be entirely unknown in exploratory data analysis, where researchers hope to identify and estimate these quantities directly from the observed data.}
In both of the above situations, a fundamental yet open question is whether and when the attribute hierarchy and even the structural $\QQ$-matrix are identifiable. 
The identifiability of HLAMs has a close connection to the uniqueness of tensor decompositions, as the probability distribution of an HLAM can be written as a mixture of  highly constrained higher-order tensors. Particularly, HLAMs can be viewed as a special family of restricted latent class models, with the  $\QQ$-matrix imposing constraints on the model parameters. However, related works on the identifiability of latent class models and uniqueness of tensor decompositions  \cite[e.g.][]{allman2009,anand2014} cannot be directly applied to HLAMs due to the constraints induced by the  $\QQ$-matrix.

To tackle identifiability under such structural constraints,
some recent works \citep{xu2017,xu2018,id-dina,fang2019,partial,slam,chen2020scdm}  proposed identifiability conditions for latent attribute models. However, most of them 
\citep{xu2017,xu2018,id-dina,fang2019,chen2020scdm}
considered scenarios without any attribute hierarchy; \cite{partial} assumed both the true $\QQ$-matrix and true configurations of attribute patterns are known and fixed; 
\cite{slam} considered the problem of learning the set of truly existing attribute patterns but assumed the $\QQ$-matrix is correctly specified beforehand. 
All these previous works did not directly take into account the hierarchical graphical structure of the attribute hierarchy, therefore their results can not provide explicit and sharp identifiability conditions for an HLAM.
On the other hand, in the cognitive diagnostic modeling literature, researchers \citep{kc2019,cai2018,heller2019} recently studied the ``completeness'' of the $\QQ$-matrix, a relevant concept to be revisited in Section \ref{sec-main}, under  attribute hierarchy. 
But these results can not ensure identifying uniquely the model parameters that determine the probabilistic HLAM.
In summary, establishing identifiability without assuming any knowledge of the  $\QQ$-matrix and the attribute hierarchy still remains unaddressed in the literature, and it is indeed a technically challenging task.

This paper addresses this identifiability question for popular HLAMs under an arbitrary attribute hierarchy. 
We develop explicit sufficient conditions for identifying the attribute hierarchy, the $\QQ$-matrix, and all the model parameters in an HLAM. 
%Due to the duality of the model formulation of DINA and DINO \cite{chen2015statistical}, our conclusion directly applies to the DINO model.
%The proposed conditions explicitly characterize the interplay between the $\QQ$-matrix and the attribute hierarchy graph, and how they affect identifiability. 
These sufficient conditions become also necessary when the latent pattern space is saturated with no  hierarchy. 
While for cases where there is a nonempty hierarchy, we discuss the necessity of these individual conditions and relax them in several nontrivial and interesting ways. 
% -- The following sentence is also added in the revision, I ma not sure if this needs to be added -- %
 {Based on these and going further, we then establish the fully general necessary and sufficient identifiability conditions for the attribute hierarchy and all the model parameters under a fixed $\QQ$-matrix.}
% -- The following sentence is also added in the revision -- %
Our results in this regard sharply characterize the different roles played by different types of attributes in the attribute hierarchy graph.
The theoretical developments can be used to assess the validity of an estimated HLAM  obtained from any estimation method.
They also provide insights into designing useful diagnostic tests under attribute hierarchy with minimal restrictions.

The rest of the paper is organized as follows. In Section \ref{sec-setup}, we introduce the model setup of the HLAMs. In Section \ref{sec-main}, we present sufficient conditions on identifiability of $\QQ$, attribute hierarchy, and model parameters. 
%We provide various necessary conditions in Section \ref{sec-nece} to characterize the fine boundary between necessity and sufficiency of identifiability conditions for different types of attributes.
 {In Section \ref{sec-nece}, to thoroughly investigate how to close the gap between the necessity and sufficiency of the identifiability conditions, we focus on the case where $\QQ$ is fixed and derive the fully general necessary and sufficient conditions for identifying the attribute hierarchy and model parameters.
In Section \ref{sec-multiple}, we provide an extension of the identifiability result %from the two-parameter HLAMs to another subfamily
 to other types of HLAMs that have potentially more parameters than that studied in Sections \ref{sec-main}-\ref{sec-nece}.}
We give a brief discussion in Section \ref{sec-disc}. 
All the technical proofs are presented in the Supplementary Material.

\section{Model Setup and Examples}\label{sec-setup}

%\subsection{Setup of Hierarchical Latent Attribute Models}\label{sec-sub-setup}
This section introduces the model setup of HLAMs. 
We first introduce some notation.  
For an integer $m$, denote $[m]=\{1,2,\ldots,m\}$.
%For two vectors $\bo a=(a_1,\ldots,a_m)$ and $\bo b=(b_1,\ldots,b_m)$ of the same length, denote $\bo a\succeq \bo b$ if $a_i\geq b_i$ for all $i\in[m]$, and denote $\bo a\nsucceq\bo b$ otherwise. Define operations ``$\preceq$'' and ``$\npreceq$'' similarly.
For a set $\mca$, denote its cardinality by $|\mca|$.
Denote the $K\times K$ identity matrix by $I_K$ and
the $K$-dimensional  all-one and all-zero vectors by  $\one_K$ and $\zero_K$, respectively.
%Denote the binary indicator function by $I(\bo\cdot)$, and the sigmoid function by $\sigma(x)=1/(1+e^{-x})$.

An HLAM consists of two types of subject-specific  binary variables, the observed responses $\rr=(r_1,\ldots,r_J)\in\{0,1\}^J$ to   $J$ \textit{items}; and the latent attribute pattern $\aaa=(\alpha_1,\ldots,\alpha_K)\in\{0,1\}^K$, with  $\alpha_k$ indicating  the mastery or deficiency of the $k$th attribute.
%%%%% -- R1 part begins here
In this work, $K$ is assumed known and fixed. This assumption is well suited for the motivating applications in cognitive diagnosis, where the number and also the real-world meanings of the latent attributes are usually known in the context of the application, and it is of interest to identify and learn other quantities from data.
%%%%% -- R1 part ends here
Next, we first describe the distribution of the latent attributes.
Attribute $k$ is said to be the prerequisite of attribute $\ell$ and denoted by $k\to \ell$, if any pattern $\aaa$ with $\alpha_k=0$ and $\alpha_\ell=1$ is ``forbidden'' to exist. This is a common assumption in applications such as cognitive diagnosis {to model subjects' learning process} \citep{leighton2004attribute, templin2014hierarchical}. 
A subject's latent pattern $\bo a$ is assumed to follow a categorical distribution of population proportion parameters $\pp=(p_{\aaa},\,\aaa\in\{0,1\}^K)$, with $p_{\aaa}\geq 0$ and $\sum_{\aaa} p_{\aaa}=1$. 
In particular, any pattern $\aaa$ not respecting the hierarchy is deemed impossible to exist with population proportion $p_{\aaa}=0$.
An attribute hierarchy is a set of prerequisite relations among the $K$ attributes, which we denote by 
{$$\mce=\{k\to\ell:\, \text{attribute }k\text{ is a prerequisite for }\ell\}.$$
{Generally, an attribute hierarchy $\mce$ implies a directed acyclic graph among the $K$ attributes with no directed cycles; this graph constrains which attribute patterns are permissible or forbidden.}
Specifically, any $\mce$ would induce a set of allowable configurations of attribute patterns out of $\{0,1\}^K$, which we denote by $\mca(\mce)$, or simply $\mca$ when it causes no confusion. 
For an arbitrary $\mce$, the all-zero and all-one attribute patterns $\zero_K$ and $\one_K$ always belong to the induced $\mca$.
This is because any prerequisite relation among attributes would not rule out the existence of the pattern possessing no attributes or the pattern possessing all attributes.
When there is no attribute hierarchy among the $K$ attributes, $\mce=\varnothing$ and $\mca=\{0,1\}^K$. } The set $\mca$ is a proper subset of $\{0,1\}^K$ if $\mce\neq\varnothing$. 
An attribute hierarchy determines the sparsity pattern of the vector of proportion parameters $\pp$, {because $p_{\aaa}>0$ if and only if $\aaa\in\mca(\mce)$, that is, if and only if $\aaa$ is permissible under $\mce$}.
In this sense, a nonempty attribute hierarchy necessarily leads to degenerate parameter space for $\pp$, as certain entries of $\pp$ will be constrained to zero.

%%%%% -- R1 part begins here
 {In the practice of studying the attribute hierarchy in cognitive diagnosis, the case of $k \to \ell$ and $\ell \to k$ would indicate the two skill attributes $\alpha_k$ and $\alpha_\ell$ are prerequisites for each other, which is not interpretable and hence is not used in modeling. 
Similarly, the case of having any cycle in the attribute hierarchy graph in the form of $k_1\to k_2 \to \cdots \to k_m \to k_1$ is also not interpretable.
Therefore, a   directed acyclic graph (DAG) structure among the latent attributes is well suited to describe the hierarchical nature of attributes that carry these substantive meanings.}
%%%%% -- R1 part ends here
We emphasize here that the  DAG  of attribute hierarchy in an HLAM has a different nature from that in a Bayesian network \citep{pearl1986, nielsen2009bn}. This is because the DAG of attribute hierarchy encodes hard constraints on what variable patterns are permissible/forbidden, while the DAG in a Bayesian network encodes the conditional independence relations among the variables. 
Instead, a neural network model RBForests proposed by \cite{rbf2010} shares a more similar spirit to the HLAM in this regard.
The following example illustrates this in detail.

\begin{example}
	\label{exp-rbf}
\normalfont{
Fig \ref{fig-hier} presents several hierarchies   with the size of the associated $\mca$, where a dotted arrow from $\alpha_k$ to $\alpha_\ell$  indicates $k\to \ell$ and $k$ is a direct prerequisite for $\ell$. Note that under the hierarchy in Fig \ref{fig-hier}(a), the prerequisite $1\to 3$ is an indirect prerequisite implied by $1\to 2$ (or 4) and 2 (or 4) $\to 3$.
 %(or by $1\to 4$ and $4\to 3$).
%
%The attribute hierarchy in an HLAM is a directed acyclic graph generally. 
In the literature, the RBForests proposed in \cite{rbf2010} also introduce hard constraints on allowable configurations of the binary hidden (latent) variables in a restricted Boltzmann machine (RBM). The modeling goal of RBForests is to make computing the probability mass function of observed variables tractable, while not having to limit the number of latent variables. Specifically, in an RBForest, latent variables are grouped in several full and complete binary trees of a certain depth, with variables in a tree respecting the following constraints: if a latent variable takes value zero with $\alpha_k=0$, then all latent variables in its left subtree must take value $d_l$; while if $\alpha_k=1$, all latent variables in its right subtree must take value $d_r$ ($d_l=d_r=0$ in \cite{rbf2010}).  
The attribute hierarchy model in an HLAM has a similar spirit to RBForests, and actually includes the RBForests as a special case. For instance, the hierarchy in Fig \ref{fig-hier}(c) is equivalent to a tree of depth 3 in an RBForest with $d_l=1-d_r=0$.
HLAMs allow for more general attribute hierarchies to encourage better interpretability (DAG instead of trees). 
Another fundamental difference between HLAMs and RBForests is the different joint model of the observed variables and the latent ones. An RBForest is an extension of an RBM, and they both use   the same energy function, while HLAMs model the distribution differently, as to be specified below.
}
\end{example}

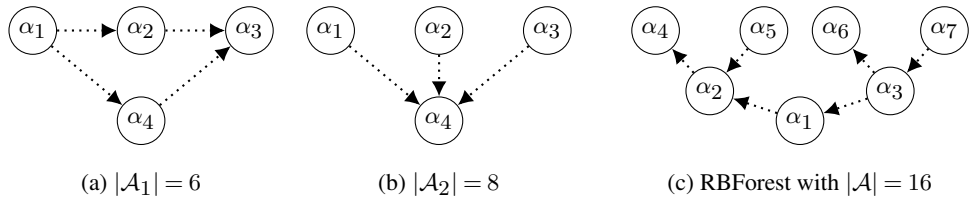
\begin{figure}[h!]
\centering
\begin{tikzpicture}[scale=1.2]

    \node (h5)[hidden] at (0, 1) {$\alpha_1$};
    \node (h6)[hidden] at (1.2, 1) {$\alpha_2$};
    \node (h7)[hidden] at (2.4, 1) {$\alpha_3$};
    \node (h8)[hidden] at (1.2, 0) {$\alpha_4$};
    
    \draw[pre] (h5) -- (h6);
    \draw[pre] (h6) -- (h7);
    \draw[pre] (h5) -- (h8);
    \draw[pre] (h8) -- (h7);

    %%%%%%
    \node (h14-1)[hidden] at (3.3, 1) {$\alpha_1$};
    \node (h14-2)[hidden] at (4.5, 1) {$\alpha_2$};
    \node (h14-3)[hidden] at (5.7, 1) {$\alpha_3$};
    \node (h14-4)[hidden] at (4.5, 0) {$\alpha_4$};
    
    \draw[pre] (h14-1) -- (h14-4);
    \draw[pre] (h14-2) -- (h14-4);
    \draw[pre] (h14-3) -- (h14-4);

%    %%%%%%
%    \node (h13-1)[hidden] at (6.6, 1) {$\alpha_1$};
%    \node (h13-2)[hidden] at (7.8, 1) {$\alpha_2$};
%    \node (h13-3)[hidden] at (9, 1) {$\alpha_3$};
%    \node (h13-4)[hidden] at (7.8, 0) {$\alpha_4$};
%    
%    \draw[pre] (h13-4) -- (h13-1);
%    \draw[pre] (h13-4) -- (h13-2);
%    \draw[pre] (h13-4) -- (h13-3);  
    
    %%%%%%%
%    \node (rbf1)[hidden] at (11.7, 0) {$\alpha_1$};
%    \node (rbf2)[hidden] at (10.7, 0.33) {$\alpha_2$};
%    \node (rbf3)[hidden] at (12.7, 0.33) {$\alpha_3$};
%    %
%    \node (rbf4)[hidden] at (10.1, 1) {$\alpha_4$};
%    \node (rbf5)[hidden] at (11.3, 1) {$\alpha_5$};
%    %
%    \node (rbf6)[hidden] at (12.1, 1) {$\alpha_6$};
%    \node (rbf7)[hidden] at (13.3, 1) {$\alpha_7$};

    \node (rbf1)[hidden] at (8.5, 0) {$\alpha_1$};
    \node (rbf2)[hidden] at (7.5, 0.33) {$\alpha_2$};
    \node (rbf3)[hidden] at (9.5, 0.33) {$\alpha_3$};
    \node (rbf4)[hidden] at (6.9, 1) {$\alpha_4$};
    \node (rbf5)[hidden] at (8.1, 1) {$\alpha_5$};
    \node (rbf6)[hidden] at (8.9, 1) {$\alpha_6$};
    \node (rbf7)[hidden] at (10.1, 1) {$\alpha_7$};
    
    \draw[pre] (rbf1) -- (rbf2);
    \draw[pre] (rbf3) -- (rbf1);
    \draw[pre] (rbf2) -- (rbf4);
    \draw[pre] (rbf5) -- (rbf2);
    \draw[pre] (rbf3) -- (rbf6);
    \draw[pre] (rbf7) -- (rbf3);

    \node at (1.2, -0.7) {(a) $|\mca_1|=6$};
    
    \node at (4.5, -0.7) {(b) $|\mca_2|=8$}; 
    
    % \node at (7.8, -0.7) {(c) $|\mca_3|=9$}; 
    
    % \node at (11.7, -0.7) {(d) RBForest with $|\mca|=16$};
    \node at (8.5, -0.7) {(c) RBForest with $|\mca|=16$};  
\end{tikzpicture}

\caption{Different attribute hierarchies among binary attributes, the first two for $K=4$ (where $|\{0,1\}^4|=16$) and the last for $K=7$ (where $|\{0,1\}^7|=128$). For example, the set of allowed attribute patterns under  hierarchy (a) is $\mca_1=\{\zero_4,\,(1000),\,(1100),\,(1001),\,(1101),\,\one_4\}$.}
\label{fig-hier}
\end{figure}
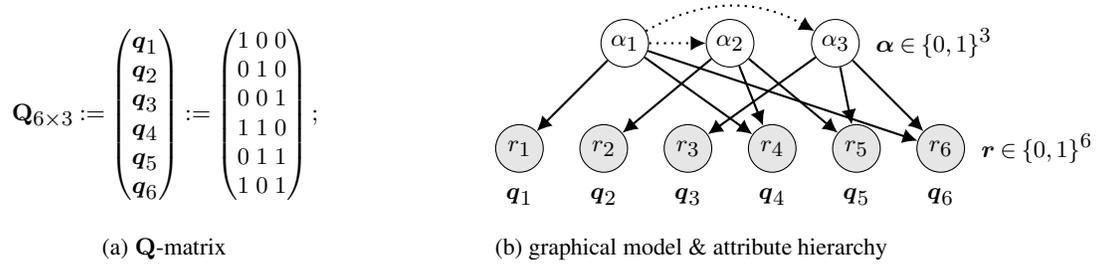
\begin{figure}[h!]
\centering
\begin{minipage}{0.35\textwidth}
$$
\QQ_{6\times 3}:=
\begin{pmatrix}
\bq_1\\
\bq_2\\
\bq_3\\
\bq_4\\
\bq_5\\
\bq_6
\end{pmatrix}:=\begin{pmatrix}
1 & 0 & 0\\
0 & 1 & 0\\
0 & 0 & 1\\
1 & 1 & 0\\
0 & 1 & 1\\
1 & 0 & 1\\
\end{pmatrix};\quad\quad\quad\quad
$$
\end{minipage}%%%%%%%%%%%%%%%%
\hfill
\begin{minipage}{0.55\textwidth}
\centering
    \begin{tikzpicture}[scale=1.4]

    \node (v1)[neuron] at (0, 0) {$r_1$};
    \node (v2)[neuron] at (0.8, 0) {$r_2$};
    \node (v3)[neuron] at (1.6, 0) {$r_3$};
    \node (v4)[neuron] at (2.4, 0) {$r_4$};
    \node (v5)[neuron] at (3.2, 0) {$r_5$};
    \node (v6)[neuron] at (4, 0)   {$r_6$};
    \node[right=0.1cm of v6] (v) {$\rr \in \{0, 1\}^6$};

    \node (h1)[hidden] at (1.0, 1) {$\alpha_1$};
    \node (h2)[hidden] at (2.0, 1) {$\alpha_2$};
    \node (h3)[hidden] at (3.0, 1) {$\alpha_3$};
    \node[right=0.1cm of h3] (h) {$\aaa \in \{0, 1\}^3$};

    \node [below=0.1cm of v1] {$\bq_1$};
    \node [below=0.1cm of v2] {$\bq_2$};
    \node [below=0.1cm of v3] {$\bq_3$};
    \node [below=0.1cm of v4] {$\bq_4$};
    \node [below=0.1cm of v5] {$\bq_5$};
    \node [below=0.1cm of v6] {$\bq_6$};    
    
%    \node [below=0.5cm of v1] {$(100)$};
%    \node [below=0.5cm of v2] {$(010)$};
%    \node [below=0.5cm of v3] {$(001)$};
%    \node [below=0.5cm of v4] {$(110)$};
%    \node [below=0.5cm of v5] {$(011)$};
%    \node [below=0.5cm of v6] {$(101)$};

    \draw[qedge] (h1) -- (v1) node [midway,above=-0.12cm,sloped] {}; 
    
    \draw[qedge] (h1) -- (v4) node [midway,above=-0.12cm,sloped] {}; 
    
    \draw[qedge] (h1) -- (v6) node [midway,above=-0.12cm,sloped] {};

    \draw[qedge] (h2) -- (v2) node [midway,above=-0.12cm,sloped] {}; 
    
    \draw[qedge] (h2) -- (v4) node [midway,above=-0.12cm,sloped] {}; 
    
    \draw[qedge] (h2) -- (v5) node [midway,above=-0.12cm,sloped] {}; 
    
    \draw[qedge] (h3) -- (v3) node [midway,above=-0.12cm,sloped] {}; 
    
    \draw[qedge] (h3) -- (v5) node [midway,above=-0.12cm,sloped] {};

    \draw[qedge] (h3) -- (v6) node [midway,above=-0.12cm,sloped] {}; 
    
    \draw[->,dotted,thick] (h1) -- (h2);
    % \draw[dotted,thick] (h2) -- (h3);
    
    \path
(h1) edge[->,dotted,bend left,thick] node [right] {} (h3);
    
\end{tikzpicture}
\end{minipage}

\vspace{3mm}
\begin{minipage}{0.35\textwidth}
\centering
	(a) $\QQ$-matrix\quad\quad\quad\quad
\end{minipage}
\hfill
\begin{minipage}{0.55\textwidth}
%\centering
	(b) graphical model \& attribute hierarchy
\end{minipage}

\caption{A binary structural matrix and the corresponding graphical model with {(solid)} directed edges from the latent to the observed variables representing dependencies. Below the observed variables in (b) are the row vectors of $\QQ_{6\times 3}$, i.e., the item loading vectors.
% The attribute hierarchy is 
 The dotted arrows indicate the attribute hierarchy with $\mce=\{1\to 2,\,1\to 3\}$ and $\mca=\{\zero_3,\,(100),\,(110),\,(101),\,\one_3\}$.}
\label{fig-q63}
\end{figure}

%The observed vector $\RR$ can be viewed as the responses to a set of $J$ \textit{items}. The $\aaa$ and $\bq_j$ are both $K$-dimensional binary vectors, which describes the intrinsic property of a potential subject and an item, respectively. In practice, given a sample of size $N$, the observations are $N$ response vectors $\RR_1,\ldots,\RR_J\in\{0,1\}^J$, and the latent variables are $N$ attribute patterns $\AA_1,\ldots,\AA_N\in\{0,1\}^K$; while the structural dependence of $\RR_i$'s on $\AA_i$'s is encoded by matrix $\QQ$ with row vectors $\bq_1,\ldots,\bq_J\in\{0,1\}^K$. We next build a probabilistic model for the observed responses. First, we assume a random effect model for the attribute patterns of the subjects, that is, the vectors $\AA_1,\ldots,\AA_N$ are i.i.d. from a categorical distribution with population proportion parameters $\pp=(p_{\aaa},\,\aaa\in\{0,1\}^K)$, i.e., $\mathbb P(\AA_i=\aaa)=p_{\aaa}$. The vector of proportions  $\pp$  lives in the $(2^K-1)$-dimensional simplex, with $2^K$ nonnegative entries summing to one. In practice, there could be $p_{\aaa}=0$ meaning that the attribute pattern $\aaa$ is deemed impossible in the population. An attribute hierarchy in an HLAM specifically defines the sparsity pattern of the vector $\pp$.

On top of the model of the latent attributes,
an HLAM uses a $J\times K$ binary matrix $\QQ=(q_{j,k})$ to encode the structural relationship between the $J$ observed response variables and the $K$ latent attributes.
In cognitive diagnostic assessments, the matrix $\QQ$ is often specified by domain experts to summarize which abilities each test item targets on \citep{Tatsuoka1990,davier2008general,rupp2010diagnostic,dela2011}.
 Specifically, $q_{j,k}=1$ if and only if the response $r_j$ to the $j$th item has statistical dependence on latent variable $\alpha_k$.
 The  distribution of   $r_j$, i.e., $\theta_{j,\aaa}:=\mathbb P(r_j=1\mid \aaa)$, only depends on  its ``parent'' latent attributes $\alpha_k$'s that are connected to $r_j$, i.e., $\{\alpha_k: q_{j,k}=1\}$.
 The structural matrix $\QQ$ naturally induces a bipartite graph connecting the latent and the observed variables, with edges  corresponding  to entries of ``1'' in   $\QQ=(q_{j,k})$. 
  Fig \ref{fig-q63} presents an example of a structural matrix $\QQ$ and its corresponding directed graphical model between the $K=3$ latent attributes and $J=6$ observed variables. The solid edges from the latent attributes to the observed variables are specified by $\QQ_{6\times 3}$.
 As also can be seen from the graphical model, the observed responses to the $J$ items are conditionally independent given the latent attribute pattern $\aaa$.

In the psychometrics literature, various HLAMs adopting the $\QQ$-matrix concept have been proposed with the goal of diagnosing targeted  attributes \citep{junker2001cognitive,templin2006measurement,davier2008general,HensonTemplin09,dela2011}. They are often called the cognitive diagnostic models. 
The general family of latent attribute models are also widely used in other scientific areas including psychiatric evaluation \citep{templin2006measurement, jaeger2006distinguishing, dela2018} with the goal of diagnosing patients' various mental disorders, and epidemiological diagnosis of disease etiology \citep{wu2016partially, wu2017nested, o2019causes}.
These applications share the common key interest in identifying the multivariate discrete latent attributes.
 
In this work, we mainly focus on a popular and fundamental type of modeling assumptions under such a framework; as to be revealed soon, this modeling assumption also has close connections to Boolean matrix factorization  \citep{rava2016boolean,ormachine2017}.
Specifically, we mainly consider the HLAMs that assume a logical \textit{ideal response} $\Gamma_{\bq_j,\aaa}$ given an attribute pattern $\aaa$ and an item loading vector $\bq_j$ in the noiseless case. Then item-level noise parameters are further introduced to account for uncertainty of observations. 
The following are two popular ways to define the ideal response.

The first is the Deterministic Input Noisy output ``And'' gate (DINA) model   \citep{junker2001cognitive,de2004higher,davier2014dina,culpepper2015bayesian}. The DINA model assumes a conjunctive  relationship among the attributes. The ideal response of attribute pattern $\aaa$ to item~$j$ is
	\begin{flalign}
		\label{eq-and}
		\text{(DINA ideal response)} && \Gamma^{\myand}_{\bq_j,\aaa} = \prod_{k=1}^K \alpha_k^{q_{j,k}}, \qquad  \qquad \qquad  \qquad &&
	\end{flalign}
	 {where the convention $0^0\equiv 1$ is adopted.}
	To interpret, $\Gamma_{\bq_j,\aaa}$ in \eqref{eq-and} indicates whether a  pattern $\aaa$ possesses all the attributes specified by the item loading vector $\bq_j$. 
	This conjunctive relationship is often assumed for diagnosis of students' mastery or deficiency of skill attributes in educational assessments, and $\Gamma_{\bq_j,\aaa}$ naturally  indicates whether a student with $\aaa$ has mastered all the  attributes required by the test item $j$.
	%  ideal response of a student when there is no measurement error.
%\end{example}
With $\Gamma_{\bq_j,\aaa}$ in \eqref{eq-and}, the   uncertainty of the responses  is further modeled by the item-specific Bernoulli parameters 
\begin{align}
	\label{eq-item}
	%\theta^+_j=1-\mathbb P(r_j=0\mid \Gamma_{\bq_j,\aaa}=1),
	 {\theta^+_j=\mathbb P(r_j=1\mid \Gamma_{\bq_j,\aaa}=1),}
	\quad 
	\theta^-_j=\mathbb P(r_j=1\mid \Gamma_{\bq_j,\aaa}=0),
\end{align}
where $\theta^+_j>\theta^-_j$ is assumed for identifiability.
For each item $j$, the  ideal response $\Gamma_{\bq_j,\bo\cdot}$, if viewed as a function of attribute patterns, divides the patterns into two latent classes $\{\aaa:\, \Gamma_{\bq_j,\aaa}=1\}$ and $\{\aaa:\,\Gamma_{\bq_j,\aaa}=0\}$; and  for these two latent classes, respectively, the {item parameters} quantify the noise levels of the response to item $j$ that deviates from the ideal response. Note that the $\theta_{j,\aaa}$ equals either $\theta_{j}^+$ or $\theta_j^-$, depending on the ideal response $\Gamma_{j,\aaa}$.
Denote the {item parameter} vectors by $\ttt^+=(\theta_1^+,\ldots,\theta_J^+)^\top$ and $\ttt^-=(\theta_1^-,\ldots,\theta_J^-)^\top$. 
% Such a model defined by  \eqref{eq-and} and \eqref{eq-item} is called the Deterministic Input Noisy output ``And'' (DINA) model in cognitive diagnosis  \cite{junker2001cognitive}.

%\begin{example}[OR-model]\label{exp-or}
	The second model is the  Deterministic Input Noisy output ``Or'' gate (DINO) model   \citep{templin2006measurement}. The DINO model assumes the following  ideal response
	\begin{flalign}\label{eq-or}
	  \text{(DINO ideal response)} 
		&&
		\Gamma^{\myor}_{\bq_j,\aaa} = I(q_{j,k}=\alpha_k=1~\text{for at least one}~k). &&
	\end{flalign}
%	\begin{flalign}
%		\label{eq-and}
%		\text{(DINA ideal response)} && \Gamma_{\bq_j,\aaa} = \prod_{k=1}^K \alpha_k^{q_{j,k}}. \qquad  \qquad &&
%	\end{flalign}
Such a disjunctive relationship is often assumed in psychiatric measurement of mental disorders \citep{templin2006measurement,dela2018}. 
With $\Gamma_{\bq_j,\aaa}$ in \eqref{eq-or}, the uncertainty of the responses  is   modeled by the item-specific   parameters as defined in \eqref{eq-item}.
In the Boolean matrix factorization literature, a similar model was proposed  \citep{rava2016boolean,ormachine2017}. 
%OrMachine is a probabilistic model decomposing a binary matrix $\RR$ into two binary factor matrices connected by an ``or'' relation. 
Adapted to the terminology here, \cite{ormachine2017}  assumes  the ideal response takes the form 
\begin{flalign}\label{eq-orequiv}
\text{(equivalent to \eqref{eq-or})} 
		&&	
	\Gamma^{\myor}_{\bq_j,\aaa}=1-\prod_{k=1}^K(1-\alpha_k q_{j,k}), &&
\end{flalign}
which is equivalent to the definition in \eqref{eq-or}, while the model in \cite{ormachine2017} constrains  all the item-level noise parameters to be the same. 
%\end{example}

The equivalent formulation \eqref{eq-orequiv} of the DINO model shows that its ideal response is symmetric about the two vectors $\aaa$ and $\bq_j$; while for the DINA model this is not the case. 
%There is an interesting  duality \citep{chen2015statistical} between DINA and DINO with $\Gamma^{\myor}_{\bq_j,\aaa} = 1-\Gamma^{\myand}_{\bq_j,\one_K-\aaa}$.
%Due to this duality, we next will focus on the asymmetric DINA model without loss of generality and write $\Gamma^{\myand}_{\bq_j,\aaa}$ simply as $\Gamma_{\bq_j,\aaa}$ for brevity.
We next first focus on the asymmetric DINA-based HLAMs, as they are very popular and fundamental models widely used in the motivating applications of educational cognitive diagnosis.  {We also study the identifiability of DINO-based HLAMs and another type of HLAMs in Section \ref{sec-multiple}.}
For notational simplicity, we next write  $\Gamma^{\myand}_{\bq_j,\aaa}$ simply as $\Gamma_{\bq_j,\aaa}$.
Denote by $\Gamma(\QQ,\mce)$ the $J\times |\mca(\mce)|$ ideal response matrix with the $(j,\aaa)$th entry being $\Gamma_{\bq_j,\aaa}$ for $\aaa\in\mca(\mce)$.
Under the introduced setup of DINA-based HLAMs, the probability mass function of the $J$-dimensional random response vector $\RR$ takes the form of
\begin{align*}%\label{eq-pmf}
	 P(\RR=\rr\mid \QQ,\mce,\ttt^+,\ttt^-,\pp) = 
	 \sum_{\aaa\in\mca(\mce)} p_{\aaa}
	&~ \prod_{j=1}^J 
	     [\Gamma_{\bq_j,\aaa}\theta_{j}^+ + 
	     (1-\Gamma_{\bq_j,\aaa})\theta_{j}^-]^{r_j} \\ \notag
	    &\times 
	     [1-\Gamma_{\bq_j,\aaa}\theta_{j}^+ - 
	     (1-\Gamma_{\bq_j,\aaa})\theta_{j}^-]^{1-r_j},
\end{align*}
where $\rr\in\{0,1\}^J$ is an arbitrary response pattern.

\section{Identifiability of $\QQ$, Attribute Hierarchy, and Model Parameters: Establishing Sufficiency}\label{sec-main}
This section presents one main result on the sufficient conditions for identifiability of $\QQ$, $\mathcal E$, and model parameters $\ttt^+$, $\ttt^-$, and $\pp$. 
Following the definition of identifiability in the statistics literature \citep[e.g.,][]{casella2002statistical},
we say that $(\QQ,\mce,\ttt^+,\ttt^-,\pp)$ of an HLAM are identifiable if for any $(\QQ,\mce,\ttt^+,\ttt^-,\pp)$ in the parameter space constrained by $\QQ$ and $\mce$, there are no $(\bar\QQ,\bar\mce,\bar\ttt^+,\bar\ttt^-,\bar\pp)\neq (\QQ,\mce,\ttt^+,\ttt^-,\pp)$ such that
\begin{equation}
\label{eq-orig-def}
\mathbb P(\RR=\rr\mid \bar\QQ,\bar\mce,\bar\ttt^+,\bar\ttt^-,\bar\pp) =\mathbb  P(\RR=\rr\mid \QQ,\mce,\ttt^+,\ttt^-,\pp), ~\forall \rr\in \{0,1\}^J. 
\end{equation}
We point out that in the above definition of identifiability, the alternative vector of proportion parameters $\bar\pp$ is not constrained to have support on $\mca(\mce)$. Instead, the vector $\bar\pp$ should be allowed to have an arbitrary support $\bar\mca$ potentially resulting from an arbitrary $\bar\mce$; the goal of establishing identifiability is indeed to develop conditions to ensure that as long as \eqref{eq-orig-def} holds, one must have $\bar\pp=\pp$ and $\bar\mce=\mce$ from the equations in \eqref{eq-orig-def}.

We further introduce some notation and important concepts. 
Since an attribute hierarchy is a directed acyclic graph, the $K$ attributes $\{1,2,\ldots,K\}$ can be arranged in a {topological order} such that the prerequisite relation ``$\to$'' only happens in one direction; in other words, we can assume without loss of generality that $k\to\ell$ only if $k<\ell$.
Define the following \textit{reachability matrix} $\EE$ among the $K$ attributes under the attribute hierarchy. The $\EE = (e_{k,\ell})$ is a $K\times K$ binary matrix, where $e_{k,k}=1$ for all $k\in[K]$ and  {$e_{\ell,k}=1$ if attribute $k$ is a direct or indirect prerequisite for attribute $\ell$}.
In cognitive diagnosis, the concept of the reachability matrix was first considered in \cite{tatsuoka1986graph} to represent the direct and indirect relationships between attributes.
It is not hard to see that if the attributes $1,2,\ldots,K$ are in a topological order described earlier, the reachability matrix $\EE$ is a lower-triangular matrix with all the diagonal entries being one.
%
%Denote by $\mca(\mce)$ the set of all the attribute patterns that respect the hierarchy $\mce$. Then if $\pp$ is the proportion parameter vector under hierarchy $\mce$, there is $\{\aaa\in\{0,1\}^K:\,p_{\aaa}>0\} = \mca(\mce)$.
%
%Also in cognitive diagnosis, those $\QQ$-matrices with all the row vectors respecting the attribute hierarchy $\mce$ are said to satisfy the ``{restricted $\QQ$-matrix design}'' \citep[e.g.,][]{cai2018,tu2019ahm}.
%This type of $\QQ$-matrices are empirically shown in \cite{tu2019ahm} to be useful in improving clustering accuracy of diagnostic test takers. 
%Note that the reachability matrix $\EE$ always satisfies the restricted $\QQ$-matrix design under $\mce$.

Under DINA-based HLAMs, any non-empty attribute hierarchy $\mce$ defines an equivalence relation on the set of all the $\QQ$-matrices. %and we denote this by $[\QQ]^{\mce}$. 
To see this, recall $\Gamma(\QQ,\mce)$ denotes the $J\times |\mca(\mce)|$ ideal response matrix. 
If $\Gamma(\QQ_1,\mce) = \Gamma(\QQ_2,\mce)$, then $\QQ_1$ and $\QQ_2$ are said to be in the same $\mce$-induced equivalence class and we denote this by $\QQ_1 \stackrel{\mce}{\sim} \QQ_2$.
The interpretation of this definition is as follows. If under a certain hierarchy $\mce$, two different $\QQ$-matrices lead to identical ideal responses for all the permissible latent patterns in $\mca(\mce)$, then these two $\QQ$-matrices are indistinguishable based on the response data; therefore they should be treated as equivalent.
% and denote the equivalence class of $\QQ$-matrices under hierarchy $\mca$ by $[Q]^{\mca}$. 
The following example illustrates how an attribute hierarchy determines a set of equivalent $\QQ$-matrices.
\begin{example}
	\label{exp-q-equiv}
	\normalfont{
	Consider the attribute hierarchy $\mce=\{1\to  2,\,1\to 3\}$ in Fig \ref{fig-q63}, which results in $\mca(\mce)=\{\zero_3,(100),(110),(101),\one_3\}$. The identity matrix $I_3$ is equivalent to the reachability matrix $\EE$ under $\mce$ and
	}
	
	\vspace{-3mm}
	\singlespacing
	\begin{equation}\label{eq-equiv}
			I_3=\begin{pmatrix}
			1 & 0 & 0\\
			0 & 1 & 0\\
			0 & 0 & 1\\
%			%
%			1 & 1 & 0\\
%			0 & 1 & 1\\
%			1 & 0 & 1
		\end{pmatrix}
		\stackrel{\mce}{\sim}  
		\EE=
		\begin{pmatrix}
			1 & 0 & 0\\
			\textcolor{blue!70!black}{\one} & 1 & 0\\
			\textcolor{blue!70!black}{\one} & 0 & 1\\
%			%
%			1 & 1 & 0\\
%			\mathbf1 & 1 & 1\\
%			1 & 0 & 1
		\end{pmatrix}
		\stackrel{\mce}{\sim}  
		\begin{pmatrix}
			1 & 0 & 0\\
			* & 1 & 0\\
			* & 0 & 1\\
%			%
%			1 & 1 & 0\\
%			* & 1 & 1\\
%			1 & 0 & 1
		\end{pmatrix},
	\end{equation}

\doublespacing
\normalfont{where the ``$*$'''s in the third matrix above indicate unspecified values, any of which can be either 0 or 1. 
This equivalence is due to that attribute $\alpha_1$ serves as the prerequisite for both $\alpha_2$ and $\alpha_3$, and any item loading vector $\bq_j$ measuring $\alpha_2$ or $\alpha_3$ is equivalent to a modified one that also measures $\alpha_1$, in terms of classifying the  patterns in $\mca$ into two categories $\{\aaa:\,\Gamma_{\bq_j,\aaa}=1\}$ and $\{\aaa:\,\Gamma_{\bq_j,\aaa}=0\}$. 
 {Note that any $\QQ$-matrix equivalent to $I_K$ under the $\mathcal E=\{1\to 2,~1\to 3\}$ must take the form of the third $\QQ$-matrix in \eqref{eq-equiv}.
Under a DINA-based HLAM, if the true $\QQ$-matrix $\QQ^{\true}$ is not known, then any other $\QQ$ with $\QQ \stackrel{\mce}{\sim} \QQ^{\true}$ can not be distinguished from $\QQ^{\true}$ based on the observations, even if the continuous parameters $(\ttt^+, \ttt^-, \pp)$ are all known.
This is because the ideal response matrix $\Gamma(\QQ,\mce)$ is the key latent structure underlying a DINA-based HLAM, and that if $\QQ \stackrel{\mce}{\sim} \QQ^{\true}$ (equivalently, $\Gamma(\QQ,\mce) = \Gamma(\QQ^{\true},\mce)$), then $\QQ$ and $\QQ^{\true}$ are inherently not distinguishable.
} % end of darkblue
} % end of normalfont
\end{example}

%The following  theorem establishes identifiability  for an HLAM.
%$([Q]^{\mca},~ \ttt^+,~\ttt^-,~\pp)$.
Given any attribute hierarchy $\mathcal E$, the equivalence $I_K\stackrel{\mathcal E}{\sim}\EE$ is always true by definition, for which Eq.~\eqref{eq-equiv} in Example \ref{exp-q-equiv} is an example.
Before presenting the theorem on sufficient conditions for identifiability, we introduce two useful operations on a $\QQ$-matrix given an attribute hierarchy $\mathcal E$: the ``densifying'' operation $\mathcal D^{\mce}(\cdot)$ and the ``sparsifying'' operation $\mathcal S^{\mce}(\cdot)$, as follows.
\begin{definition}
	\label{def-b-dense}
	Given an attribute hierarchy $\mathcal E$ and a  matrix $\QQ$, do the following: for any $q_{j,h}=1$ and $k\to h$, set $q_{j,k}$ to ``1'' and obtain a modified matrix $\mathcal D^{\mce}(\QQ)$. This $\mathcal D^{\mce}(\QQ)$ is said to be the ``densified'' version of $\QQ$.
\end{definition}

\begin{definition}
	\label{def-b-sparse}
	Given an attribute hierarchy $\mathcal E$ and a  matrix $\QQ$, do the following: for any $q_{j,h}=1$ and  $k\to h$, set $q_{j,k}$ to ``0'' and obtain a modified matrix $\mathcal S^{\mce}(\QQ)$. This $\mathcal S^{\mce}(\QQ)$ is said to be the ``sparsified'' version of $\QQ$.
\end{definition}

Under the above two definitions, given an attribute hierarchy, there are $\mc D^{\mce}(I_K)=\EE$ and $\mc S^{\mce}(\EE) = I_K$.
{In cognitive diagnosis, the densified $\QQ$-matrix with all the row vectors respecting the attribute hierarchy $\mce$ is also said to satisfy the ``{restricted $\QQ$-matrix design}'' \citep[e.g.,][]{cai2018,tu2019ahm}; for such $\QQ$, there is $\QQ = \mc D^{\mce}(\QQ)$.}
It is worth pointing out that either the sparsifying or the densifying operation modifies $\QQ$ only within a same equivalence class. Indeed, $\mc D^{\mce}(\QQ)$ denotes the densest $\QQ$ with the largest number of ``1''s in the equivalence class, while $\mc S^{\mce}(\QQ)$ denotes the sparsest $\QQ$ with the largest number of ``0''s in the equivalence class.
In the special case with an empty attribute hierarchy, each equivalence class of $\QQ$ contains only one element which is $\QQ$ itself, so $\QQ=\mc D^{\mce}(\QQ)=\mc S^{\mce}(\QQ)$ for $\mce=\varnothing$.
As will be revealed in the following theorem, our identifiability conditions are essentially requirements on the equivalence class of $\QQ$ described using the densifying and sparsifying operations. 

%We call the two types of modifications of matrix $\QQ$ described in Conditions B and C by the name ``Operation'' B and C, respectively.

\begin{theorem}\label{thm-main}
	Consider an HLAM under the DINA model an attribute hierarchy $\mce$. 
	Then $(\Gamma(\QQ,\mce),\,\ttt^+,\,\ttt^-,\,\pp)$ are jointly identifiable if the true $\QQ$ satisfies the following conditions.
\begin{enumerate}
    \item[A.] 
    %The $\QQ$ contains a $K\times K$ submatrix $\QQ^{0}$; and setting $\QQ^0_{j,k}$ to ``0'' for any $k\to h$ and $\QQ^0_{j,h}=1$ results a matrix equal to $I_K$  up to column permutation.
    The $\QQ$ contains $K\times K$ submatrix $\QQ^0$ that is equivalent to the identity matrix $I_K$ under the hierarchy $\mce$.
    
    \noindent
    (Without loss of generality,  assume the first $K$ rows of $\QQ$ form $\QQ^0$, and denote the remaining submatrix of $\QQ$ by $\QQ^\star$.)
    
    \item[B.] 
%   The $\mathcal S^{\mce}(\QQ^\star)$, sparsified version of $\QQ^\star$, has at least two entries of ``1''s in each column.
 The $\mathcal S^{\mce}(\QQ)$, sparsified version of $\QQ$, has at least three entries of ``1''s in each column.

    \item[C.] 
    The $\mathcal D^{\mce}(\QQ^\star)$, densified version of the submatrix $\QQ^\star$, contains $K$ distinct column vectors.

\end{enumerate}
Furthermore, Conditions A, B and C are   necessary and sufficient when there exists no hierarchy with $p_{\aaa}>0$ for all $\aaa\in\{0,1\}^K$. 
\end{theorem}

%%\color{blue!70!black}
We make several remarks on the relationship between the proposed conditions and existing literature.
\begin{remark}
\normalfont{In the cognitive diagnostic modeling literature, a $\QQ$-matrix is said to be ``complete'' if it can distinguish all the $2^K$ latent attribute profiles \citep{chiu2009}. 
When the latent pattern space $\mca$ is saturated with $\mca = \{0,1\}^K$, the completeness of $\QQ$ is a natural necessary requirement for identifiability. 
When $\mca = \{0,1\}^K$, the $\QQ$-matrix is complete if it contains all the $K$ distinct standard basis vectors as row vectors, that is, $\QQ$ contains an $I_K$. When there exists a certain attribute hierarchy $\mathcal E$ leading to some $\mca\subsetneq \{0,1\}^K$, the requirement for the ``completeness'' of $\QQ$ will change. Recently, \cite{kc2019}, \cite{cai2018}, and \cite{heller2019} studied conditions for the completeness of $\QQ$ under the attribute hierarchy. But these conditions can not ensure the entire probabilistic model structure involving $\QQ$, $\mce$, and parameters $\pp$, $\ttt^+$ and $\ttt^-$ are identifiable and estimable from data. To our knowledge,  {Theorem \ref{thm-main} establishes the first identifiability result under the attribute hierarchy in the literature.}
Condition A in Theorem \ref{thm-main} is equivalent to requiring that the sparsified $\mathcal S^{\mce}(\QQ)$ contains an $I_K$. Therefore, Conditions A and B combined are equivalent to the following statement about $\mathcal S^{\mce}(\QQ)$: 
the $\mathcal S^{\mce}(\QQ)$ contains an $I_K$ and each column of it has at least three entries of ``1''s.
}
\end{remark}
%\color{black}

%\color{blue!70!black}
\begin{remark}
\normalfont{As stated in the last part of Theorem \ref{thm-main}, when there is no attribute hierarchy with $\mce=\varnothing$, Conditions A, B, and C become necessary and sufficient for the identifiability of both $\QQ$ and $(\mce,\,\ttt^+,\,\ttt^-,\,\pp)$. In such a special case with $\mce=\varnothing$, \cite{id-Q} established the necessary and sufficient identifiability conditions termed as ``\textit{completeness}'' that requires the true $\QQ$ to contain an identity submatrix $I_K$, ``\textit{repeated-measurement}'' that requires $\QQ$ to have at least three entries of ``1'' in each column, and ``\textit{distinctiveness}'' requiring that in addition to containing an $I_K$, the $\QQ$ should contain distinct column vectors in the remaining submatrix; we denote these three requirements by Conditions A$^0$, B$^0$, and C$^0$, respectively.
Our current conditions A, B, and C in Theorem \ref{thm-main} can be thought of as ``$\mce$-\textit{completeness}'', ``$\mce$-\textit{repeated-measurement}'', ``$\mce$-\textit{distinctiveness}'' given an attribute hierarchy $\mce$.
%Note that Condition A$^\star$ is equivalent to requiring that $\mc D^{\mce}(\QQ)$ contains a reachability matrix $\EE$.
When $\mce = \varnothing$, the $\mc S^{\mce}(\QQ) = \mc D^{\mce}(\QQ) = \QQ$ holds; as a result, Condition A exactly becomes requiring $\QQ$ itself to contain a submatrix $I_K$; similarly, Conditions B and C exactly reduce to the conditions B$^0$ and C$^0$ on $\QQ$ itself.
Indeed, in such cases with $\mce = \varnothing$, the current conditions of ``$\mce$-\textit{completeness}'', ``$\mce$-\textit{repeated-measurement}'', ``$\mce$-\textit{distinctiveness}'' just reduce to the  ``\textit{completeness}'', ``\textit{repeated-measurement}'', ``\textit{distinctiveness}'' conditions proposed in \cite{id-Q}.
%However, going from the previously existing conditions under $\mce=\varnothing$ to an arbitrary $\mce\neq\varnothing$ as done in Theorem \ref{thm-main} is technically highly nontrivial.
Establishing identifiability under an arbitrary attribute hierarchy $\mce$ as done in Theorem \ref{thm-main} is technically much more challenging than the existing result for $\mce=\varnothing$.
Moreover, in the later Section \ref{sec-nece}, we will thoroughly study 
that under a fixed $\QQ$-matrix, 
how the necessity of the identifiability conditions changes when there is a nonempty hierarchy.
}% end of normalfont
\end{remark}

Theorem \ref{thm-main} ensures the discrete ideal response structure $\Gamma(\QQ,\mce)$ and all the associated model parameters $(\ttt^+,\,\ttt^-,\,\pp)$ are identifiable. The following proposition complements this conclusion and further establishes identifiability of $\mce$ and $\QQ$ based on Theorem \ref{thm-main}.

\begin{proposition}\label{prop}
Consider a DINA-based HLAM.
	In addition to Conditions A--C in Theorem \ref{thm-main}, if the true $\QQ$ is known to contain an $I_K$, then $(\mce, \ttt^+,\,\ttt^-,\,\pp)$ are identifiable. 
	%and $\QQ$ can be identified up to the equivalence class under the true $\mce$. 
On the other hand, it is indeed necessary for $\QQ$ to contain an $I_K$ to ensure an arbitrary $\mce$ is identifiable. 
\end{proposition} 

%%%%%%%%%%%%%%%%%%%%%%%%%%%%
%\color{blue!70!black}Regarding the identifiability of the structural matrix $\QQ$, we have the following conclusion.
\begin{proposition}
	\label{prop-Q}
	Consider a DINA-based HLAM.
	If Conditions A--C in Theorem \ref{thm-main} are satisfied and the true $\QQ$ is known in part to contain a submatrix $I_K$ for certain $K$ items, then the equivalence class of $\QQ$ defined by the attribute hierarchy $\mce$ is identifiable. 
	That is, the specific $\QQ$ is not strictly identifiable within its equivalence class under any $\mce \neq \varnothing$, but the densified $\mc D^{\mce}(\QQ)$ and the sparsified $\mc S^{\mce}(\QQ)$ are identifiable.
\end{proposition}
The statement in Proposition \ref{prop-Q} that $\QQ$ is identifiable only up to its equivalence class is inherent to all the DINA- or DINO-type HLAMs and it is an inevitable consequence of any nonempty attribute hierarchy $\mce \neq \varnothing$; see Example \ref{exp-q-equiv}. 
But this statement will not undermine the efficacy of the identifiability conclusion, because $\mc D^{\mce}(\QQ)$ and $\mc S^{\mce}(\QQ)$ themselves are still identifiable and provide practical interpretability of the structural matrix.
%%%%%%%%%%%%%%%%%%%%%%%%%%%%
We next present a toy example illustrating how to apply Theorem \ref{thm-main} to check identifiability.

\begin{example}
	\label{exp-thm1}
	\normalfont{
	Consider the attribute hierarchy $\{\alpha_1\to \alpha_2,\,\alpha_1\to\alpha_3\}$ among $K=3$ attributes as in Fig \ref{fig-q63}. The following $8\times 3$ structural matrix $\QQ$ satisfies Conditions A, B and C in Theorem \ref{thm-main}. In particular, the first 3 rows of $\QQ$ serve as $\QQ^0$ in Condition A, and the last 5 rows serve as $\QQ^\star$.
	 In the following display, the matrix entries modified by the sparsifying operation in Condition B  and the densifying operation in Condition C   are highlighted. The resulting $\mc S^{\mce}(\QQ)$ and $\mc D^{\mce}(\QQ)$ satisfy the requirements in Conditions B and C. So the HLAM associated with $\QQ$ is identifiable.
	  %{In addition, one can similarly check that the $28\times 3$ matrix $\QQ^{\text{ECPE}}$ presented in the earlier Table \ref{tab-ecpeq} for the ECPE data under the linear hierarchy $\mce^{\text{ECPE}}$ also satisfies the identifiability conditions in Theorem \ref{thm-main}.}
	  }
	  
	  \singlespacing
		\begin{eqnarray}\label{eq-op-sparse}
		\QQ
		=
		\begin{pmatrix}
			\QQ^0\\
			\hline
			\QQ^\star
		\end{pmatrix}
		=
		\begin{pmatrix}
            &I_3&\\		
            \hline	
			1 & 0 & 0\\
			1 & 0 & 0\\
			1 & 1 & 0\\
			0 & 0 & 1\\
			1 & 1 & 1\\
		\end{pmatrix}
		&
			\stackrel{\text{\normalfont{Sparsify}}}{\Longrightarrow}
		%\underset{\text{\normalfont{Sparsify}}}{\overset{\text{\normalfont{Operation} }C}{\Longrightarrow}}
		&
		\mc S^{\mce}(\QQ)=\begin{pmatrix}
            &I_3&\\
		    \hline
			1 & 0 & 0\\
			1 & 0 & 0\\
			\textcolor{orange!80!black}{\zero} & 1 & 0\\
			0 & 0 & 1\\
			\textcolor{orange!80!black}{\zero} & 1 & 1\\
		\end{pmatrix};
\\
		%%%%%%%%%%%%%%%%%%%%
	\label{eq-op-dense}
%	   \QQ
		&
			\stackrel{\text{\normalfont{Densify}}}{\Longrightarrow}
		%\underset{\text{\normalfont{Densify}}}{\overset{\text{\normalfont{Operation} }B}{\Longrightarrow}}
		&
		\mc D^{\mce}(\QQ)=\begin{pmatrix}
            &\EE &\\			
            \hline
			1 & 0 & 0\\
			1 & 0 & 0\\
			1 & 1 & 0\\
			\textcolor{blue!70!black}{\one} & 0 & 1\\
			1 & 1 & 1\\
		\end{pmatrix}.
	\end{eqnarray}
	
\doublespacing
\end{example}

%We make a remark on the relation between the proposed conditions and existing literature.% end of blue
%In the cognitive diagnostic modeling literature, a $\QQ$-matrix is said to be ``complete'' if it can distinguish all the $2^K$ latent attribute profiles \citep{chiu2009}. 
%When the latent pattern space $\mca$ is saturated with $\mca = \{0,1\}^K$, the completeness of $\QQ$ is a natural necessary requirement for identifiability. 
%When $\mca = \{0,1\}^K$, the $\QQ$-matrix is complete if it contains all the $K$ distinct standard basis vectors as row vectors, that is, $\QQ$ contains an $I_K$. When there exists a certain attribute hierarchy $\mathcal E$ leading to some $\mca\subsetneq \{0,1\}^K$, the requirement for the ``completeness'' of $\QQ$ will change. Recently, \cite{kc2019}, \cite{cai2018}, and \cite{heller2019} studied conditions for the completeness of $\QQ$ under the attribute hierarchy. But these conditions can not ensure the entire probabilistic model structure involving $\QQ$, $\mce$, and parameters $\pp$, $\ttt^+$ and $\ttt^-$ are identifiable and estimable from data. To our knowledge,  {Theorem \ref{thm-main} establishes the first identifiability result under the attribute hierarchy in the literature.}
%%
%Condition A in Theorem \ref{thm-main} is equivalent to requiring that the sparsified $\mathcal S^{\mce}(\QQ)$ contains an $I_K$. Therefore, Conditions A and B combined are equivalent to the following statement about $\mathcal S^{\mce}(\QQ)$: 
%the $\mathcal S^{\mce}(\QQ)$ contains an $I_K$ and each column of it has at least three entries of ``1''s.

When estimating an HLAM with the goal of recovering the ideal response structure $\Gamma(\QQ,\mce)$ and the model parameters, Theorem \ref{thm-main} guarantees that Conditions A, B and C suffice and are close to being necessary. 
 While   the goal is to uniquely determine the attribute hierarchy from the identified $\Gamma(\QQ,\mce)$, the additional condition that $\QQ$ contains an $I_K$ becomes  necessary. 
This phenomenon can be better understood if one relates it to the identification criteria for the factor loading matrix in factor analysis   \citep{anderson1958introduction,bai2012};  the   loading matrix there is  often required to include an identity submatrix or  satisfy certain rank constraints, since otherwise the loading matrix can  not be identifiable due to  rotational indeterminacy.
We point out that developing identifiability theory for HLAMs that can have arbitrarily complex hierarchies is   more difficult than the case without hierarchy, and hence Theorem \ref{thm-main} is a significant  technical advancement over previous works \citep[e.g.,][]{slam,partial}.

As stated in the end of Theorem \ref{thm-main}, Conditions A, B, and C become not only sufficient but also necessary for $\QQ$ and $(\mce,\ttt^+,\ttt^-,\pp)$ to be identifiable when there is no actual hierarchy among attributes. Interestingly, the necessity of these conditions will subtly change when a nonempty attribute hierarchy comes into play. 
%Specifically, different roles an attribute plays in the graph will affect model identifiability differently.
Our next section thoroughly investigates these aspects.

%\newpage
\section{Identifiability of Attribute Hierarchy and Model Parameters: Pushing Towards Necessity}\label{sec-nece}
In order to close the gap between necessity and sufficiency, in this section we thoroughly investigate the necessity of the identifiability conditions for $(\mce,\ttt^+,\ttt^-, \pp)$ under the assumption that $\QQ$ is known and fixed.
In the following Subsection \ref{sec-sub-ind}, we first investigate the necessity of the conditions proposed in Section \ref{sec-main} individually, to gain insight into how the necessity changes as the attribute hierarchy changes. 
Then in Subsection \ref{sec-sub-bridge}, we further establish the general necessary and sufficient conditions for identifying the attribute hierarchy and other parameters under an arbitrary hierarchy graph $\mce$.

%%%%%%%%%%%%%%%%%%%%%%%%%%%%%%%%%%%%
\subsection{ {Investigating the Necessity of Conditions A, B, C Individually}}\label{sec-sub-ind}
%%%%%%%%%%%%%%%%%%%%%%%%%%%%%%%%%%%%
Our first result establishes the necessity of Condition A in Theorem \ref{thm-main}.

\begin{proposition}\label{prop-A}
Consider a DINA-based HLAM.
	Condition A that the sparsified $\mathcal S^{\mce}(\QQ)$ contains an $I_K$ is necessary for  identifiability of $(\Gamma(\QQ,\mce),\,\allowbreak\ttt^+,\ttt^-,\,\pp)$.
\end{proposition}

Proposition \ref{prop-A} shows that Condition A can not be relaxed under any attribute hierarchy.
On the other hand, Condition B and Condition C are more ``local'' in the sense that they regard individual attributes (equivalently, individual columns of the $\QQ$-matrix). Interestingly, it turns out that the necessity of these two conditions highly depends on the role of each attribute in the attribute hierarchy graph.
We next characterize the fine boundary between sufficiency and necessity of identifiability conditions for various types of attributes.
Given any attribute hierarchy graph $\mce$, we define the following four types of attributes.

\begin{definition}[Singleton Attribute]\label{def-anc}
	An attribute $k$ is a ``singleton attribute'' if there \textbf{neither exists} any attribute $h$ such that $k\to h$ \textbf{nor exists} any attribute $\ell$ such that $\ell\to k$.
\end{definition}

\begin{definition}[Ancestor Attribute]\label{def-anc}
	An attribute $k$ is an ``ancestor attribute'' if there \textbf{exists} some attribute $h$ such that $k\to h$ but  \textbf{does not exist} any attribute $\ell$ such that $\ell\to k$.
\end{definition}

\begin{definition}[Leaf Attribute]\label{def-leaf}
	An attribute $k$ is a ``leaf attribute'' if there \textbf{exists} some attribute $\ell$ such that $\ell\to k$ but  \textbf{does not exist} any attribute $h$ such that $k\to h$.
\end{definition}

\begin{definition}[Intermediate Attribute]\label{def-leaf}
	An attribute $k$ is an ``intermediate attribute'' if there  \textbf{exists} some attribute $\ell$ with $\ell\to k$ and  also \textbf{exists} some attribute $h$ with $k\to h$.
\end{definition}

%In terms of terms on the graph defined by the attribute hierarchy, 
The above four definitions together describe a full categorization of attributes given any attribute hierarchy. In other words, given any $\mce$, an attribute is either a singleton, or an ancestor, or a leaf, or an intermediate attribute.
As a special case, when the attribute pattern space $\mca=\{0,1\}^K$ is saturated, all the $K$ attributes are singleton attributes.

\begin{example}\label{exp-leighton}
\normalfont{
	Leighton \textit{et al.} \citep{leighton2004attribute} is among the first works that considered the attribute hierarchy method for the purpose of cognitive diagnosis. In particular, they presented and named the four different types of hierarchies among $K=6$ attributes, as shown in our Fig \ref{fig-leighton-hier}. 
	In our terminology, in plot (a), attribute 1 is an ancestor attribute, attribute 6 is a leaf attribute, and the remaining attributes 2, 3, 4, 5 are intermediate attributes; in plot (b), the roles of the six attributes are the same as those in plot (a); in plot (c), attribute 1 is an ancestor attribute, attribute 2 and 3 are intermediate attributes, attributes 4, 5, 6 are leaf attributes; in plot (d), attribute 1 is an ancestor attribute, and the remaining   2, 3, 4, 5, 6 are leaf attributes.
	}
	\end{example}

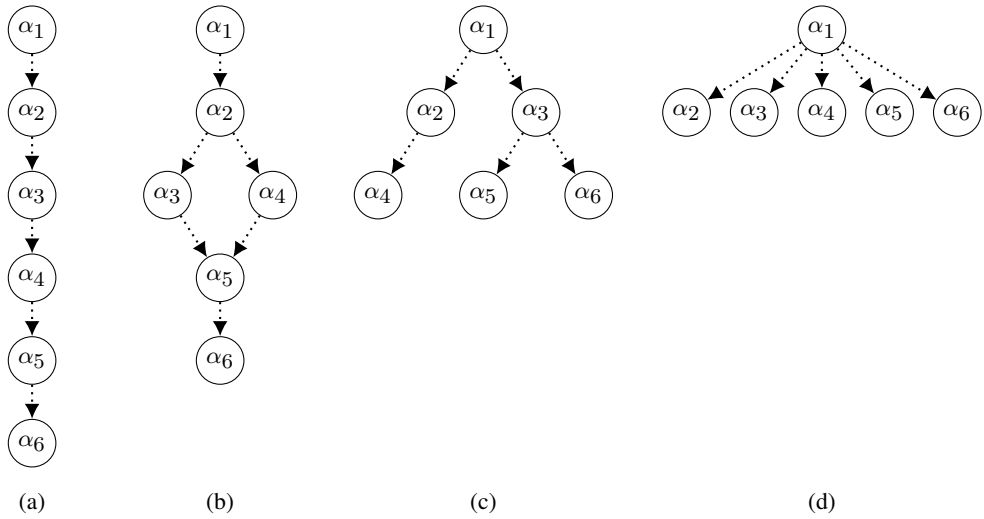
\begin{figure}[h!]
\centering
\begin{tikzpicture}[scale=1]
% linear among 6 attributes
	\node (h1)[hidden] at (0, 0) {$\alpha_1$};
	
    \node (h2)[hidden] at (0, -1.1) {$\alpha_2$};
    \node (h3)[hidden] at (0, -2.2) {$\alpha_3$};
    
    \node (h4)[hidden] at (0, -3.3) {$\alpha_4$};
    \node (h5)[hidden] at (0, -4.4) {$\alpha_5$};
    \node (h6)[hidden] at (0, -5.5) {$\alpha_6$};

    \draw[pre] (h1) -- (h2); 
    \draw[pre] (h2) -- (h3);
    \draw[pre] (h3) -- (h4);
    \draw[pre] (h4) -- (h5);
    \draw[pre] (h5) -- (h6);
    
% convergent among 6 attributes
	\node (h21)[hidden] at (2.5, 0)  {$\alpha_1$};
	\node (h22)[hidden] at (2.5, -1.1) {$\alpha_2$};
    \node (h23)[hidden] at (1.8, -2.2) {$\alpha_3$}; 
    \node (h24)[hidden] at (3.2, -2.2) {$\alpha_4$};
    \node (h25)[hidden] at (2.5, -3.3) {$\alpha_5$};
    \node (h26)[hidden] at (2.5, -4.4) {$\alpha_6$};

    \draw[pre] (h21) -- (h22); 
    \draw[pre] (h22) -- (h23);
    \draw[pre] (h22) -- (h24);
    
    \draw[pre] (h23) -- (h25);
    \draw[pre] (h24) -- (h25);
    \draw[pre] (h25) -- (h26);
    
% divergent among 6 attributes
\node (h31)[hidden] at (6, 0)  {$\alpha_1$};
\node (h32)[hidden] at (5.3, -1.1)  {$\alpha_2$};
\node (h33)[hidden] at (6.7, -1.1)  {$\alpha_3$};
\node (h34)[hidden] at (4.6, -2.2)  {$\alpha_4$};
\node (h35)[hidden] at (6, -2.2)  {$\alpha_5$};
\node (h36)[hidden] at (7.4, -2.2)  {$\alpha_6$};

    \draw[pre] (h31) -- (h32); 
    \draw[pre] (h31) -- (h33);
    \draw[pre] (h32) -- (h34);
    \draw[pre] (h33) -- (h35);
    \draw[pre] (h33) -- (h36);
    
% unstructured among 6 attributes
\node (h41)[hidden] at (10.5, 0)  {$\alpha_1$};
\node (h42)[hidden] at (8.7, -1.1)  {$\alpha_2$};
\node (h43)[hidden] at (9.6, -1.1)  {$\alpha_3$};
\node (h44)[hidden] at (10.5, -1.1)  {$\alpha_4$};
\node (h45)[hidden] at (11.4, -1.1)  {$\alpha_5$};
\node (h46)[hidden] at (12.3, -1.1)  {$\alpha_6$};

  \draw[pre] (h41) -- (h42);
  \draw[pre] (h41) -- (h43);
  \draw[pre] (h41) -- (h44);
  \draw[pre] (h41) -- (h45);
  \draw[pre] (h41) -- (h46);
  
\node(hier1)[] at (0, -6.3) {(a)};
\node(hier2)[] at (2.5, -6.3) {(b)};
\node(hier3)[] at (6, -6.3) {(c)};
\node(hier4)[] at (10.5, -6.3) {(d)};

\end{tikzpicture}    

\caption{Four attribute hierarchies presented in \cite{leighton2004attribute}, named as: (a) \textit{linear}, (b) \textit{convergent}, (c) \textit{divergent}, and (d) \textit{unstructured}. For example, in (b), $\alpha_1$ is an ancestor attribute, $\alpha_2,\ldots,\alpha_6$ are intermediate attributes, $\alpha_7$ is a leaf attribute, and there are no singleton attributes.}
\label{fig-leighton-hier}
\end{figure}

{For ease of discussion, in the following conclusions regarding necessity of the identifiability conditions, 
%we shall focus on the $\QQ$-matrices whose row vectors respect the attribute hierarchy, that is, satisfy the restricted $\QQ$-matrix design.
we shall focus on the $\QQ$-matrices that satisfy the restricted $\QQ$-matrix design.
 {Recall that a $\QQ$-matrix is said to satisfy the restricted $\QQ$-matrix design if each of its row vectors is a permissible attribute pattern under the hierarchy $\mathcal E$.}
In the literature of cognitive diagnostic modeling, the  restricted $\QQ$-matrix design  is shown empirically to be useful in improving clustering accuracy of diagnostic test takers \citep{tu2019ahm}.
Our theoretical findings in the rest of this subsection reveal that in addition to the  restricted $\QQ$-matrix design, what other requirements are necessary to ensure identifiability.
}

%To begin with, our next result shows Condition B that ``each latent attribute {in $\mathcal S^{\mce}(\QQ)$} is measured by at least three items'' can be relaxed under the \textit{common-ancestor} hierarchy, whose definition is as follows.
%\begin{definition}[Common-Ancestor Hierarchy]
%	An attribute hierarchy $\mce$ is said to be a {common-ancestor} hierarchy if there exists some latent attribute $k$ that serves as a direct or indirect prerequisite for all the other attributes.
%\end{definition}
%\noindent This family of the {common-ancestor} attribute hierarchy is quite general and includes many specific attribute structures. 
%Indeed, the \textit{linear} hierarchy, \textit{convergent} hierarchy, \textit{divergent} hierarchy and the so-called \textit{unstructured} hierarchy presented in \cite{templin2014hierarchical} (shown in our Fig \ref{fig-leighton-hier} in Example \ref{exp-leighton}) all belong to the common-ancestor hierarchy family.

{Before presenting the next identifiability result, we first introduce a new notion of identifiability of the attribute hierarchy $\mathcal E$ and proportion parameters $\pp$. Under an unknown nonempty hierarchy $\mathcal E\neq\varnothing$, if all row vectors of $\QQ$ respect the attribute hierarchy, then there exists a trivial nonidentifiability issue that can be resolved by introducing an equivalence relation, similar in spirit to that in \cite{partial}. To see this, consider $K=2$ and $\mathcal E=\{1\to 2\}$, then a $\QQ$-matrix $\QQ = \EE = (1, 0;~ 1, 1)$ has both rows respecting the attribute hierarchy. Further, consider the simplest special case without any item-level noise, $1-\theta_1^+ = 1-\theta_2^+ = \theta_1^- = \theta_2^- = 0$. Now if $\mathcal E$ is unknown, then it is not hard to see that any alternative proportion parameters $\bar\pp$ satisfying the following equations will be nondistinguishable from the true parameters $\pp$:

\vspace{-3mm}
	\singlespacing
\begin{align}\label{eq-partial}
	%\begin{cases}
		p_{(00)} = \bar p_{(00)} + \bar p_{(01)};\quad
		p_{(10)} = \bar p_{(10)};\quad
		p_{(11)} = \bar p_{(11)}.
	%\end{cases}
\end{align}
\doublespacing

Such phenomenon is closely related to the $\pp$-partial identifiability defined in \cite{partial},
which means when $\QQ$ does not contain an identity submatrix $I_K$, the proportion parameters can at best be identified up to the equivalence classes induced by $\QQ$.
In the current toy example, the attribute patterns $(00)$ and $(01)$ are equivalent under $\QQ=(1, 0;~ 1, 1)$ because $\Gamma_{\QQ, (00)} = \Gamma_{\QQ, (01)}$, and hence $\bar p_{(00)}$ and $\bar p_{(01)}$ can be identified up to their sum at best, as illustrated in \eqref{eq-partial}.
Therefore, we will say that $(\ttt^+, \ttt^-, [\mathcal E], [\pp])$ are identifiable, if $(\ttt^+, \ttt^-)$ are identifiable and the only nonidentifiability about $\pp$ is due to the equivalence relation in \eqref{eq-partial}; here $[\pp]$ denotes the equivalence class of proportion parameters satisfying \eqref{eq-partial} and the $[\mathcal E]$ denotes the associated equivalence class of  hierarchies.
We point out that such nonidentifiability is somewhat trivial and can be easily resolved, by simply defining the final $\bar{\mathcal E}^\star$ to be the \textit{hierarchy with the most directed edges among all the possible hierarchies in the equivalence class $[\bar{\mathcal E}]$}. It is easy to see that such $\bar{\mathcal E}^\star$ equals the true $\mathcal E$ in the toy example, because in order for $\bar{\mathcal E}$ to have the most directed edges, one needs to set $\bar p_{(01)} = 0$ under \eqref{eq-partial} and that exactly makes the resulting $\bar\pp=\pp$ and $\bar{\mathcal E}^\star = {\mathcal E} = \{1\to 2\}$. By a similar reasoning, this procedure also works more generally for any hierarchy $\mathcal E$.
Therefore, when a fixed $\QQ$-matrix has all rows respecting the hierarchy, it is still very meaningful and useful to study the identifiability of $(\ttt^+, \ttt^-, [\mathcal E], [\pp])$ and to investigate the minimal identifiability conditions. 
Our results in this section will establish the necessary and sufficient identifiability conditions in this regard.}

%
%\begin{proposition}[Identifiability for Common-Ancestor Hierarchy]\label{prop-two-e}
%%Consider those $\QQ$-matrices whose row vectors respect the hierarchy $\mce$. 
%Consider a DINA-based HLAM with a fixed $\QQ$-matrix.
%Under a common-ancestor hierarchy, if $\QQ$ contains two copies of the reachability matrix $\EE$ as a submatrix, then $(\ttt^+, \ttt^-, [\mathcal E],  [\pp])$ are identifiable.
%\end{proposition}
%
%Proposition \ref{prop-two-e} implies that a $2K\times K$ matrix $\QQ = (\EE^\top; \EE^\top)^\top$ containing just two copies of the reachability matrix $\EE$ suffices to ensure identifiability of the hierarchy and all the unknown model parameters. 
%
%
%	
%
%
%%
%The conclusion of Proposition \ref{prop-two-e} is in contrast to the result obtained in \cite{xu2016} for a latent attribute model without any attribute hierarchy; in that case, the condition that each attribute is measured by at least three items is indeed necessary; a similar requirement is also imposed in \cite{xu2017}.
%Intuitively, the relaxation of the condition happens since a nonempty attribute hierarchy necessarily leads to a degenerate parameter space for the population proportion parameters. Indeed, the number of nonzero free parameters is reduced to less than $2^K-1$ when $\mca(\mce)\neq\{0,1\}^K$, giving rise to relaxed requirements on the $\QQ$-matrix for establishing identifiability.
%
%
%
%  
%
%
%The above Proposition  \ref{prop-two-e} 
%shows the possibility of relaxing Condition B.

In the following Propositions \ref{prop-nece-sing}--\ref{prop-nece-int}, we show how Condition B can be generally relaxed, depending on whether the attribute is ancestor, leaf, or intermediate.

\begin{proposition}[Necessary Condition for Singleton Attribute]\label{prop-nece-sing}
Consider a DINA-based HLAM.
The following hold for a singleton attribute $k$ in {any} attribute hierarchy.
% that is, an attribute that is not connected to any other attribute.
\begin{itemize}
	\item[(a)] $\sum_{j=1}^J q_{j,k}\geq 3$ is necessary for the identifiability of $(\mce,\ttt^+,\ttt^-,\pp)$.
	\item[(b)] There exists scenarios where the equality in part (a) is achieved with $\sum_{j=1}^J q_{j,k} = 3$ and the identifiability of $(\ttt^+,\ttt^-, [\mce], [\pp])$ is guaranteed.
\end{itemize}
\end{proposition}

\begin{proposition}[Necessary Condition for Ancestor or Leaf Attribute]
\label{prop-nece-anc} 
Consider a DINA-based HLAM with a fixed $\QQ$-matrix whose row vectors respect the hierarchy $\mce$. Denote the $(j,k)$th entry of $\mathcal S^{\mce}(\QQ)$ by $q^{\sparse}_{j,k}$. The following conclusions hold for  $k$ if attribute $k$ is either an ancestor attribute or a leaf attribute.
%Consider an attribute $k$ for which there is some attribute $h$ such that $k\to h$ but there does not exist any attribute $\ell$ such that $\ell\to k$.
\begin{itemize}
	\item[(a)] $\sum_{j=1}^J q^{\sparse}_{j,k}\geq 2$ is necessary for  the identifiability of $(\mce,\ttt^+,\ttt^-,\pp)$.
	\item[(b)] There exist scenarios where the equality in part (a) is achieved with $\sum_{j=1}^J q^{\sparse}_{j,k} = 2$ and the identifiability of $(\ttt^+,\ttt^-, [\mce], [\pp])$ is guaranteed.
\end{itemize}
\end{proposition}

\begin{proposition}[Necessary Condition for Intermediate Attribute]\label{prop-nece-int}
Consider a DINA-based HLAM with a fixed $\QQ$-matrix whose row vectors respect the hierarchy $\mce$. Denote the $(j,k)$th entry of $\mathcal S^{\mce}(\QQ)$ by $q^{\sparse}_{j,k}$.
The following statements hold for an intermediate attribute $k$. 
\begin{itemize}
	\item[(a)]  $\sum_{j=1}^J q^{\sparse}_{j,k}\geq 1$ is necessary for the identifiability of $(\mce,\ttt^+,\ttt^-,\pp)$.
	\item[(b)] There exist scenarios where the equality in part (a) is achieved with $\sum_{j=1}^J q^{\sparse}_{j,k} = 1$ and the identifiability of $(\ttt^+,\ttt^-, [\mce], [\pp])$ is guaranteed.
\end{itemize}
\end{proposition}

Propositions \ref{prop-nece-sing}--\ref{prop-nece-int} together characterize the different identifiability phenomena caused by different types of attributes in the attribute hierarchy graph.
An intuitive explanation behind these conclusions is as follows. For a singleton attribute $k$ that is not connected to any other attribute in the attribute hierarchy graph, no additional information is provided by the other attributes. Therefore the requirement of $k$ being measured by $\geq 3$ items in the $\QQ$-matrix is necessary. This aligns well with the conclusion for a latent attribute model without any hierarchy established in \cite{xu2016} and \cite{id-dina}, where all the attributes are singletons and each needs to be measured by $\geq 3$ items.
However, this requirement can be relaxed for any other type of attribute which is somewhat connected in the attribute hierarchy graph. In particular, fewer measurements are needed for $k$ in the $\QQ$-matrix as more information is available for this attribute in the attribute hierarchy graph. 
For a ancestor attribute $k$ with some ``child'' or a leaf attribute with some ``parent'' as one-sided information, the requirement is relaxed to $k$ being measured by $\geq 2$ items in {$\mathcal S^{\mce}(\QQ)$}; while for an intermediate attribute $k$ with both some child and some parent as two-sided information, the requirement is further relaxed to $k$ being measured by $\geq 1$ items in {$\mathcal S^{\mce}(\QQ)$}.

We next discuss the necessity of Condition C. %that $\mc D^{\mce}(\QQ)$ should contain $K$ distinct column vectors in addition to a reachability matrix $\EE$ as a submatrix. 
Given a $\QQ$, we denote by $\QQ_{1:K,\bcolon}$ the submatrix consisting of its first $K$ rows and by $\QQ_{(K+1):J,\bcolon}$ the submatrix consisting of its last $J-K$ rows.
For a $\QQ$ with rows respecting the attribute hierarchy, Condition C requires $\QQ_{(K+1):J,\,k}\neq\QQ_{(K+1):J,\,\ell}$ for any $k\neq \ell$ when $\QQ_{1:K,\bcolon}=\EE$. 
We have the following result.
%The following result details how the necessity of this condition depends on the attribute hierarchy.

%\begin{theorem}[Discussing Necessity of Condition C]\label{thm-nec-dist}
%%	Suppose the first $K$ rows of $\QQ$ is equivalent to $I_K$ under the hierarchy $\mce$, that is, Condition A is satisfied. The following statements hold.
%%	\begin{itemize}
%%		\item[(a)] Suppose two attributes $k$ and $\ell$ are not connected in the attribute hierarchy, that is, $k\not\to \ell$ and $\ell\not\to\ell$. 
%Consider a DINA-based HLAM with a fixed $\QQ$-matrix whose row vectors respect the hierarchy $\mce$.
%Then the condition that $\QQ_{(K+1):J,\,k}\allowbreak\neq\QQ_{(K+1):J,\,\ell}$  (when $\QQ_{1:K,\bcolon}=\EE$)  is \textbf{necessary} for identifiability in either of the following two cases.
%	\begin{itemize}
%	\item[(a)] There is neither a path from $k$ to $\ell$ nor a path from $\ell$ to $k$. %in the hierarchy.
%	\item[(b)] There is a path between $k$ and $\ell$, say $k\to \ell$, and $k$ is an ancestor attribute.
%	\end{itemize}
%	The condition that $\QQ_{(K+1):J,\,k}\neq\QQ_{(K+1):J,\,\ell}$ is \textbf{not necessary} for identifiability in the following case.
%	\begin{itemize}
%	\item[(c)]There is a path between $k$ and $\ell$, say $k\to \ell$, and $k$ is \textbf{not} an ancestor attribute.
%	\end{itemize} 
%\end{theorem}

\begin{proposition}[Discussing Necessity of Condition C]\label{thm-nec-dist}
Consider a DINA-based HLAM with a fixed $\QQ$ whose row vectors respect the hierarchy $\mce$.
The condition that $\QQ_{(K+1):J,\,k}\allowbreak\neq\QQ_{(K+1):J,\,\ell}$  (when $\QQ_{1:K,\bcolon}=\EE$)  is \textbf{necessary} for identifiability if both $\alpha_k$ and $\alpha_{\ell}$ are singleton attributes.
\end{proposition}

\subsection{Bridging the Necessity and Sufficiency of the Identifiability Conditions}\label{sec-sub-bridge}
%Proposition \ref{prop-two-e} also indicates that the hierarchy can drastically change the identifiability conclusion given a same $\QQ$-matrix and actually makes the study more technically challenging. 
%%%
Still under a fixed and known $\QQ$-matrix as in Section \ref{sec-sub-ind}, we next investigate how the sufficient identifiability conditions for $(\ttt^+,\ttt^-,\pp)$ can meet the necessary identifiability conditions proposed earlier in Propositions \ref{prop-nece-anc}--\ref{thm-nec-dist}.
In the next theorem, we establish that the individual necessary conditions established in Section \ref{sec-sub-ind} combined are actually sufficient to guarantee the identifiability in fully general scenarios.
This result therefore establishes the general {necessary and sufficient condition} on the $\QQ$-matrix for identifiability under an arbitrary attribute structure.

\begin{theorem}[Necessary and Sufficient Conditions under a Fully General $\mce$]\label{thm-general}
Consider a DINA-based HLAM with a fixed $\QQ$-matrix whose row vectors respect the hierarchy $\mce$.
Then Condition A and the following Condition B$^\star$ and C$^\star$ are {necessary and sufficient} for the identifiability of $(\ttt^+,\ttt^-, [\mce], [\pp])$.

\begin{itemize}
\item[B$^\star$.]  In $\mathcal S^{\mce}(\QQ)$,  any intermediate attribute is each measured by $\geq 1$ items, any ancestor attribute and any leaf attribute is each measured by $\geq 2$ items, and any singleton attribute is each measured by $\geq 3$ items.

\item[C$^\star$.] For any two singleton attributes $\alpha_k$ and $\alpha_\ell$, there is $\QQ_{(K+1):J,\, k} \neq \QQ_{(K+1):J,\, \ell}$. (Assume $\QQ_{1:K,\bcolon} = \EE$ under Condition A.)
\end{itemize}
\end{theorem}

Theorem \ref{thm-general} covers any type of attribute structure and allows for any type of attributes in the attribute hierarchy graph.
In the special case where there are no singleton attributes in the attribute hierarchy graph, the necessary and sufficient identifiability conditions in Theorem \ref{thm-general} can be simplified.
We term such a family of hierarchies without any singleton attributes the \textit{connected-graph hierarchy}.
%The connected-graph hierarchy family is a  general family and covers the previously defined common-ancestor hierarchy as a sub-family. To see this, note that an attribute hierarchy graph with multiple ancestors serving as prerequisite for a common leaf \textit{does not} belong to the common-ancestor  family but \textit{does belong to} the the connected-graph  family.

\begin{corollary}[Necessary and Sufficient Condition under  a Connected Graph Hierarchy]\label{thm-connected}
Consider a DINA-based HLAM with fixed $\QQ$-matrix whose row vectors respect the hierarchy $\mce$.
Suppose the $K$ attributes form a connected graph. Then Condition A and the following Condition D are {necessary and sufficient} for the identifiability { of $(\mce, \ttt^+, \ttt^-, [\pp])$}.
\begin{itemize}
\item[D.] In $\mathcal S^{\mce}(\QQ)$, any ancestor attribute and any leaf attribute is each measured by $\geq 2$ items, and any intermediate attribute is each measured by $\geq 1$ items.
\end{itemize}
\end{corollary}

\begin{remark}
	\label{rmk-general}
	\normalfont{
	In the first extreme case, if $\mce=\varnothing$ without any true hierarchy among attributes, then Conditions A, B$^\star$, and C$^\star$ in Theorem \ref{thm-general} exactly become Conditions A, B, C in Theorem \ref{thm-main} in Section \ref{sec-main}.
	In the second extreme case, if there does not exist any singleton attribute in the attribute hierarchy graph, then  Condition B$^\star$ in Theorem \ref{thm-general} reduces to Condition D in the above Corollary \ref{thm-connected}; and Condition C$^\star$ in Theorem \ref{thm-general} should be understood as always satisfied and hence can be omitted.
	Namely, under a connected-graph hierarchy without any singleton attributes, the Conditions A, B$^\star$, and C$^\star$ in Theorem \ref{thm-general} exactly reduce to Conditions A and D in Corollary \ref{thm-connected}.
	Therefore, Theorem \ref{thm-general} covers Corollary \ref{thm-connected} as a special case and is indeed fully general.
	We state these two results separately to highlight both the most general form of the result, and also how the necessary and sufficient conditions simplify under the popular family of connected-graph hierarchy as depicted in Corollary \ref{thm-connected}.
	}% end of normalfont
\end{remark}

%{Corollary \ref{thm-connected} also implies that, if considering the linear hierarchy in particular,}  the minimum number of items needed for identifiability among $K$ attributes is $K+2$. Under the linear hierarchy with $\mce=\{1\to 2\to \cdots \to K\}$, such a $(K+2)\times K$ matrix is
%%$\QQ=( \EE^\top, 
%%	{\ee_1}^\top,
%%	\one_K^\top)^\top
%%$
%
%\vspace{-3mm}
%	\singlespacing
%$$\QQ=\begin{pmatrix}
%	\EE\\
%	{\ee_1} %\ee_k
%	\\
%	\one_K
%\end{pmatrix},$$
%\doublespacing
%
%\noindent
%where $\ee_1 =(1,0,\ldots,0)$ is a standard basis vector.
%{The following example illustrates this conclusion with the linear hierarchy and also the minimal requirements on $\QQ$ under other attribute hierarchies considered in \cite{leighton2004attribute}.}
The following example illustrates the minimal requirements on $\QQ$ under those attribute hierarchies considered in \cite{leighton2004attribute}.

\begin{example}
\normalfont{
Under the linear hierarchy $\mathcal E = \mce^{\text{linear}}$ in Fig \ref{fig-multiple}(b), the $8\times 6$ matrix $\QQ^{\text{linear}}_{8\times 6}$ shown in Fig \ref{fig-multiple}(a) encodes the minimal requirement to ensure an identifiable model. 
Fig \ref{fig-multiple}(b) visualizes the sparsified version of $\QQ^{\text{linear}}_{8\times 6}$ as the directed solid edges from the latent attributes to the observed item responses. 
%The ancestor attribute 1 and leaf attribute 6 each has two items measuring it, and all the intermediate attributes 2, 3, 4, 5  each has one item measuring it in $\mc S^{\mce}(\QQ^{\text{linear}}_{8\times 6})$.
Under the so-called convergent hierarchy and divergent hierarchy presented earlier in Fig \ref{fig-leighton-hier}, the minimal requirement on $\QQ$ for model identifiability are presented in parts (c)-(d) and parts (e)-(f)  of Figure \ref{fig-multiple}, repectively.
%Specifically, for the convergent hierarchy $\mce = \mce^{\text{conv}}$ in Fig \ref{fig-multiple}(d), the $\QQ^{\text{conv}}_{8\times 6}$ in Fig \ref{fig-multiple}(c) gives an identifiable model under minimal conditions. The ancestor attribute $1$ and the leaf attribute $6$ each has two items measuring it in $\mc S^{\mce}(\QQ^{\text{conv}}_{8\times 6})$, while the remaining four attributes $2,3,4,5$ are all intermediate attributes and each has only one item measuring it in $\mc S^{\mce}(\QQ^{\text{conv}}_{8\times 6})$. 
For the divergent hierarchy $\mce = \mce^{\text{div}}$ in Fig \ref{fig-multiple}(f), the $\QQ^{\text{div}}_{10\times 6}$ in Fig \ref{fig-multiple}(c) gives an identifiable model under minimal conditions. 
%The ancestor attribute 1 and all the leaf attributes $4,5,6$ are each measured by two items, and the two intermediate attributes 2 and 3 are each measured by only one item.
}% end of normalfont
\end{example}

\color{black}
\begin{figure}[h!]
\begin{minipage}[c]{0.35\textwidth}
\centering
$$
\QQ^{\text{linear}}_{8\times 6} = \begin{pmatrix}
	1 & 0 & 0 & 0 & 0 & 0 \\
	1 & 1 & 0 & 0 & 0 & 0 \\
	1 & 1 & 1 & 0 & 0 & 0 \\
	1 & 1 & 1 & 1 & 0 & 0 \\
	1 & 1 & 1 & 1 & 1 & 0 \\
	1 & 1 & 1 & 1 & 1 & 1 \\
	\hline
	1 & 0 & 0 & 0 & 0 & 0 \\
	1 & 1 & 1 & 1 & 1 & 1 \\
\end{pmatrix}
$$
\end{minipage}
\hfill
\begin{minipage}[c]{0.6\textwidth}
 \centering
 	\begin{tikzpicture}[scale=1]
    \node (h1)[hidden] at (0,0) {$\alpha_1$};
    \node (h2)[hidden] at (1,0) {$\alpha_2$};
    \node (h3)[hidden] at (2,0) {$\alpha_3$};
    \node (h4)[hidden] at (3,0) {$\alpha_4$};
    \node (h5)[hidden] at (4,0) {$\alpha_5$};
    \node (h6)[hidden] at (5,0) {$\alpha_6$};
    
    \draw[pre] (h1) -- (h2); 
    \draw[pre] (h2) -- (h3);
    \draw[pre] (h3) -- (h4);
    \draw[pre] (h4) -- (h5);
    \draw[pre] (h5) -- (h6);
    
    \node (v1)[neuron] at (-1,-1.2) {$r_7$};
    \node (v2)[neuron] at (0,-1.2) {$r_1$};
    \node (v3)[neuron] at (1,-1.2) {$r_2$};
    \node (v4)[neuron] at (2,-1.2) {$r_3$};
    \node (v5)[neuron] at (3,-1.2) {$r_4$};
    \node (v6)[neuron] at (4,-1.2) {$r_5$};
    \node (v7)[neuron] at (5,-1.2) {$r_6$};
    \node (v8)[neuron] at (6, -1.2) {$r_8$};
    
    \draw[qedge] (h1) -- (v1);
    \draw[qedge] (h1) -- (v2);
    \draw[qedge] (h2) -- (v3);
    \draw[qedge] (h3) -- (v4);
    \draw[qedge] (h4) -- (v5);
    \draw[qedge] (h5) -- (v6);
    \draw[qedge] (h6) -- (v7);
    \draw[qedge] (h6) -- (v8);
 	\end{tikzpicture}
\end{minipage}

\vspace{2mm}
\begin{minipage}[c]{0.35\textwidth}\centering
	(a) $\QQ^{\text{linear}}_{8\times 6}$
\end{minipage}
\hfill
\begin{minipage}[c]{0.6\textwidth}\centering
	(b) visualization of the sparsified $\mathcal S^{\mce}(\QQ^{\text{linear}}_{8\times 6})$
\end{minipage}

\bigskip
%%% convergent hierarchy begins here %%%
%%% convergent hierarchy begins here %%%
%%% convergent hierarchy begins here %%%
\begin{minipage}[c]{0.38\textwidth}
\centering
$$
\QQ^{\text{conv}}_{8\times 6}
=
\begin{pmatrix}
	1 & 0 & 0 & 0 & 0 & 0 \\
	1 & 1 & 0 & 0 & 0 & 0 \\
	1 & 1 & 1 & 0 & 0 & 0 \\
	1 & 1 & 0 & 1 & 0 & 0 \\
	1 & 1 & 1 & 1 & 1 & 0 \\
	1 & 1 & 1 & 1 & 1 & 1 \\
	\hline
	1 & 0 & 0 & 0 & 0 & 0 \\
	1 & 1 & 1 & 1 & 1 & 1 
\end{pmatrix}
$$
\end{minipage}
\hfill
\begin{minipage}[c]{0.6\textwidth}
\centering
\begin{tikzpicture}[scale=1]
% convergent among 6 attributes
    \node (h21)[hidden] at (-4.4, 2.5) {$\alpha_1$};
    \node (h22)[hidden] at (-3.3, 2.5) {$\alpha_2$};
    \node (h23)[hidden] at (-2.2, 3.2) {$\alpha_3$};
    \node (h24)[hidden] at (-2.2, 1.8) {$\alpha_4$}; 
    \node (h25)[hidden] at (-1.1, 2.5) {$\alpha_5$};
    \node (h26)[hidden] at (0, 2.5)  {$\alpha_6$};
    
    \draw[pre] (h21) -- (h22); 
    \draw[pre] (h22) -- (h23);
    \draw[pre] (h22) -- (h24);
    
    \draw[pre] (h23) -- (h25);
    \draw[pre] (h24) -- (h25);
    \draw[pre] (h25) -- (h26);
    
    \node (v1)[neuron] at (-4.4, 1.3) {$r_1$};
    \node (v7)[neuron] at (-5.4, 1.3) {$r_7$};
    \draw[qedge] (h21) -- (v1);
    \draw[qedge] (h21) -- (v7);
    
    \node (v2)[neuron] at (-3.3, 1.3) {$r_2$};
    \draw[qedge] (h22) -- (v2);
    
    \node (v3)[neuron] at (-2.2, 4.4) {$r_3$};
    \draw[qedge] (h23) -- (v3);
    
    \node (v4)[neuron] at (-2.2, 0.6) {$r_4$};
    \draw[qedge] (h24) -- (v4);
    
    \node (v5)[neuron] at (-1.1, 1.3) {$r_5$};
    \draw[qedge] (h25) -- (v5);
    
    \node (v6)[neuron] at (0, 1.3) {$r_6$};
    \node (v8)[neuron] at (1, 1.3) {$r_8$};
    \draw[qedge] (h26) -- (v6);
    \draw[qedge] (h26) -- (v8);
        
\end{tikzpicture}  
\end{minipage}  

\vspace{2mm}
\begin{minipage}[c]{0.38\textwidth}\centering
	{(c) $\QQ^{\text{conv}}_{8\times 6}$}
\end{minipage}
\hfill
\begin{minipage}[c]{0.6\textwidth}\centering
	{(d) visualization of the sparsified $\mathcal S^{\mce}(\QQ^{\text{conv}}_{8\times 6})$}
\end{minipage}

\bigskip
%%% divergent hierarchy begins here %%%
%%% divergent hierarchy begins here %%%
%%% divergent hierarchy begins here %%%
\begin{minipage}[c]{0.38\textwidth}
\centering
$$
\QQ^{\text{div}}_{10\times 6}
=
\begin{pmatrix}
	1 & 0 & 0 & 0 & 0 & 0 \\
	1 & 1 & 0 & 0 & 0 & 0 \\
	1 & 0 & 1 & 0 & 0 & 0 \\
	1 & 1 & 0 & 1 & 0 & 0 \\
	1 & 0 & 1 & 0 & 1 & 0 \\
	1 & 0 & 1 & 0 & 0 & 1 \\
	\hline
	1 & 0 & 0 & 0 & 0 & 0 \\
	1 & 1 & 0 & 1 & 0 & 0 \\
	1 & 0 & 1 & 0 & 1 & 0 \\
	1 & 0 & 1 & 0 & 0 & 1 \\
\end{pmatrix}
$$
\end{minipage}
\hfill
\begin{minipage}[c]{0.6\textwidth}
\centering
\begin{tikzpicture}[scale=1]

% divergent among 6 attributes
\node (h31)[hidden] at (0.7, 0)  {$\alpha_1$};
\node (h32)[hidden] at (0, -1)  {$\alpha_2$};
\node (h33)[hidden] at (1.4, -1)  {$\alpha_3$};
\node (h34)[hidden] at (-0.7, -2)  {$\alpha_4$};
\node (h35)[hidden] at (0.7, -2)  {$\alpha_5$};
\node (h36)[hidden] at (2.1, -2)  {$\alpha_6$};
    
    \draw[pre] (h31) -- (h32); 
    \draw[pre] (h31) -- (h33);
    \draw[pre] (h32) -- (h34);
    \draw[pre] (h33) -- (h35);
    \draw[pre] (h33) -- (h36);

    \node (v31)[neuron] at (-0.6,0) {$r_1$};
    \node (v32)[neuron] at (2,0) {$r_7$};
    % for h32
%    \node (v33)[neuron] at (-0.6,0) {$r_2$};
    \node (v34)[neuron] at (-1.3,-1) {$r_2$};
    % for h33
%    \node (v35)[neuron] at (2,0) {$r_3$};
    \node (v36)[neuron] at (2.7,-1) {$r_3$};
    % for h34
    \node (v37)[neuron] at (-2, -2) {$r_4$};
    \node (v38)[neuron] at (-1.4,-3) {$r_{8}$};
    % for h35
    \node (v39)[neuron] at (0, -3) {$r_5$};
    \node (v310)[neuron] at (1.4,-3) {$r_{9}$};
    % for h36
    \node (v311)[neuron] at (3.4, -2) {$r_{6}$};
    \node (v312)[neuron] at (2.8, -3) {$r_{10}$}; 

    % q edges
    \draw[qedge] (h31) -- (v31);
    \draw[qedge] (h31) -- (v32);
%    \draw[qedge] (h32) -- (v33);
    \draw[qedge] (h32) -- (v34);
%    \draw[qedge] (h33) -- (v35);
    \draw[qedge] (h33) -- (v36);
    \draw[qedge] (h34) -- (v37);
    \draw[qedge] (h34) -- (v38);
    \draw[qedge] (h35) -- (v39);
    \draw[qedge] (h35) -- (v310);
    \draw[qedge] (h36) -- (v311);
    \draw[qedge] (h36) -- (v312); 
  
\end{tikzpicture}  
\end{minipage}  

\vspace{2mm}
\begin{minipage}[c]{0.38\textwidth}\centering
	 {(e) $\QQ^{\text{div}}_{10\times 6}$}
\end{minipage}
\hfill
\begin{minipage}[c]{0.6\textwidth}\centering
	 {(f) visualization of the sparsified $\mathcal S^{\mce}(\QQ^{\text{div}}_{10\times 6})$}
\end{minipage}

\caption{ {Minimally sufficient requirements on $\QQ$ for identifiability under the \textit{linear} hierarchy, \textit{convergent} hierarchy, and \textit{divergent} hierarchy proposed in \cite{leighton2004attribute}, respectively.}}
\label{fig-multiple}
\end{figure}
%%\color{blue!70!black}

%In summary, in this section we have discussed in detail the necessity of the sufficient identifiability conditions proposed in Section \ref{sec-main}. 
%We show that Condition A is necessary under any attribute hierarchy while Condition B and Condition C can be relaxed in nontrivial ways depending on the particular attribute hierarchy graph.
%Specifically, 
%%Proposition \ref{prop-two-e} starts with the common-ancestor attribute hierarchy and establishes  
%Theorems \ref{prop-nece-sing}--\ref{prop-nece-int} regard Condition B and they show that the minimal number of items needed for each attribute differs as the attribute type differs; 
%Theorem \ref{thm-nec-dist} regards Condition C and it exhaustively discusses its necessity for different pairs of attributes.
%%We also give the necessary and sufficient identifiability 
%Theoretically, one general takeaway message is that, the more available information an attribute has from $\mce$, the less requirement on $\QQ$ is needed for this attribute to establish identifiability.
%Practically, the proposed necessary conditions provide a minimal requirement on the structure of the design matrix $\QQ$ when developing identifiable cognitive diagnostic tests.

%%%%%%%%%%%
%\color{blue!70!black}
\section{Identifiability of other HLAMs different from the DINA-based HLAMs}\label{sec-multiple}

We also study identifiability of some other HLAMs in addition to the DINA-based HLAMs.

% section 5.1
\subsection{DINO-based HLAMs}
As introduced earlier in Section \ref{sec-setup}, the DINO model is also a popular type of latent attribute model often used for psychiatric and clinical measurement of mental disorders \citep{templin2006measurement, dela2018}.
A careful examination of the definitions of ideal responses $\Gamma^{\myand}$ and $\Gamma^{\myor}$ in \eqref{eq-and} and \eqref{eq-orequiv} reveals the following relationship, 
%it is not hard to see that DINO is dual to DINA in the sense that 
\begin{equation}\label{eq-dual}
	\Gamma^{\myor}_{\bq_j, \aaa} = 1 - \Gamma^{\myand}_{\bq_j, \one_K - \aaa},
\end{equation}
where $\one_K - \aaa = (1-\alpha_1,\ldots,1-\alpha_K)^\top$ also denotes an attribute pattern.
Building upon such duality between DINA and DINO, the following proposition characterizes how the identifiability results obtained under a DINA-based HLAM can be translated into those under a DINO-based HLAM.

\begin{proposition}
	\label{prop-dino}
	Consider a DINO-based HLAM with a fixed $\QQ$-matrix and an unknown attribute hierarchy $\mce$. Define the reversed attribute hierarchy $\mce^{\rev}$ as
	\begin{equation}
		\label{eq-reverse}
		\mce^{\rev} = \{\ell \to k:\, \text{if }k\to \ell\text{ under the original hierarchy }\mce\}.
	\end{equation}
	\begin{itemize}
    \item[(a)] For any $\aaa \in\{0,1\}^K$,  $\aaa\in\mca(\mce)$ if and only if $\one_K - \aaa\in\mca(\mce^{\rev})$. That is, any attribute pattern  $\aaa$ that is allowable under the original hierarchy $\mce$ if and only if another attribute pattern $\aaa' = \one - \aaa$ is allowable under the reversed hierarchy $\mce^{\rev}$.
    
    \vspace{2mm}
    \item[(b)] The attribute hierarchy $\mce$ and model parameters under the DINO-based HLAM are identifiable \textbf{if and only if} the reversed attribute hierarchy $\mce^{\rev}$ and model parameters are identifiable under a DINA-based HLAM with the same $\QQ$-matrix.
	\end{itemize}
	
\end{proposition}

For any attribute hierarchy graph $\mce$, the reversed hierarchy $\mce^{\rev}$ in \eqref{eq-reverse} is another directed graph among attributes, where the direction of each arrow in $\mce$ is reversed.
Therefore, for the same set of $K$ attributes, any ancestor attribute in $\mce$ becomes a leaf attribute in $\mce^{\rev}$, and any leaf in $\mce$ in turn becomes an ancestor in $\mce^{\rev}$.
Any intermediate attribute or singleton attribute remain the same type when $\mce$ is reversed to be $\mce^{\rev}$.
Proposition \ref{prop-dino} provides guidelines on how to check identifiability for a DINO-based HLAM using the identifiability results established earlier for DINA-based HLAMs. 
%In fact, since the necessary and sufficient identifiability conditions under a DINA-based HLAM provided in Theorem \ref{thm-general} depends on the $\QQ$-matrix structure, we can obtain the following corollary for a DINO-based HLAM.
In particular, we have the following necessary and sufficient conditions for identifiability of $(\mce, \ttt^+, \ttt^-, \pp)$ under a DINO-based HLAM with a fixed $\QQ$-matrix.

\begin{corollary}[Necessary and Sufficient Conditions under a General $\mce$ for a DINO-based HLAM]\label{cor-dino}
Consider a DINO-based HLAM with an attribute hierarchy $\mce$ and a fixed $\QQ$-matrix whose rows respect the reversed hierarchy $\mce^{\rev}$. Consider the following condition.
\begin{itemize}
\item[A$^\star$.] The $\mce^{\rev}$-densified matrix $\mc D^{\mce^{\rev}}(\QQ)$ contains a submatrix which is the reachability matrix under the reversed hierarchy $\mce^{\rev}$. %(denoted by $\EE(\mce_0^{\rev})$).
\end{itemize}
Then this Condition A$^\star$, and the earlier Conditions B$^\star$--C$^\star$ given in Theorem \ref{thm-general} are {necessary and sufficient} for the identifiability { of $(\mce,\ttt^+,\ttt^-,\pp)$}.
%\begin{itemize}
%\item[B$^\star$.]  In $\mathcal S^{\mce}(\QQ)$,  any intermediate attribute is each measured by $\geq 1$ items, any ancestor attribute and any leaf attribute is each measured by $\geq 2$ items, and any singleton attribute is each measured by $\geq 3$ items.
%
%\item[C$^\star$.] For any two singleton attributes $\alpha_k$ and $\alpha_\ell$, there is $\QQ_{(K+1):J,\, k} \neq \QQ_{(K+1):J,\, \ell}$. (Assume $\QQ_{1:K,\bcolon} = \EE$ under Condition A.)
%\end{itemize}
\end{corollary}

% section 5.2
\subsection{Main-effect-based HLAMs}
Another family of HLAMs in the literature \citep[e.g.,][]{dibello1995unified, davier2008general, HensonTemplin09} incorporate the main effects of latent attributes into the model.
We next review these main-effect-based HLAMs in the following Example \ref{exp-other} and then provide the identifiability result for them.

\begin{example}[HLAMs which Model the Main Effects of Attributes]\label{exp-other}
{\normalfont
%Multi-parameter HLAMs can be categorized into two general types, the main-effect models and the all-effect models.
The \textit{main-effect HLAMs} assume the main effects of the attributes measured by each item indicated by $\bq_j$ play a role in distinguishing the item parameters. Under a main-effect HLAM the Bernoulli parameter $\theta_{j,\aaa}$ can be written as
\begin{equation}\label{eq-maineff}
\theta^{\text{main-eff}}_{j,\aaa}   =
f\Big(\beta_{j,0}+ {\sum}_{k=1}^K\beta_{j,k} q_{j,k} \alpha_k \Big),
\end{equation}
where $f(\cdot)$ is a link function.
Note not all the $\beta$-coefficients in the above display are needed in the model specification; instead, only when $q_{j,k} = 1$ will $\beta_{j,k}$ be needed and truly incorporated in the model.
Different link functions $f(\cdot)$ in \eqref{eq-maineff} lead to different models, including 
%the reduced Reparameterized Unified Model \citep[reduced-RUM;][]{dibello1995unified} with $f(\cdot)$ being the exponential function, 
the Linear Logistic Model \citep[LLM;][]{maris1999estimating} with $f(\cdot)$ being the sigmoid function, and the Additive Cognitive Diagnosis Model \citep[ACDM;][]{dela2011} with $f(\cdot)$ being the identity.
When $f(\cdot)$ is a monotonically increasing function, it is usually assumed in practice that each $\beta_{j,k} > 0$ wherever $q_{j,k} = 1$ for interpretability. 
}% end of normalfont

\normalfont{There are also \textit{all-effect HLAMs} that model not only the main effects but also all the interaction effects of attributes. The  Bernoulli parameter $\theta_{j, \aaa}$ of an all-effect model is
\begin{align}\label{eq-alleff}
%\theta^{\text{all-eff}}_{j,\aaa} =
%f\Big({\sum}_{S\subseteq \mathcal K_{j}}\beta_{j,S}{\prod}_{k\in S}\alpha_k\Big).
\theta^{\text{all-eff}}_{j,\aaa} = 
	f\Big( & \beta_{j,0} + {\sum}_{k=1}^K \beta_{j,k} (q_{j,k} \alpha_{k}) + {\sum}_{1\leq k_1 < k_2\leq K} \beta_{j, k_1 k_2} (q_{j,k_1}\alpha_{k_1}) (q_{j,k_2}  \alpha_{k_2}) + 
	\\ \notag
	& 
	\cdots + \beta_{j,12\cdots K} {\prod}_{k=1}^K (q_{j,k} \alpha_{k})
	\Big).
\end{align}
Similarly as in \eqref{eq-maineff}, not all the $\beta$-coefficients above are needed in the model specification.
When $f(\cdot)$ in \eqref{eq-alleff} is the identity function, \eqref{eq-alleff} gives the Generalized DINA (GDINA) model in \cite{dela2011}; and when $f(\cdot)$ is the sigmoid function, \eqref{eq-alleff} gives the Log-linear Cognitive Diagnosis Models (LCDMs) in \cite{HensonTemplin09}; see also the General Diagnostic Models (GDMs) in \cite{davier2008general}. 
We generally call the main-effect HLAMs in \eqref{eq-maineff} and the all-effect HLAMs in \eqref{eq-alleff} the main-effect-based HLAMs, because they both incorporate the main effects of the latent attributes in to the model.
}% end of normalfont
\end{example}

Under the main-effect-based HLAMs, the probability mass function of the $J$-dimensional random response vector $\RR$ can be generally written as
\begin{align*}%\label{eq-pmf}
	 P(\RR=\rr\mid \QQ,\mce,\ttt^+,\ttt^-,\pp) = 
	 \sum_{\aaa\in\mca(\mce)} p_{\aaa} \prod_{j=1}^J 
	     \theta_{j,\aaa}^{r_j}
	    \times 
	     (1-\theta_{j,\aaa})^{1-r_j},
\end{align*}
where $\rr\in\{0,1\}^J$ is an arbitrary response pattern.
Notably, these main-effect-based HLAMs generally have quite different algebraic structures from the family of two-parameter HLAMs, the DINA and the DINO models.
%the family of two-parameter HLAMs, namely the DINA and the DINO models, have very different algebraic structure from the main-effect-based HLAMs in Example \ref{exp-other}.
The key structure of any two-parameter HLAM is captured by the ideal response $\Gamma_{\bq_j,\aaa}$ in \eqref{eq-and} or \eqref{eq-or}, under the ``AND'' or ``OR'' operations, respectively.
Intuitively, the two-parameter HLAMs are characterized by a probabilistic version of the \textit{Boolean product} of two groups of binary vectors, the group of $\bq_j$'s and the group of $\aaa$'s; %just as the Boolean matrix factorization method described by the \cite{ormachine2017}; 
however, this is not the case for any HLAM in Example \ref{exp-other} due to the incorporation of the main effects of attributes.
Indeed, incorporating main effects in the form of ${\sum}_{k=1}^K\beta_{j,k} q_{j,k} \alpha_k$ in \eqref{eq-maineff} or \eqref{eq-alleff} is taking a \textit{inner product} of vectors $\bq_j$, $\aaa$ and an additional $\beta$-coefficient vector, rather than the Boolean product.
Because of such distinction, the necessary and sufficient identifiability conditions derived carefully for the two-parameter HLAMs in Sections \ref{sec-main}-\ref{sec-nece} are not applicable to main-effect-based HLAMs.

Next we give a set of sufficient conditions for the identifiability of main-effect-based HLAMs.
%Despite that such HLAMs have a rather distinct nature form their two-parameter counterparts, 
The technical concept of $\Gamma(\QQ, \mce)$ (specifically, with $\Gamma = \Gamma^{\myand}$ defined in \eqref{eq-and}) introduced earlier in Section \ref{sec-main} is still useful here.
Denote the collection of all the per-item Bernoulli parameters by $\bo\Theta = (\theta_{j,\aaa})$.
We have the following theorem.

\begin{theorem}[Identifiability of HLAMs which Model the Main Effects of Attributes]\label{thm-mult}
Consider an HLAM that incorporates the main effects of the attributes with $\QQ$ and $\mce$ both unknown. 
Suppose $\bo\Theta$ satisfies a natural inequality constraint $\theta_{j,\aaa} \neq \theta_{j,\aaa'}$ if $\Gamma_{\bq_j, \aaa} \neq \Gamma_{\bq_j, \aaa'}$.
If $\Gamma(\QQ, \mce)$ satisfies the following conditions with the number of columns known, then the $(\bo\Theta, \pp)$ and $\Gamma(\QQ, \mce)$ are identifiable.
\begin{enumerate}
\item[E.] There exist two disjoint sets of items $S_1$, $S_2\subseteq[J]$, such that $\Gamma(\QQ_{S_1,\bcolon}\, , \, \mce)$ and $\Gamma(\QQ_{S_2,\bcolon}\, , \, \mce)$ each has distinct column vectors.

\item[F.] For any $\aaa \neq \aaa' \in\mca(\mce)$, there exists some item $j\not\in S_1\cup S_2$ such that $\Gamma_{\bq_j,\aaa} \neq \Gamma_{\bq_j,\aaa'}$.

\item[G.] For any $\aaa\in\mca(\mce)$ and $\aaa' \in \{0,1\}\setminus\mca(\mce)$, there exists some item $j\in[J]$ such that $\Gamma_{\bq_j, \aaa} \neq \Gamma_{\bq_j, \aaa'}$.
\end{enumerate}
In addition to the above three conditions, if $\QQ$ is known in part to contain an identity submatrix $I_K$, then the attribute hierarchy $\mce$ is identifiable from $\Gamma(\QQ, \mce)$.
\end{theorem}

%It is not hard to check that the $28\times 3$ matrix $\QQ^{\text{ECPE}}$ in the earlier Table \ref{tab-ecpeq} for the ECPE data under the linear hierarchy $\mce^{\text{ECPE}}$ also satisfies the identifiability conditions in Theorem \ref{thm-mult} for the main-effect-based HLAMs.
For the main-effect-based HLAMs, the ideal response matrix $\Gamma(\QQ, \mce)$ may not sharply characterize  the entire latent structure due to the incorporation of the main effects, which is in contrary to the DINA-based HLAMs. To see this, considering two latent patterns $\aaa$ and $\aaa'$ with $\Gamma_{\bq_j, \aaa} = \Gamma_{\bq_j, \aaa'} = 0$, then the specification in \eqref{eq-maineff} or \eqref{eq-alleff} implies that there is potentially $\theta_{j,\aaa} \neq \theta_{j,\aaa'}$.
Therefore it is hard, if at all possible, to explicitly characterize the necessary identifiability conditions in terms of $\Gamma(\QQ, \mce)$ for main-effect-based HLAMs.
However, the $\Gamma(\QQ, \mce)$ is still useful to derive sufficient conditions for identifiability, as revealed in the above Theorem \ref{thm-mult}.
This is because if $\Gamma_{\bq_j, \aaa} = \Gamma_{\bq_j, \aaa'} = 1$, the two attribute patterns $\aaa$ and $\aaa'$ both satisfy $\aaa\succeq \bq_j$ and $\aaa'\succeq \bq_j$ by the definition in \eqref{eq-and}. 
This implies both patterns $\aaa$ and $\aaa'$ possess all the attributes measured by the vector $\bq_j$.
As a result, the definition of main-effect-based models in \eqref{eq-maineff} or in \eqref{eq-alleff} shows that there must be $\theta_{j,\aaa} = \theta_{j,\aaa'}$ for these two patterns.
This intuitively explains why $\Gamma(\QQ, \mce)$ can be used to describe a set of sufficient identifiability conditions for the main-effect-based HLAMs.
%Note that it is hard, if at all possible, to explicitly characterize the necessary identifiability conditions for main-effect-based HLAMs that only depend on $\QQ$, $\mce$, and $\Gamma(\QQ, \mce)$.

%%%%%%
%The main-effect-based HLAMs with $\theta^{\text{main-eff}}_{j,\aaa}$ taking the form of \eqref{eq-maineff} is more general than two-parameter HLAMs in the sense that it allows different attribute patterns to have more diverse item parameters.
We make a remark on the relationship between the main-effect-based HLAMs and the DINA-based HLAMs studied in the previous Sections \ref{sec-main}--\ref{sec-nece}.
On the one hand, the main-effect-based HLAMs are more general than DINA-based HLAMs in the sense that the formulation of $\theta^{\text{main-eff}}_{j,\aaa}$ in \eqref{eq-maineff} or $\theta^{\text{all-eff}}_{j,\aaa}$ in \eqref{eq-alleff} can generally allow for more than two Bernoulli parameters for each $j$, while DINA-based HLAMs always have two parameters $\theta_j^+$ and $\theta_j^-$ for each $j$.
On the other hand, however, we would like to point out that in this work we still put the main focus on the DINA-based two-parameter HLAMs, 
which are widely used in the motivating applications of cognitive diagnosis in educational settings.
Indeed, these educational settings are where the attribute hierarchy receives the most attention in modeling the sequential acquisition of skill attributes \citep[e.g.,][]{leighton2004attribute, Gierl, wang2020hier}.
On the practical side, assuming the conjunctive relationship among the attributes as in DINA is often believed to be suitable for modeling the response mechanism of diagnostic test items in such settings
%exam questions 
 \citep[e.g.,][]{junker2001cognitive,de2004higher,culpepper2015bayesian}.
%: only if a student masters all the skills measured by question $j$ will he/she be considered ``capable'' of this question and otherwise ``incapable''. 
On the theoretical side, the identifiability of two-parameter DINA-based HLAMs is also more intriguing to study because of the Boolean product involved. The rich combinatorial nature of the DINA-based HLAMs gives the opportunity to close the gap between the necessity and sufficiency of identifiability requirements; interestingly, these minimal requirements are explicit conditions on the discrete structure: the $\QQ$-matrix and attribute types, as depicted in Section \ref{sec-nece}. 
Therefore, we believe that closely examining the DINA-based two-parameter HLAMs and establishing the necessary and sufficient identifiability conditions for them (as done in Sections \ref{sec-main}--\ref{sec-nece}) are highly desirable, due to their theoretical interest and practical relevance.

\color{black}

\section{Discussion}\label{sec-disc}
In this paper, we provide a first study on identifiability of the hierarchical latent attribute model, a complex-structured latent variable model popular in modeling modern assessment data.
We propose sufficient identifiability conditions that explicitly depend on the attribute hierarchy graph and the structural $\QQ$-matrix. 
We also discuss the necessity of the identifiability conditions and sharply characterize the different impacts on identifiability cast by different types of attributes in the attribute hierarchy graph.
%%%%%%%%%%%%%
In this paper we mainly focus on the basic and popular HLAMs, the DINA-based HLAMs, where each item is modeled using two parameters. 
{We also extend the theory to other types of HLAMs in Section \ref{sec-multiple}.

% For general HLAMs with multiple item parameters per item, since the two-parameter models can usually be viewed as submodels of them, the proposed conditions also serve as a necessary requirement for their identifiability. With the newly developed proof techniques for hierarchical models in this work, it would be interesting to investigate identifiability of the multi-parameter HLAMs as well as other multivariate discrete latent variable models in the future.
% {We also extend the theory to those main-effect-based HLAMs in Section \ref{sec-multiple} by providing sufficient identifiability conditions. The two-parameter HLAMs, especially the DINA model, is arguably more widely used for cognitive diagnosis in educational assessment settings compared to other modeling assumptions. This is because assuming the conjunctive ``AND'' relationship among the attributes is believed to be quite suitable for modeling the response process of exam questions: only if a student masters all the skills measured by each question will he/she be considered to be capable of this question. And these educational settings are indeed where the attribute hierarchy method receives the most attention in modeling the sequential acquisition of skill attributes \citep[e.g., see][]{leighton2004attribute, Gierl, wang2020hier}.}% end of blue

One nice implication of identifiability is the estimability of both the latent structure and the parameters that define the probabilistic model. When the proposed conditions are satisfied, all the components of the HLAM can be uniquely and consistently estimated from data based on maximum likelihood.
 {In practical data analysis under the HLAM framework, if the $\QQ$ and $\mce$ are specified by domain experts or applied researchers, then before seeing any data, one can check whether $\QQ$ and $\mce$ satisfy our proposed conditions to assess model identifiability.
On the other hand, if $\QQ$ and $\mce$ are not known and one hopes to estimate them exploratorily from data, our identifiability results can also be useful. In such scenarios, one can check whether the estimated $\widehat{\QQ}$ and $\widehat{\mce}$ satisfy necessary identifiability conditions; if not, then more careful investigation of the diagnostic test design may be needed.}
Therefore, this study provides useful insights into designing valid diagnostic tests and drawing valid scientific conclusions from assessment data under a potentially complicated attribute hierarchy.

%\bigskip

%
\singlespacing
%\section*{Acknowledgements} 
%This research was partially supported by National Science Foundation CAREER SES-1846747, DMS-1712717, SES-1659328.
%This research was also partially supported by grants R01ES027498 and R01ES028804 of the National Institutes of Environmental Health Sciences of the United States National Institutes of Health and received funding from the European Research Council under the European Union's Horizon 2020 research and innovation program (grant agreement No 856506).

%% include appendix from an external file here
% \appendix{}
% \label{sec-app}
% \input{appendix_hlam} 

%\bigskip
% 
% \begin{supplement}\label{suppA}
% \sname{Supplement to}
%   \stitle{``Identifiability of Hierarchical Latent Attribute Models"}
%  \slink[doi]{xxx}
%  \sdescription{The supplementary material contains the technical proofs of all the theoretical results presented in the paper.
%  }
%\end{supplement}

\bigskip
\bibliographystyle{apalike}
\bibliography{ref}

%%%%%%%%%%%%%%%%%%%%%%%%%%%%%%%%%%%%%%%%%%%%%%%%%%%%%%%%%%%%%%%%%%%%%%%%%%%%%%%%%%%%%%%%%%%%%%%%%%%%%%%%%%%%%%%%%%%%%%%%%%%%
\section*{Supplementary Materials}
The supplementary material contains some illustrative examples and proofs of the theory.
\par
%%%%%%%%%%%%%%%%%%%%%%%%%%%%%%%%%%%%%%%%%%%%%%%%%%%%%%%%%%%%%%%%%%%%%%%%%%%%%%%%%%%%%%%%%%%%%%%%%%%%%%%%%%%%%%%%%%%%%%%%%%%%
\section*{Acknowledgements}
%Write the acknowledgements here.
This work was  supported by NSF CAREER SES-1846747, DMS-1712717, SES-1659328; NIH NIEHS R01ES027498 and R01ES028804; and funding from the European Research Council under the European Union's Horizon 2020 research and innovation program No 856506.
%\par

%%%%%%%%%%%%%%%%%%%%%%%%%%%%%%%%%%%%%%%%%%%%%%%%%%%%%%%%%%%%%%%%%%%%%%%%%%%%%%%%%%%%%%%%%%

\newpage
\appendix{}
\label{sec-supp}
\centerline{\large\bf Supplement to ``Identifiability of Hierarchical Latent Attribute Models''}

\bigskip\bigskip
%%\vspace{2pt}
%% \centerline{\large\bf IF SECOND LINE IS NEEDED}
%%\vspace{2pt}
%% \centerline{\large\bf IF THIRD LINE IS NEEDED}
%%\vspace{.25cm}
%% \author{Yuqi Gu and Gongjun Xu}
%\vspace{.4cm} 
%\centerline{Yuqi Gu and Gongjun Xu} 
%\vspace{.4cm}
% \centerline{\it Columbia University and University of Michigan}
%\vspace{.55cm}
% \centerline{\bf Supplementary Material}
%\vspace{.55cm}
%\fontsize{9}{11.5pt plus.8pt minus .6pt}\selectfont
%\noindent
\doublespacing

In this Supplementary Material, we give several illustrative examples in Section \ref{sec-illus}. We provide the proof of Theorem \ref{thm-main} in Section \ref{pf-thm1}, the proofs of propositions in Section \ref{pf-prop14}, the proofs of Propositions \ref{prop-nece-sing}--\ref{thm-nec-dist} in Section \ref{pf-prop58},  the proofs of Theorem \ref{thm-general} and Corollary \ref{thm-connected} in Section \ref{pf-bridge},  the proofs of Proposition \ref{prop-dino}, Corollary \ref{cor-dino}, Theorem \ref{thm-mult} in Section \ref{pf-multiple}, and the proofs of the identifiability statement in Examples \ref{exp-num-attr} and \ref{exp-diff} and three technical lemmas in Section \ref{pf-example}.
\par

\setcounter{section}{0}
\setcounter{equation}{0}
\def\theequation{S\arabic{section}.\arabic{equation}}
\def\thesection{S\arabic{section}}

\fontsize{12}{14pt plus.8pt minus .6pt}\selectfont

\section{Illustrative Examples}\label{sec-illus}

\begin{example}%[Illustration of Theorem \ref{prop-nece-sing}, \ref{prop-nece-anc}, \ref{prop-nece-int}]
\label{exp-num-attr}
\normalfont{
Consider an attribute hierarchy among $K=4$ attributes: $\mathcal E=\{2\to 3,\, 3\to 4\}$. Then attribute 1, 2, 3, 4 are singleton attribute, ancestor attribute, intermediate attribute, and leaf attribute, respectively. Consider the $\QQ^{\text{id.}}$ in Fig \ref{fig-id}(a). The necessary conditions established in Propositions \ref{prop-nece-sing}--\ref{prop-nece-int} indicate that removing any solid edges from the graphical model illustration in part (c) of the figure results in nonidentifiability.
On the other hand, the hierarchical model under this $\QQ^{\text{id.}}$ is identifiable, as shown later in this document.
}
\end{example}

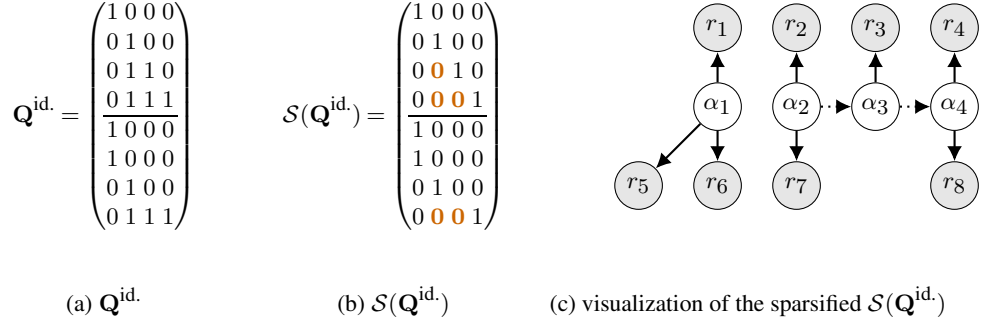
\begin{figure}[h!]
 \centering
 \begin{minipage}[c]{0.25\textwidth}
 	$$
 	\QQ^{\text{id.}}=\begin{pmatrix}
 		1 & 0 & 0 & 0\\
 		0 & 1 & 0 & 0\\
 		0 & 1 & 1 & 0\\
 		0 & 1 & 1 & 1\\
 		\hline
 		1 & 0 & 0 & 0\\
 		1 & 0 & 0 & 0\\
 		0 & 1 & 0 & 0\\
 		0 & 1 & 1 & 1\\
 	\end{pmatrix}
 	$$
 \end{minipage}
 \hfill
 \begin{minipage}[c]{0.25\textwidth}
 	$$
 	\mathcal S(\QQ^{\text{id.}})=\begin{pmatrix}
 		1 & 0 & 0 & 0\\
 		0 & 1 & 0 & 0\\
 		0 & \textcolor{orange!80!black}{\zero} & 1 & 0\\
 		0 & \textcolor{orange!80!black}{\zero} & \textcolor{orange!80!black}{\zero} & 1\\
 		\hline
 		1 & 0 & 0 & 0\\
 		1 & 0 & 0 & 0\\
 		0 & 1 & 0 & 0\\
 		0 & \textcolor{orange!80!black}{\zero} & \textcolor{orange!80!black}{\zero} & 1\\
 	\end{pmatrix}
 	$$
 \end{minipage}
 \hfill
\begin{minipage}[c]{0.46\textwidth}\centering
	\begin{tikzpicture}[scale=1.05]
    \node (h1)[hidden] at (0,0) {$\alpha_1$};
    \node (h2)[hidden] at (1,0) {$\alpha_2$};
    \node (h3)[hidden] at (2,0) {$\alpha_3$};
    \node (h4)[hidden] at (3,0) {$\alpha_4$};
    
    \draw[pre] (h2) -- (h3);
    \draw[pre] (h3) -- (h4);
    
    \node (v1)[neuron] at (0,1) {$r_1$};
    \node (v2)[neuron] at (1,1) {$r_2$};
    \node (v3)[neuron] at (2,1) {$r_3$};
    \node (v4)[neuron] at (3,1) {$r_4$};
    
    \node (v5)[neuron] at (-1,-1) {$r_5$};
    \node (v6)[neuron] at (0,-1) {$r_6$};
    \node (v7)[neuron] at (1,-1) {$r_7$};
    \node (v8)[neuron] at (3,-1) {$r_8$};
    
    \draw[qedge] (h1) -- (v1);
    \draw[qedge] (h2) -- (v2);
    \draw[qedge] (h3) -- (v3);
    \draw[qedge] (h4) -- (v4);
    
    \draw[qedge] (h1) -- (v5);
    \draw[qedge] (h1) -- (v6);
    
    \draw[qedge] (h2) -- (v7);
    \draw[qedge] (h4) -- (v8);
%%%	  
\end{tikzpicture}
\end{minipage}

%%%%%%%%%%%%%%%%%%%%%%%%%%%%%%%%%
\bigskip
\begin{minipage}[c]{0.25\textwidth}\centering
	(a) $\mathcal \QQ^{\text{id.}}$
\end{minipage}\hfill
\begin{minipage}[c]{0.25\textwidth}\centering
	(b) $\mathcal S(\QQ^{\text{id.}})$
\end{minipage}\hfill
\begin{minipage}[c]{0.46\textwidth}
	(c) visualization of the sparsified $\mathcal S(\QQ^{\text{id.}})$
\end{minipage}
\caption{A $\QQ$-matrix that gives an identifiable model under the hierarchy $\mathcal E=\{2\to 3,\, 3\to 4\}$. %with $\alpha_1$ being a singleton, $\alpha_2$ being an ancestor, $\alpha_3$ being an intermediate, and $\alpha_4$ being a leaf.  
Removing any solid edge or $r$-node (observed variable) in (c) renders nonidentifiability.}
\label{fig-id}
\end{figure}

We use the following Example \ref{exp-diff} to illustrate the different conclusions in different scenarios established in Proposition \ref{thm-nec-dist}.

\begin{example}\label{exp-diff}
\normalfont{
In Fig \ref{fig-diff}, we give two examples illustrating the conclusions of Proposition \ref{thm-nec-dist}.
In Fig \ref{fig-diff}(a), the edges in blue emanating from $\alpha_k=\alpha_1$ and the edges in red emanating from $\alpha_{\ell}=\alpha_2$ point to the same set of items {$\{r_5,r_6,r_7\}$}; in Fig \ref{fig-diff}(b), the blue (dashed) edges from $\alpha_k=\alpha_3$ and the red edges from $\alpha_{\ell}=\alpha_4$ point to the same set of items $\{r_8, r_9\}$. 
Therefore, in each of the two plots in Figure \ref{fig-diff},  attributes $k$ and $\ell$ share the same set of ``children'' in the set of items in $\{K+1,\ldots,J\}=\{5,\,6,\,7,\,8,\,9\}$; in other words, $\QQ_{(K+1):J,\,k}=\QQ_{(K+1):J,\,\ell}$.
In particular, the blue dashed edges in Fig \ref{fig-diff}(b) correspond to entries in $\QQ$ that become zero in its sparsified version.
So by Proposition \ref{thm-nec-dist}, the $\QQ$-matrix in (a) leads to a nonidentifiable model; while the $\QQ$-matrix in (b) gives an identifiable model since it satisfies the sufficient identifiability conditions in the later Theorem \ref{thm-general}.
Note that when there is no attribute hierarchy with $\mathcal E'=\varnothing$, both of the two $\QQ$-matrices visualized in Fig \ref{fig-diff}(a) and (b) would lead to a nonidentifiable model. This is because the existence of the pair of red and blue edges (including dashed ones) violates Condition C that $\mc D^{\mce}(\QQ) (=\QQ)$ should contain $K$ distinct columns in addition to an $\EE (=I_K)$; this condition is necessary for identifiability when $\mathcal E'=\varnothing$. 
}
	\begin{figure}[h!]
	\begin{minipage}[c]{0.46\textwidth}\centering
	\begin{tikzpicture}[scale=1.05]
    \node (h1)[hidden] at (0,0) {$\alpha_1$};
    \node (h2)[hidden] at (1.1,0) {$\alpha_2$};
    \node (h3)[hidden] at (2.2,0) {$\alpha_3$};
    \node (h4)[hidden] at (3.3,0) {$\alpha_4$};
    
%    \draw[pre] (h2) -- (h3);
    \draw[pre] (h3) -- (h4);
    
    \node (v1)[neuron] at (0,1.1) {$r_1$};
    \node (v2)[neuron] at (1.1,1.1) {$r_2$};
    \node (v3)[neuron] at (2.2,1.1) {$r_3$};
    \node (v4)[neuron] at (3.3,1.1) {$r_4$};
    
    \node (v5)[neuron] at (-1.1,-1.1) {$r_5$};
    \node (v6)[neuron] at (0,-1.1) {$r_6$};
    \node (v7)[neuron] at (1.1,-1.1) {$r_7$};
    \node (v8)[neuron] at (2.2,-1.1) {$r_8$};
    \node (v9)[neuron] at (3.3,-1.1) {$r_9$};
    
    \draw[qedge] (h1) -- (v1);
    \draw[qedge] (h2) -- (v2);
    \draw[qedge] (h3) -- (v3);
    \draw[qedge] (h4) -- (v4);
    
    \draw[qedge, blue!70!black] (h1) -- (v5);
    \draw[qedge, blue!70!black] (h1) -- (v6);
    \draw[qedge, blue!70!black] (h1) -- (v7);

    \draw[qedge] (h3) -- (v8);
    \draw[qedge] (h4) -- (v9);
    
    \draw[qedge, red!70!black] (h2) -- (v5);
    \draw[qedge, red!70!black] (h2) -- (v6);
    \draw[qedge, red!70!black] (h2) -- (v7);
    
%%%	  
\end{tikzpicture}
\end{minipage}
\hfill
\begin{minipage}[c]{0.46\textwidth}\centering
	\begin{tikzpicture}[scale=1.05]
    \node (h1)[hidden] at (0,0) {$\alpha_1$};
    \node (h2)[hidden] at (1.1,0) {$\alpha_2$};
    \node (h3)[hidden] at (2.2,0) {$\alpha_3$};
    \node (h4)[hidden] at (3.3,0) {$\alpha_4$};
    
    \draw[pre] (h2) -- (h3);
    \draw[pre] (h3) -- (h4);
    
    \node (v1)[neuron] at (0,1.1) {$r_1$};
    \node (v2)[neuron] at (1.1,1.1) {$r_2$};
    \node (v3)[neuron] at (2.2,1.1) {$r_3$};
    \node (v4)[neuron] at (3.3,1.1) {$r_4$};
    
    \node (v5)[neuron] at (-1.1,-1.1) {$r_5$};
    \node (v6)[neuron] at (0,-1.1) {$r_6$};
    \node (v7)[neuron] at (1.1,-1.1) {$r_7$};
    \node (v8)[neuron] at (2.2,-1.1) {$r_8$};
    \node (v9)[neuron] at (3.3,-1.1) {$r_9$};

    \draw[qedge] (h1) -- (v1);
    \draw[qedge] (h2) -- (v2);
    \draw[qedge] (h3) -- (v3);
    \draw[qedge] (h4) -- (v4);
    \draw[qedge] (h1) -- (v5);
    \draw[qedge] (h1) -- (v6);
    \draw[qedge] (h2) -- (v7);  
    
    \draw[qedge, blue!70!black, dashed] (h3) -- (v8);
    
    \draw[qedge, blue!70!black, dashed] (h3) -- (v9);
    
    \draw[qedge, red!70!black] (h4) -- (v8);
    \draw[qedge, red!70!black] (h4) -- (v9);

%%%	  
\end{tikzpicture}
\end{minipage}

\medskip
\begin{minipage}[c]{0.4\textwidth}\centering
	~~~~~~(a) identifiable, $\alpha_3$ and $\alpha_4$ are not singletons
\end{minipage}\hfill
\begin{minipage}[c]{0.4\textwidth}\centering
	(b) not identifiable as $\alpha_2\to \alpha_3$ and $\alpha_2$ is an ancestor 
\end{minipage}\hfill

\caption{Examples illustrating Proposition \ref{thm-nec-dist}. 
The dashed blue edges in (b) correspond to $\QQ$-matrix entries that become zero in the sparsified version.}
\label{fig-diff}
\end{figure}
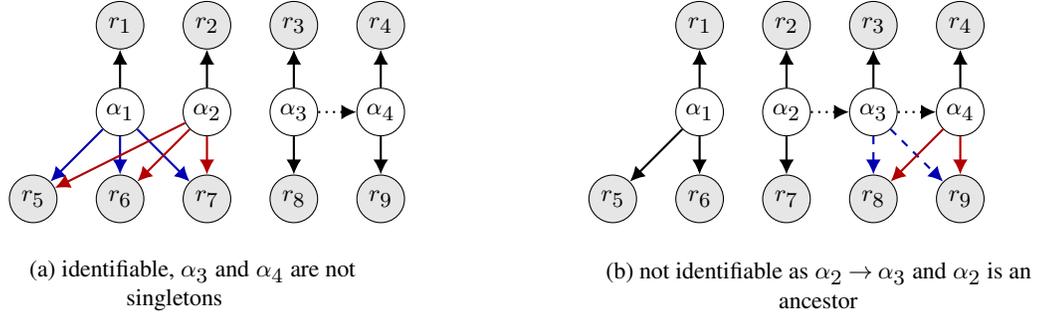
\end{example}

\section{Proof of Theorem 1}\label{pf-thm1}

%\section*{Proof of the Theorem \ref{thm-main}}

%This Appendix includes the proof of Theorem \ref{thm-main} in Section \ref{sec-main}. The proofs of the other conclusions are included in the Supplementary Material.
We introduce some notation and technical preparations before presenting the proof. Denote an arbitrary response vector by $\rr=(r_1,\ldots,r_J)$, and  write $\ee_k$ as a standard basis vector, whose $k$th element is one and the rest are zero.
For two vectors $\bo a=(a_1,\ldots,a_m)$ and $\bo b=(b_1,\ldots,b_m)$ of the same length, denote $\bo a\succeq \bo b$ if $a_i\geq b_i$ for all $i\in[m]$, and denote $\bo a\nsucceq\bo b$ otherwise. Define operations ``$\preceq$'' and ``$\npreceq$'' similarly.
For an item index $j\in[J]$, denote by $\QQ_{-j,\bcolon}$ the $(J-1)\times K$ submatrix of $\QQ$ after removing the $j$th row from $\QQ$; denote by $\QQ_{-j,k}$ the $k$th column of this submatrix $\QQ_{-j,\bcolon}$.

We next define a useful technical quantity, a $2^J\times 2^K$ marginal probability matrix $T$-matrix $T(\QQ,\ttt^+,\ttt^-)$ as follows. The rows of $T(\QQ,\ttt^+,\ttt^-)$ are indexed by all the possible response patterns $\rr\in\{0,1\}^J$ and columns by all the possible latent attribute patterns $\aaa\in\{0,1\}^K$. The $(\rr,\aaa)$-entry of the $T$-matrix is the marginal probability of a subject with attribute pattern $\aaa$ providing positive responses to the set of items $\{j:\, r_j=1\}$. Namely, denote a random response vector by $\RR$ and a random latent attribute profile by $\AA$, for arbitrary $\rr$ and $\aaa$ there is
\begin{equation}
	\label{eq-def-T}
	T_{\rr,\aaa}(\QQ,\ttt^+,\ttt^-) = \mathbb P(\RR\succeq\rr\mid \ttt^+,\ttt^-,\AA=\aaa)
	= \prod_{j:\, r_j=1} \theta_{j,\aaa},
\end{equation}
Denote the row vector of the $T$-matrix corresponding to response pattern $\rr$ by $T_{\rr,\bcolon}(\QQ,\ttt^+,\ttt^-)$, and denote the column vector of the $T$-matrix corresponding to attribute pattern $\aaa$ by $T_{\bcolon,\aaa}(\QQ,\ttt^+,\ttt^-)$.
From the above definition \eqref{eq-def-T}, it is not hard to see that $\mathbb P(\rr\mid \QQ, \ttt^+,\ttt^-,\pp) = \mathbb P(\rr\mid \bar \QQ,\bar \ttt^+,\bar \ttt^-,\bar \pp)$ for all $\rr\in\{0,1\}^J$ if and only if $T(\QQ,\ttt^+,\ttt^-)\pp = T(\bar \QQ,\bar \ttt^+,\bar \ttt^-)\bar\pp$ (which we also denote by $T\pp =  \bar  T\bar\pp$). This implies that we can focus on the $T$-matrix structure and establish identifiability by showing that $T\pp =  \bar  T\bar\pp$ gives  $(\QQ, \ttt^+,\ttt^-,\pp) = (\bar \QQ,\bar \ttt^+,\bar \ttt^-,\bar \pp)$ under certain conditions.

The $T$-matrix has another nice algebraic property, established in \cite{xu2017}, that will be frequently used in the later proof. We restate it here. The $T(\QQ,\ttt^+,\ttt^-)$ can be viewed as a map taking two general $J$-dimensional vectors $\ttt^+,\ttt^-\in\mathbb R^J$ as input. For an arbitrary $J$-dimensional vector $\ttt^\star\in\mathbb R^J$, 
 there exists a $2^K\times 2^K$ invertible matrix $D(\ttt^\star)$ that only depends on $\ttt^\star$ such that,
\begin{equation}
	\label{eq-algebra}
	T(\QQ,\ttt^+-\ttt^\star,\ttt^- -\ttt^\star)
	=D(\ttt^\star) T(\QQ,\ttt^+,\ttt^-).
\end{equation}

\begin{proof}[\textbf{Proof of Theorem \ref{thm-main}}]	
\textit{We first show the sufficiency of Conditions $A$, $B$ and $C$ for identifiability of $(\Gamma(\QQ,\mce),\,\ttt^+,\,\ttt^-,\,\pp)$.}
Since Condition $A$ is satisfied, from now on we assume without loss of generality that
\begin{equation}
    \QQ =  \begin{pmatrix}
       \QQ^0\\
       \QQ^\star
\end{pmatrix},\quad \Gamma(\QQ^0,\mce) = \Gamma(I_K,\mce).
\end{equation}

We next show that if 
for any $\rr\in\{0,1\}^J$,
\begin{equation}\label{eq-def}
T_{\rr,\bcolon}(\QQ, \ttt^+,\ttt^-) \pp = T_{\rr,\bcolon}(\bar \QQ,  \bar\ttt^+,\bar\ttt^-)\bar\pp,
\end{equation}
then $\Gamma(\bar\QQ,\bar\mce) = \Gamma(\QQ,\mce)$ and $(\bar\ttt^+,\bar\ttt^-,\bar\pp) = (\ttt^+,\ttt^-,\pp)$. We denote the submatrix of $\bar\QQ$ consisting of its first $K$ row vectors by $\bar\QQ^0$, and the remaining submatrix by $\bar\QQ^\star$, so $\bar\QQ = ((\bar\QQ^0)^\top, (\bar\QQ^\star)^\top)^\top$.

For any item set $S\subseteq\{1,\ldots,J\}$, denote $\ttt^+_S=\sum_{j\in S}\theta^+_j\ee_j$, %$\ttt^-_S = \sum_{j\in S}\theta^-_j\ee_j$,
 and denote $\ttt^-_S$, $\bar\ttt^+_S$, and $\bar\ttt^-_S$ similarly; here $\ee_j$ represents a $J$-dimensional standard basis vector with the $j$th entry being one. Consider the response pattern $\rr^\star= \sum_{j\in S}\ee_j$ and any $\ttt^\star = \sum_{j\in S}\theta^\star_j\ee_j$, then the matrix transformation property in Eq.~\eqref{eq-def-T} together with Eq.~\eqref{eq-def} implies that
\begin{equation}\label{eq-tra}
T_{\rr^\star,\bcolon}( \QQ,\ttt^+_S-\ttt^\star,\ttt^-_S-\ttt^\star)\pp = T_{\rr^\star,\bcolon}(\bar \QQ,\bar\ttt^+_S-\ttt^\star,\bar\ttt^-_S-\ttt^\star)\bar\pp.
\end{equation} 
When it causes no ambiguity, we will sometimes  denote $T_{\rr^\star,\bcolon}( \QQ,\ttt^+_S-\ttt^\star,\ttt^-_S-\ttt^\star)=T_{\rr^\star,\bcolon}$ for notational simplicity.

We prove the theorem in 6 steps as follows.

\medskip
\noindent
\textbf{Step 1.} In this step we show if \eqref{eq-def} holds, the $\bar\QQ^0$ must take the following upper-triangular form with all diagonal elements being one, up to a column permutation.
\begin{equation}\label{eq-trian}
\bar \QQ^0=
    \begin{pmatrix}
    1 & * & \dots  & * \\
    0 & 1 & \dots  & * \\
    \vdots & \vdots & \ddots & \vdots \\
    0 & 0 & \dots  & 1
\end{pmatrix}.
\end{equation}

We need the following useful lemmas, whose proofs are presented in the Supplementary Material.
\begin{lemma}\label{lem-basic} 
The following statements about $\QQ$, $\mc D^{\mce}(\QQ)$, and $\mc S^{\mce}(\QQ)$ hold.
\begin{itemize}
    \item[(a)] If $\QQ$ satisfies Conditions $A$ with the first $K$ rows forming the $\QQ^0$, then for any $k, h\in[K]$ and $k\neq h$,  $\bq_k\succeq\bq_h$ happens only if $k\to h$.
    \item[(b)] Suppose $\QQ$ satisfies Condition $C$. 
    If $k\to h$ under the attribute hierarchy, then the $\QQ^{\star,C}:=\mc D^{\mce}(\QQ^{\star})$ defined in Condition $C$ must satisfy $\QQ^{\star,C}_{\bcolon,k} \succ  \QQ^{\star,C}_{\bcolon,h}$.
\end{itemize}
\end{lemma}
%The proof of Lemma \ref{lem-basic} is straightforward and hence is omitted.

From now on, we denote by $\QQ^{C}:=\mathcal D^{\mce}(\QQ)$ the densified version of $\QQ$ under the hierarchy $\mce$ and denote entries of $\QQ^C$ by $q_{j,k}^{\dense}$.
\begin{lemma}\label{lem-order}
Suppose the true $\QQ$ satisfies Conditions $A$, $B$, $C$ under the attribute hierarchy $\mce$ with $\mathcal D^{\mce}(\QQ_{1:K,\bcolon})=\EE$. 
If there exists an item set  $S\subseteq\{K+1,\ldots,J\}$ such that 
$$
\max_{m\in S}q^{\dense}_{m,h}=0,\quad \max_{m\in S} q^{\dense}_{m,j}=1~~\forall j\in\mathcal J
$$ 
for some attribute $h\in[K]$ and a set of attributes $\mathcal J\subseteq[K]\setminus\{h\}$, then 
$$
\vee_{j\in\mathcal J}\,\bar\bq_j\nsucceq\bar\bq_h.
$$
\end{lemma}

We now proceed with the proof of Step 1.
%Since the column vectors of $\tilde Q^\star$ are distinct, 
We first introduce the concept of the \textit{lexicographic order} between two vectors of the same length. For two binary vectors $\bo a=(a_1,\ldots,a_L)^\top$ and $\bo b=(b_1,\ldots,b_L)^\top$ both of length $L$, we say $\bo a$ is of smaller lexicographic order than $\bo b$ and denote by $\bo a\prec_{\text{lex}}\bo b$, if either $a_1<b_1$, or there exists some $l\in\{2,\ldots,L\}$ such that $a_l<b_l$ and $a_m=b_m$ for all $m=1,\ldots,l-1$. 
Since $\QQ^{\star,C}:=\mathcal D^{\mce}(\QQ^\star)$ contains $K$ distinct column vectors, the $K$ columns of $\QQ^{\star,C}$ can be arranged in an increasing lexicographic order. 
Without loss of generality, we assume that
\begin{equation}\label{eq-qlex}
\QQ^{\star,C}_{\bcolon,1}
\prec_{\text{lex}}
\QQ^{\star,C}_{\bcolon,2}
\prec_{\text{lex}}
\cdots
\prec_{\text{lex}}
\QQ^{\star,C}_{\bcolon,K}.
\end{equation}

We use an induction method to prove the conclusion.
First consider attribute $1$. Since $\QQ^{\star,C}_{\bcolon,1}$ has the smallest lexicographic order among the columns of $\QQ^{\star,C}$, there must exist an item set $S\subseteq\{K+1,\ldots,J\}$ such that 
$$
q^{\dense}_{S,1}=0,\quad q^{\dense}_{S,\ell}=1~~\forall \ell=2,\ldots,K.
$$
Based on the above display, we apply Lemma \ref{lem-order} to obtain %$\bar\bq_{\ell}\nsucceq\bar\bq_{1}$ for all $\ell=2,\ldots,K$. 
$
\vee_{\ell=2}^K\bar\bq_\ell \nsucceq \bar\bq_1.
$
This means  there exists $b_1\in[K]$ such that the $b_1$-th column vector of $\bar \QQ^0$ must equal the basis vector $\ee_1$, %$$(\underbrace{1}_{\text{column }1},\zero)^\top=\ee_{1},$$ 
i.e., we have $\bar \QQ^0_{\bcolon,b_1}=\ee_{1}$. 

Now we assume as the inductive hypothesis that for $h\in[K]$ and $h>1$, we have a distinct set of attributes $\{m_1,\ldots,m_{h-1}\}\subseteq[K]$ such that their corresponding column vectors in $\bar\QQ_{1:K,\bcolon}$ satisfy 
\begin{equation}\label{eq-induc-q1}
\forall i=1,\ldots,h-1,\quad \bar\QQ_{1:K,b_i}=(*,\ldots,*,\underbrace{1}_{\text{column }i}, 0,\ldots,0)^\top.
\end{equation}
 Now we consider attribute $h$. By \eqref{eq-qlex}, the column vector $\QQ^{\star,C}_{\bcolon,h}$ has the smallest lexicographic order among the $K-h-1$ columns in $\{\QQ^{\star,C}_{\bcolon,h},\,\allowbreak  \QQ^{\star,C}_{\bcolon,h+1},\,\allowbreak \ldots, \QQ^{\star,C}_{\bcolon,K}\}$, therefore similar to the argument in the previous paragraph, there must exist an item set $S\subseteq\{K+1,\ldots,J\}$ such that 
\begin{equation}\label{eq-sigma-hk}
q^{\dense}_{S,h}=0,\quad q^{\dense}_{S,\ell}=1~~\forall \ell=h+1,\ldots,K.
\end{equation}
Therefore Lemma \ref{lem-order} gives 
$
\vee_{\ell=h+1}^K\bar\bq_{\ell}\nsucceq \bar\bq_{h},
$
which further implies there exists an attribute $b_h$ such that
\begin{equation}\label{eq-mkh}
%\begin{cases}
\max_{\ell\in\{h+1,\ldots,K\}}\bar  q_{\ell,b_h}=0,\quad
\bar q_{h,b_h}=1.
%\end{cases}
\end{equation}
We point out that $b_h\not\in\{b_1,\ldots,b_{h-1}\}$, because by the induction hypothesis \eqref{eq-induc-q1} we have  $\bar q_{h,b_i}=0$ for $i=1,\ldots,h-1$.
So $\{b_1,\ldots,b_{h-1},b_h\}$ contains $h$ distinct attributes.
Furthermore, \eqref{eq-mkh} gives that 
$$\bar \QQ^0_{\bcolon,b_h} = (*,\ldots,*,\allowbreak\underbrace{1}_{\text{column }h},\allowbreak 0,\ldots,0)^\top,$$
which generalizes \eqref{eq-induc-q1} by extending $h-1$ there to $h$.
Therefore, we use the induction argument to obtain
$$\forall k\in[K],\quad
\bar \QQ^0_{\bcolon,b_k} = (*,\ldots,*,\underbrace{1}_{\text{column }k}, 0,\ldots,0)^\top,$$
which   means
\begin{equation}\label{eq-qmk}
\bar \QQ^0_{\bcolon,\,(b_1,\ldots,b_K)} =
\begin{pmatrix}
    1 & * & \dots  & * \\
    0 & 1 & \dots  & * \\
    \vdots & \vdots & \ddots & \vdots \\
    0 & 0 & \dots  & 1
\end{pmatrix},
\end{equation}
and the conclusion of Step 1 in \eqref{eq-trian} is proved.

\medskip
\noindent
\textbf{Step 2.} In this step we prove $\bar \theta^+_j=\theta^+_j$ for all $j\in\{K+1,\ldots,J\}$. It can be proved in the same way as Step 2 of the proof of Theorem 1 in \cite{id-Q}, and we omit the details here. Note that the fact that $p_{\one}>0$ holds under any attribute hierarchy is used here.

\medskip
\noindent
\textbf{Step 3.} 
In this step we use induction to prove $\bar \theta^-_k=\theta^-_k$ for all $k\in[K]$ and $\bar \QQ_{1:K,\bcolon} \stackrel{\mce}{\sim} I_K.$

\medskip
\noindent\textbf{Step 3.1.}
First consider those attribute $k$ for which there \textit{does not exist} another attribute $h$ such that $ \QQ^{\star, C}_{\bcolon,h}\prec  \QQ^{\star, C}_{\bcolon,k}$; and we first aim to show $\bar \theta^-_k=\theta^-_k$ for such $k$.
By part (b) of Lemma \ref{lem-basic}, we have that $k\not\to h$ for any attribute $h\neq k$. 
For this $k$, define
\begin{equation}\label{eq-ttt}
\ttt^\star = \sum_{j=1}^K \bar \theta^-_j \ee_j + \sum_{j>K:\, q_{j,k}=0} \theta^-_j\ee_j + \sum_{j>K:\, q_{j,k}=1} \theta^+_j\ee_j,
\end{equation}
then $T_{\rr^\star,\bcolon}(\bar \QQ,\bar\ttt^+-\ttt^\star,\bar\ttt^--\ttt^\star)=\zero$. Further, we claim $T_{\rr^\star,\bcolon}( \QQ,\ttt^+-\ttt^\star,\ttt^--\ttt^\star)$ would equal zero for any $\aaa\neq
\one_K-\ee_k
=:\aaa^\star$, so the only potentially nonzero element in $T_{\rr^\star,\bcolon}$ is $ T_{\rr^\star,\aaa^\star}$. More specifically,
\begin{align}\label{eq-nzero}
&~ T_{\rr^\star,\aaa}( \QQ,\ttt^+-\ttt^\star,\ttt^--\ttt^\star)\\ \notag 
&= ~ \begin{cases}
~(\theta^-_k - \bar \theta^-_k) 
\underset{j\leq K:\, j\neq k} \prod (\theta^+_j - \bar \theta^-_j) \\
\quad\times \underset{j>K:\, q_{j,k}=0} \prod(\theta^+_j-\theta^-_j)
\underset{j>K:\, q_{j,k}=1} \prod(\theta^-_j-\theta^+_j),& \aaa=\aaa^\star;\\
~0,& \aaa\neq\aaa^\star.
\end{cases}
\end{align}
The reasoning behind \eqref{eq-nzero} is as follows.
Consider any other attribute pattern $\aaa\neq\aaa^\star$ with $\alpha_h=0$ for some $h\neq k$. Since for $k$ we have $ \QQ^{\star,C}_{\bcolon,k}\nsucceq  \QQ^{\star,C}_{\bcolon,h}$ for any $h\neq k$, there must exist some item $j>K$ s.t. $q_{j,k}=0$ and $q_{j,h}=1$. For this particular item $j$, we have $T_{\rr^\star,\aaa}$ contains a factor of $(\theta_{j,\aaa}-\theta^-_j)=(\theta^-_j-\theta^-_j)=0$, so $T_{\rr^\star,\aaa}=0$. This shows that $T_{\rr^\star,\aaa}\neq 0$ only if $\alpha_h=1$ for all $h\neq k$.
Furthermore, we claim that  $T_{\rr^\star,\one_K}= 0$ also holds; this is because there exists $j>K$ s.t. $q^{\sparse}_{j,k}=1$ (recall that $q^{\sparse}_{j,k}$ is the $(j,k)$th entry of the sparsified $\mathcal S^{\mce}(\QQ)$), and for this particular item $j$ we have $\theta_{j,\one_K}=\theta^+_j$ so $T_{\rr^\star,\one_K}$ contains a factor of $(\theta^+_j-\theta^+_j)=0$. Now we have shown \eqref{eq-nzero} holds. 
Equation \eqref{eq-tra} leads to
\begin{align}\label{eq-induc1}
0=&~\sum_{\aaa\in\mca_0} T_{\rr^\star,\aaa}p_{\aaa} 
=T_{\rr^\star,\aaa^\star}p_{\aaa^\star}\\ \notag
=&~ (\theta^-_k - \bar \theta^-_k) \prod_{j\leq K:\atop j\neq k}(\theta^+_j - \bar \theta^-_j) \prod_{j>K:\atop q_{j,k}=0}(\theta^+_j-\theta^-_j)\prod_{j>K:\atop q_{j,k}=1}(\theta^-_j-\theta^+_j) p_{\aaa^\star}.
\end{align}
We claim that $\aaa^\star$ respects the attribute hierarchy so $p_{\aaa^\star}>0$. This is true because we have shown earlier $k\not\to h$ for any attribute $h\neq k$. Therefore in \eqref{eq-induc1} the only factor that could potentially be zero is $(\theta^-_k - \bar \theta^-_k)$, and we obtain $\bar \theta^-_k=\theta^-_k$.
This completes the first step of the induction.

\medskip
\noindent\textbf{Step 3.2.}
Now as the inductive hypothesis, we consider attribute $k$ and assume that for any other attribute $h$ s.t. $\mathcal D^{\mce}(\QQ^\star_{\bcolon,h})\prec \mathcal D^{\mce}(\QQ^\star_{\bcolon,k})$, we already have $\bar \theta^-_h=\theta^-_h$. 
Recall $\mathcal H_k=\{h\in[K]\setminus\{k\}:\,k\to h\}$ denotes all the attributes that have higher level in the attribute hierarchy than attribute $k$.
By part (b) of Lemma \ref{lem-basic}, this implies for any $h\in\mt H_k$, we have $\bar \theta^-_h=\theta^-_h$.
Also, by Condition $B$ in the theorem, there exist two items $j_1, j_2>K$ s.t. $q_{j_i,k}=1$ and $q_{j_i,h}=0$ for all $h\in\mt H_k$, for $i=1,2$.

Before proceeding with the proof of $\bar \theta^-_k = \theta^-_k$, we need to introduce a useful lemma, whose proof is provided in the Supplementary Material.
\begin{lemma}\label{lem-ch}
Suppose the true $\QQ$ satisfies Conditions $A$, $B$, $C$ under the attribute hierarchy $\mce$. If $\vee_{h\in\mathcal K}\,\bar \bq_h\succeq \bar\bq_m$ for some $\mathcal K \subseteq [J]$, some $m\in[J]\setminus\mathcal K$ and 
$\big|(\mathcal K\cup \{m\})\cap \{K+1,\ldots,J\}\big|\leq 1$, 
%$\#[(\mathcal K\cup \{m\})\cap \{K+1,\ldots,J\}]\leq 1$, 
then $\bar \theta^+_m=\theta^+_m$.
\end{lemma}

By Condition $B$, there exist two different items $j_1, j_2>K$ s.t. $q^{\sparse}_{j_i,k}=1$ and $q^{\sparse}_{j_i,h}=0$ for all $h\in\mt H_k$ for $i=1,2$; note that there are also $q_{j_i,k}=1$ and $q_{j_i,h}=0$. We next aim to show that in $\bar\QQ$, we must also have $\bar\bq_{j_i,h}=0$ for all $h\in\mt H_k$ for $i=1,2$.
We prove this in two steps.

\smallskip
\noindent
\textbf{Step 3.2 Part I.}
First, we use proof by contradiction to show the $\bar \bq_{h}$ satisfies that, for any attribute $m\not\to h$ the following holds,
%$\bar\bq_h\nsucceq\bar\bq_m$ for any attribute $m$ s.t. $m\not\to h$. 
\begin{equation}\label{eq-qhm}
    \max\Big( \max_{\ell\in[K]:\,  \QQ^\star_{\bcolon,\ell} \nprec   \QQ^\star_{\bcolon,m}} \bar\bq_{\ell},~
    \bar\bq_h\Big)
    \nsucceq\bar \bq_m,
\end{equation}
where the $\max$ operator applied to vectors of the same length means taking the element-wise maximum of the vectors and obtaining a new vector of that same length.
Suppose \eqref{eq-qhm} does not hold, then applying Lemma \ref{lem-ch} we obtain $\bar \theta^+_m=\theta^+_m$. Note that we also have $\bar \theta^-_h = \theta^-_h$ by the inductive hypothesis. Define 
\begin{equation}\label{eq-qhm-1}
    \ttt^\star = \bar \theta^-_h\ee_h 
    + \sum_{\ell\leq K:\, \QQ^\star_{\bcolon,\ell} \nprec   \QQ^\star_{\bcolon,m}}\bar \theta^-_\ell\ee_\ell
    + \bar \theta^+_m\ee_m
    + \sum_{j>K:\, q_{j,m}=0} \theta^-_j\ee_j,
\end{equation}
then with this $\ttt^\star$, we claim that the RHS of \eqref{eq-tra} is zero, $ \bar  T_{\rr^\star,\bcolon}\bar\pp=0$. This claim is true because $ \bar  T_{\rr^\star,\aaa}$ contains a factor $f_{\aaa}$ of the following form
\begin{align*}
& f_{\aaa}=
(\bar\theta_{h,\aaa} - \bar \theta^-_h)
\prod_{l:\,\QQ^\star_{\bcolon,\ell}\nprec \QQ^\star_{\bcolon,m}}
(\bar\theta_{\ell,\aaa} - \bar \theta^-_\ell)(\bar\theta_{m,\aaa} - \bar \theta^+_m)\neq 0 \quad \text{only if}\\
&\aaa\succeq\max\Big( \max_{\ell\in[K]: \, \QQ^\star_{\bcolon,\ell}  \nprec   \QQ^\star_{\bcolon,m}} \bar\bq_{\ell},~\bar\bq_h\Big)~\text{ and }~
\aaa\nsucceq\bar\bq_m,
\end{align*}
which is impossible because of \eqref{eq-qhm}, so $f_{\aaa}=0$ and $ \bar  T_{\rr^\star,\aaa}=0$ for all $\aaa$. Therefore by \eqref{eq-tra} we have $T_{\rr^\star,\bcolon}\pp =  \bar  T_{\rr^\star,\bcolon}\bar\pp=0$. Note that $\bar \theta^-_h=\theta^-_h$ and $\bar \theta^+_m=\theta^+_m$, and now we consider the term $T_{\rr^\star,\aaa}$. Then due to the last term in $\ttt^\star$ defined in \eqref{eq-qhm-1}, we have $T_{\rr^\star,\aaa}\neq 0$ only if $\aaa\succeq\bq_j$ for all $j>K$ s.t. $q_{j,m}=0$. We claim that such $\aaa$ must also satisfy $\aaa\succeq\bq_\ell$ for any $\ell\leq K$ s.t. $\QQ^\star_{\bcolon,\ell} \nprec   \QQ^\star_{\bcolon,m}$. This is because for any $\ell\leq K$ s.t. $\QQ^{\star,C}_{\bcolon,\ell} \nprec   \QQ^{\star,C}_{\bcolon,m}$, there must exist an item $j>K$ such that $q^{\dense}_{j,m}=0$ and $q^{\dense}_{j,\ell}=1$, then the fact that $\aaa\succeq\bq_j$ for this $j$ ensures $\alpha_\ell=1$ and $\aaa\succeq\bq_\ell$ (recall $\bq_\ell\stackrel{\mce}{\sim}\ee_\ell$). Therefore,
\begin{align}\label{eq-qhm-tra}
 T_{\rr^\star,\aaa} =  
 \begin{cases}
(\theta^+_h - \bar \theta^-_h)
\underset{\ell\leq K:\, \QQ^{\star,C}_{\bcolon,\ell}\nprec \QQ^{\star,C}_{\bcolon,m}}
 \prod (\theta^+_\ell - \bar \theta^-_\ell)
(\theta^-_m - \bar \theta^+_m)\\
\qquad \times \underset{j>K: \, q_{j,m}=0} \prod
(\theta^+_j - \theta^-_j), &\text{if }\aaa\in\mca_1 ;\\
0, & \text{otherwise}.
\end{cases}
\end{align}
where 
\begin{align*}
\mca_1=  &~
\{\aaa\in\mca:\, \aaa\succeq\bq^{\dense}_j~\text{s.t.}~q^{\dense}_{j,m}=0;\,\aaa\succeq\bq^{\dense}_h;\,\aaa\nsucceq\bq^{\dense}_m\}\\
=&~
\{\aaa\in\mca:\, \alpha_\ell=1~\text{for all}~\ell~\text{s.t.}~\QQ^{\star,C}_{\bcolon,\ell}\nprec \QQ^{\star,C}_{\bcolon,m};\,\alpha_h=1;\,\alpha_m=0\}.  
\end{align*}
We claim that there exists some attribute pattern in $\mca_1$ that respects the attribute hierarchy, i.e., there exists $\aaa^\star\in\mca_1$ with $p_{\aaa^\star}>0$. This can be seen by  noting the following two facts: first, the assumption $m\not\to h$ in the beginning of the current Step 3.2.1 yields that an $\aaa$ with $\alpha_m=0$ and $\alpha_h=1$ does not violate the attribute hierarchy; second, an $\aaa$ satisfying  $\alpha_\ell=1~\text{for all}~\ell$ s.t. $\QQ^{\star,C}_{\bcolon,\ell}\nprec \QQ^{\star,C}_{\bcolon,m}$ also does not contradict $\alpha_m=0$ under the hierarchy, because by part (b) of Lemma \ref{lem-basic}, if $\QQ^{\star,C}_{\bcolon,\ell}\nprec \QQ^{\star,C}_{\bcolon,m}$ then $m\not\to h$. 
Now we have proven the claim there exists $\aaa^\star\in\mca_1$ with $p_{\aaa^\star}>0$. Combined with \eqref{eq-qhm-tra}, we obtain
$$
(\theta^+_h - \bar \theta^-_h)
\prod_{\ell\leq K:\,  \QQ^\star_{\bcolon,\ell}\nprec \QQ^\star_{\bcolon,m}}
(\theta^+_\ell - \bar \theta^-_\ell)
(\theta^-_m - \bar \theta^+_m)
\prod_{j>K:\, q_{j,m}=0}
(\theta^+_j - \theta^-_j)\Big(\sum_{\aaa\in\mca_1}p_{\aaa}\Big)=0
%\quad\sum_{\aaa\in\mca_1}p_{\aaa}\succeq p_{\aaa^\star}>0.
$$
and $\sum_{\aaa\in\mca_1}p_{\aaa}\geq p_{\aaa^\star}>0$. This gives a contradiction because each factor in the above display is nonzero.
Now we have reached the goal of Step 3.2.1 of proving \eqref{eq-qhm}.

We remark here that \eqref{eq-qhm} has some nice consequences. Considering the $K\times K$ matrix $\bar \QQ^0_{\bcolon,(b_1,\ldots,b_K)}$ in \eqref{eq-trian} shown in Step 1 and the particular attribute $h$, we actually have obtained that for any $m\not\to h$, the $m$-th column of $\bar \QQ^0_{\bcolon,(b_1,\ldots,b_K)}$ not only has the last $(K-m)$ entries equal to zero, but also has $\bar \QQ^0_{h,b_m}=0$.
Equivalently, considering the columns of $\bar \QQ$ are arranged just in the order $(b_1,\ldots,b_K)$ without loss of generality, we have \begin{equation}
    \label{eq-idq}
    \bar q_{h,m}=0 ~\text{for any attribute}~m\not\to h. 
\end{equation}

\smallskip
\noindent
\textbf{Step 3.2 Part II.}
In this step we use proof by contradiction to show that for $i=1$ and $2$ there is
\begin{equation}
    \label{eq-qji}
    \max\Big(\max_{\ell\leq K:\,\ell\to h}\bar \bq_\ell,
    ~\bar\bq_{j_i} \Big)
    \nsucceq \bar\bq_h.
\end{equation}
Suppose \eqref{eq-qji} does not hold for $i=1$, i.e., $\max(\max_{\ell\leq K:\,\ell\to h}\bar \bq_\ell, ~\bar\bq_{j_1} )\succeq \bar\bq_h$. Then by Lemma \ref{lem-ch} we have $\bar \theta^+_h = \theta^+_h$. 
We define 
\begin{equation}
    \label{eq-qji-1}
    \ttt^\star = \bar \theta^+_h\ee_h + \sum_{\ell\leq K:\,\ell\to h} \bar \theta^-_\ell\ee_\ell
    %+ \sum_{m\leq K:\atop m\not\to h}\theta^-_m\ee_m
    + \bar \theta^-_{j_1}\ee_{j_1} %+ \theta^-_{j_2}\ee_{j_2}
    + \sum_{j>K:\,j\neq j_1,\, q_{j,h}=0} \theta^-_j\ee_j,
\end{equation}
and note that the item $j_2$ is included in the last term of summation above since $q_{j_2,h}=0$.
%Note that the two summation terms in the above definition do not conflict each other, i.e., $\{\ell\leq K:\ell\to h\}\cap \{m\leq K: q_{j_1,m}=1\}=\varnothing$, because Condition $B$ implies $q_{j_1,\ell}=0$ for any 
With $\ttt^\star$ defined as in \eqref{eq-qji-1}, we have $ \bar  T_{\rr^\star,\aaa}=0$ for all $\aaa$ because of the first three terms in \eqref{eq-qji-1} and the assumption that $\max(\max_{\ell\leq K:\,\ell\to h}\bar \bq_\ell, \allowbreak\bar\bq_{j_1} )\succeq \bar\bq_h$. So \eqref{eq-tra} gives $T_{\rr^\star,\bcolon}\pp= \bar  T_{\rr^\star,\bcolon}\bar\pp=0$. 
Consider $T_{\rr^\star,\aaa}$, then $T_{\rr^\star,\aaa}\neq 0$ only if $\aaa\nsucceq\bq_h$ and $\aaa\succeq\bq_{j_2}$ because of the terms $\bar \theta^+_h\ee_h$ and $\theta^-_{j_2}\ee_{j_2}$ included in $\ttt^\star$ defined in \eqref{eq-qji-1}. 
Further, because of the last term in $\ttt^\star$ defined in \eqref{eq-qji-1}, we have $T_{\rr^\star,\aaa}\neq 0$ only if $\aaa$ satisfies $\alpha_k=1$, $\alpha_h=0$, and
%\begin{equation*}
$\alpha_m=1~\forall m~\text{s.t.}~\exists j>K,\, j\neq j_1,\, q_{j,h}=0,\, q_{j,m}=1$,
%\end{equation*}
or equivalently,
\begin{equation}
    \label{eq-qji-2}
\alpha_m=1~\forall m~\text{s.t.}~\QQ^{\star,C}_{-j_1,m}\nprec \QQ^{\star,C}_{-j_1,h}.
\end{equation}
We claim that any such $\aaa$ satisfying $\alpha_k=1$, $\alpha_h=0$, and \eqref{eq-qji-2} also satisfies $\aaa\succeq\bq^{\dense}_{j_1}$, because of the reasoning as follows.
We next show $\alpha_b\geq q^{\dense}_{j_1,b}$ for all attribute $b$. 
Define
    \begin{equation}\label{eq-double}
        \ttt^{\star\star} = \ttt^\star\text{ in }\eqref{eq-qji-1} + \sum_{b\leq K:\, k\not\to b,\, \QQ^{\star,C}_{-j_1,b}\prec \QQ^{\star,C}_{-j_1,h}} \theta^-_b\ee_b,
    \end{equation}
    and with this $\ttt^{\star\star}$ and its corresponding response pattern $\rr^{\star\star}$, we still have $ \bar  T_{\rr^{\star\star},\bcolon}\bar\pp=0$ and hence $T_{\rr^{\star\star},\bcolon}\pp=0$. The $T_{\rr^{\star\star},\aaa}\neq 0$ only if $\aaa$ satisfies
    \begin{align}\label{eq-prop}
    \begin{cases}
    \alpha_k=1,~\alpha_h=0,\\
    \alpha_m=1~\forall m~
        \text{s.t.}~\QQ^{\star,C}_{-j_1,m}\nprec \QQ^\star_{-j_1,h},\\
    \alpha_b=1~\forall b~
        \text{s.t.}~\QQ^{\star,C}_{-j_1,b}\prec \QQ^{\star,C}_{-j_1,h}\text{ and }k\not\to b.
    \end{cases}
    \end{align}
We denote the set of attribute patterns having the above properties by $\mca_2=\{\aaa\in\{0,1\}^K:\, \aaa\text{ satisfies }\eqref{eq-prop}\}$.
Note the following two things: (i) first, $\QQ^{\star,C}_{-j_1,m}\nprec \QQ^{\star,C}_{-j_1,h}$ implies $m\not\to h$, because otherwise by Lemma \ref{lem-basic} there is  $\QQ^{\star,C}_{\bcolon,m}\prec \QQ^{\star,C}_{\bcolon,h}$ and hence $\QQ^{\star,C}_{-j_1,m}\prec \QQ^{\star,C}_{-j_1,h}$;
(ii) second, $k\not\to b$ implies $h\not\to b$, since otherwise $h\to b$ and $k\to h$ would imply $k\to b$.
And we have the conclusion that there exists some $\aaa^\star\in\mca_2$ that respects the attribute hierarchy with $p_{\aaa^\star}>0$, because $\alpha_h=0$ does not contradict any $\alpha_\ell=1$ as specified in \eqref{eq-prop} according to (i) and (ii).
We next show that for $\aaa\in\mca_2$, $\alpha_b\geq q_{j_1,b}$ for any $b$ must hold. 
To show this we only need to consider those $b$ such that $q_{j_1,b}=1$ and show any $\aaa\in\mca_2$ must have $\alpha_b=1$ for such $b$. By Condition $B$, $q_{j_1,b}=1$ implies $b\not\in\mt H_k$ (i.e., $k\not\to b$). Then for such $b$, if $\QQ^{\star,C}_{-j_1,b}\nprec \QQ^{\star,C}_{-j_1,h}$, then by \eqref{eq-prop} we have $\alpha_b=1$; and if $\QQ^{\star,C}_{-j_1,b}\prec \QQ^{\star,C}_{-j_1,h}$, combining the fact that $k\not\to h$, by \eqref{eq-prop} we also have $\alpha_b=1$. So the conclusion that $\aaa\in\mca_2$, $\alpha_b\geq q_{j_1,b}$ for any $b$ is reached.

Now we have obtained for $\aaa\in\mca_2$ there is $\aaa\succeq\bq_{j_1}$. This results in $\aaa\succeq \bq_j$ for any $j>K$ s.t. $q_{j,h}=0$, i.e, $\aaa\succeq\max_{j>K:\,q_{j,h}=0}\bq_j$.
We further claim that for any $\aaa\in\mca_2$, the $\aaa\succeq\bq_\ell$ for all $\ell\to h$ must hold. This is because by Condition $B$, for any $\ell\to h$ there exists $j>K$ such that $q_{j,h}=0$ and $q_{j,\ell}=1$. And combining with the previously obtained $\aaa\succeq\max_{j>K:q_{j,h}=0}\bq_j$, we have the conclusion that $\alpha_\ell=1$ and $\aaa\succeq\bq_\ell$. Therefore $\aaa\succeq\bq_\ell$ for all $\ell\to h$.
Considering the $T_{\rr^{\star\star},\bcolon}\pp=0$ with $\ttt^{\star\star}$ defined in \eqref{eq-double}, we have
$$
(\theta^-_h-\bar \theta^+_h)\prod_{\ell\leq K:\, \ell\to h}(\theta^+_\ell - \bar \theta^-_\ell)(\theta^+_{j_1} - \bar \theta^-_{j_1})
\prod_{j>K:\,j\neq j_1,\, q_{j,h}=0}(\theta^+_j-\theta^-_j)
\Big(\sum_{\aaa\in\mca_2}p_{\aaa}\Big)=0.
$$
This leads to a contradiction, since every factor in the above display is nonzero. Now we have reached the goal of Step 3.2.2 of proving \eqref{eq-qji} for $i=1$, and using the exactly same argument gives \eqref{eq-qji} for $i=2$.

Combining the results of Step 3.2.1 (in \eqref{eq-qhm}) and Step 3.2.2 (in \eqref{eq-qji}), we obtain an important observation that
\begin{equation}
    \label{eq-step323}
    \bar q_{j_i,h} = 0~\forall~h\in\mt H_k,~~i=1,2.
\end{equation}
This is true because Step 3.2.1 reveals $\bar q_{h,\ell}$ can  potentially equal one only for those $\ell$ that is the prerequisite of attribute $h$ (i.e., $\bar q_{h,\ell}=1$ only if $\ell\to h$); and further, Step 3.2.2 establishes that taking the element-wise maximum of the vector $\max_{l\to h}\bar \bq_{\ell}$ and the vector $\bar\bq_{j_i}$ still does not give a vector that requires all the attributes covered by $\bar \bq_h$. Therefore $\bar q_{j_i,h}$ must equal zero. 
Precisely, \eqref{eq-qhm} in Step 3.2.1 implies $\bar\bq_h -\max_{\ell\leq K:\,\ell\to h}\bar \bq_{\ell}  = \ee_h$. 
%(0,\ldots,0,\underbrace{1}_{\text{column }h},0,\ldots,0)$. 
And Step 3.2.2 further implies $\bar q_{j_i,h}=0$, since otherwise $\max\big(\max_{\ell\leq K:\,\ell\to h}\bar \bq_\ell,~\bar\bq_{j_i} \big)\succeq \bar \bq_h$ would happen, contradicting~\eqref{eq-qji}.

\medskip
\noindent
\textbf{Step 3.2 Part III.}
In this step we prove $\bar \theta^-_k=\theta^-_k$ based on \eqref{eq-step323}. Define
\begin{equation}
    \label{eq-gk}
    \ttt^\star  = \bar \theta^-_k
    + \sum_{m\leq K:\, m\neq k,\, m\not\in\mt H_k}\bar \theta^-_m\ee_m
    + \sum_{j>K:\, q_{j,k}=1}\theta^+_j\ee_j
    + \sum_{j>K:\, q_{j,k}=0}\theta^-_j\ee_j,
\end{equation}
and we claim that $ \bar  T_{\rr^\star,\bcolon}\bar\pp=0$ with this $\ttt^\star$ defined above, because of the following reasoning. First, due to the first two terms in \eqref{eq-gk}, $ \bar  T_{\rr^\star,\aaa}\neq 0$ only if $\aaa$ satisfies $\alpha_k=1$ and $\alpha_m=1$ for any attribute $m\not\in \{k\}\cup\mt H_k$. 
% For such $\aaa$, just consider the item $j_1$ with the property $q_{j_1,k}=1$ and $q_{j_1,h}=0$ for all $h\in\mt H_k$, then such item $j_1$ must be included in the third term in \eqref{eq-gk} (i.e., $\sum_{j>K:\, q_{j,k}=1}\theta^+_j\ee_j$). 
Note that in Step 2 we obtained $\bar \theta^+_j=\theta^+_j$ for all $j>K$, then $ \bar  T_{\rr^\star,\aaa}\neq 0$ only if $\aaa\in\{\aaa:\,\aaa\nsucceq \bar\bq_j~\forall~j>K\text{ s.t. }q_{j,k}=1\}=:\mca_3$. 
However considering the item $j_1$ with the property $q_{j_1,k}=1$ and $q_{j_1,h}=0$ for all $h\in\mt H_k$, then such item $j_1$ must be included in the third term in \eqref{eq-gk} (i.e., $\sum_{j>K:\, q_{j,k}=1}\theta^+_j\ee_j$), and we have shown \eqref{eq-step323} in Step 3.2.1 and 3.2.2 that $\bar q_{j_i,h}=1$ only if $h\not\in\mt H_k$. This implies that for all $\aaa\in\mca_3$, there must be $\aaa\succeq\bar \bq_{j_i}$ and $\bar\theta_{j_i,\aaa}=\bar \theta^+_{j_i}$. So we have shown that for any $\aaa\in\{0,1\}^K$, there must be $ \bar  T_{\rr^\star,\aaa}=0$, and the claim that $ \bar  T_{\rr^\star,\bcolon}\bar\pp=0$ is proved. And we have $T_{\rr^\star,\bcolon}\pp=0$.

Next, we consider $T_{\rr^\star,\aaa}$. Due to the last two terms in \eqref{eq-gk}, $T_{\rr^\star,\aaa}\neq 0$ only if $\aaa\in\mca_4$ with $\mca_4$ defined as
\begin{align*}
\mca_4=
\{\aaa:\, \aaa\succeq\bq_j~\forall~j>K~\text{s.t.}~q_{j,k}=0;\quad \aaa\nsucceq\bq_j~\forall~j>K~\text{s.t.}~q_{j,k}=1
\}.
\end{align*}
We claim that for any $\aaa\in\mca_4$, there is $\aaa\succeq\bq_m$ for all $m\not\in\mt H_k$. This claim is true because $\aaa\in\mca_4$ implies $\alpha_m=1$ for all attribute $m$ such that $\QQ^{\star,C}_{\bcolon,m}\nprec \QQ^{\star,C}_{\bcolon,k}$. Recall our inductive hypothesis made in Step 3.1 that $\bar \theta^-_m = \theta^-_m$ for all attribute $m$ that satisfies $\QQ^{\star,C}_{\bcolon,m}\prec \QQ^{\star,C}_{\bcolon,k}$, then we have $T_{\rr^\star,\aaa}\neq 0$ only if $\aaa$ further belongs to the following set $\mca_5$,
\begin{align*}
   \mca_5=\{\aaa:\, \alpha_m=1
    &~\forall m\in[K]~\text{s.t. } \QQ^{\star,C}_{\bcolon,m}\nprec \QQ^{\star,C}_{\bcolon,k} \\
    &\text{ (due to the last two terms in \eqref{eq-gk})};\\
    \alpha_m=1 
    &~\forall m\in[K]~\text{s.t. }
    \QQ^{\star,C}_{\bcolon,m}\prec \QQ^{\star,C}_{\bcolon,k}~\text{and}~m\not\in\mt H_k \\
    &\text{ (due to the 2nd term in \eqref{eq-gk})}\}\\
    =\{\aaa:\, \alpha_m=1 &~\forall m\in[K]~\text{s.t. }
    m\not\in\mt H_k
    \},
\end{align*}
where the last equality uses Lemma \ref{lem-basic} that $\QQ^{\star,C}_{\bcolon,m}\nprec \QQ^{\star,C}_{\bcolon,k}$ implies $k\not\to m$.
From $T_{\rr^\star,\aaa}\neq 0$ only if $\aaa\in\mca_5$,  we have that for all $\aaa\in\mca_5$, there is $\aaa\succeq\bq_m$ for any attribute $m\not\in\mt H_k$, and hence $\theta_{m,\aaa}=\theta^+_m$. 

Furthermore, we claim that if $T_{\rr^\star,\aaa}\neq 0$ (which implies $\aaa\in\mca_5$), we have $\aaa\nsucceq\bq_k$ for the following reason. For $\aaa\in\mca_5$, there is $\alpha_m=1$ for all $m\not\in\mt H_k$.  Consider the item $j_1$ with $q_{j_1,k}=1$ and $q_{j_1,h}=0$ for all $h\in\mt H_k$, and for this $j_1$, there is 
%\begin{equation}
   % \label{eq-aq}
    $\aaa\succeq\bq_{j_1} - \ee_k$.
%\end{equation}
Then since $\theta^+_{j_1}\ee_{j_1}$ is included in \eqref{eq-gk}, in order to have $T_{\rr^\star,\aaa}\neq 0$ we must have $\aaa\nsucceq\bq_{j_1}$. Combined with $\aaa\succeq\bq_{j_1} - \ee_k$, we obtain $\alpha_k=0$ and $\theta_{k,\aaa}=\theta^-_k$.
Denote $\mca_6=\mca_5\cap\{\aaa:\,\alpha_k=0\}$, and we have $T_{\rr^\star,\aaa}\neq 0$ only if $\aaa\in\mca_6.$
Importantly, any $\aaa$ in $\mca_6$ does not violate the attribute hierarchy since $\alpha_k=0$ does not contradict $\alpha_m=1$ for $m\not\in\mt H_k$ as specified in $\mca_5$. Therefore $p_{\aaa}>0$  for all $\aaa\in\mca_6$ under the attribute hierarchy.

Finally, with \eqref{eq-gk}, we conclude that
%based on \eqref{eq-tra} we have
\begin{align*}
     T_{\rr^\star,\aaa} 
    =&~
    \begin{cases}
     (\theta^-_k-\bar \theta^-_k)
  \underset{m\leq K:\, m\neq k,\, m\not\in\mt H_k} 
  \prod(\theta^+_m-\bar \theta^-_m)
  \\
     \qquad\qquad
\times 
 \underset{j>K:\, q_{j,k}=1} \prod(\theta^-_j-\theta^+_j)
     \underset{j>K:\, q_{j,k}=0} \prod(\theta^+_j-\theta^-_j), 
  & \aaa\in\mca_6;
     \\
     0,
    & \text{otherwise}.
    \end{cases}
\end{align*}
and further 
\begin{align*}
    0= T_{\rr^\star,\bcolon}\pp = & (\theta^-_k-\bar \theta^-_k)
    \prod_{m\leq K:\, m\neq k,\, m\not\in\mt H_k}(\theta^+_m-\bar \theta^-_m)\\
&\times    \prod_{j>K:\, q_{j,k}=1}(\theta^-_j-\theta^+_j) 
    \prod_{j>K:\, q_{j,k}=0}(\theta^+_j-\theta^-_j)
    \Big(\sum_{\aaa\in\mca_6} p_{\aaa}\Big).
\end{align*}
Then since in the last paragraph we have shown $\sum_{\aaa\in\mca_6} p_{\aaa}>0$, the only potentially zero factor in the above display could only be $(\theta^-_k-\bar \theta^-_k)$. Now we have obtained $\bar \theta^-_k=\theta^-_k$, and the proof of Step 3.2.3 is complete.

\medskip\noindent\textbf{Step 3.3.}
Now we complete the inductive argument in the current Step 3 and conclude $\bar \theta^-_k=\theta^-_k$ for all attribute $k\in[K]$.
By completing the induction, we have obtained one more useful byproduct in the proof of Step 3, which is \eqref{eq-idq} that $\bar q_{h,m}=0$ for any attribute $m\not\to h$. This exactly means under the true attribute hierarchy and the induced attribute pattern set $\mca$, the first $K$ items of $\bar Q$ is equivalent to the identity matrix $I_K$. Namely, we obtain
%\begin{equation}
%    \label{eq-idid}
    $\bar \QQ_{1:K,\bcolon} \stackrel{\mce}{\sim} I_K$.
%\end{equation}

\medskip
\noindent
\textbf{Step 4.}
In this step we prove $\bar \QQ\stackrel{\mce}{\sim} \QQ$.
Without loss of generality, we assume the columns of $\bar \QQ$ is arranged in the order $(b_1,b_2,\ldots,b_K)$.
Recall that $\mca\subseteq\{0,1\}^K$ denotes the set of attribute patterns that respect the specified attribute hierarchy.
For each $j\in\{K+1,\ldots,J\}$, in the following two parts (i) and (ii), we first prove 
$\mca_{\star} := \{\aaa\in\mca:\, \bar\Gamma_{j,\aaa}=1,\, \Gamma_{j,\aaa}=0\}=\varnothing$ in (i);
and then prove 
$\mca_{\star\star} := \{\aaa\in\mca:\, \bar\Gamma_{j,\aaa}=0,\, \Gamma_{j,\aaa}=1\}=\varnothing$
in (ii). Together, these two conclusions would imply $\bar\bq_j \stackrel{\mce}{\sim} \bq_j$.

\begin{enumerate}
    \item[(i)] We use proof by contradiction and suppose $\mca_{\star} = \{\aaa\in\mca:\, \aaa\succeq\bar\bq_j,\, \aaa\nsucceq\bq^{\sparse}\}\neq\varnothing$ for some $j\in\{K+1,\ldots,J\}$. 
    Then $\sum_{\aaa\in\mca_{\star}}p_{\aaa}>0$.
    Define
    \begin{equation}\label{eq-step4}
        \ttt^\star = \sum_{k\leq K:\, \bar q_{j,k}=1}\theta^-_k\ee_k + \theta^+_j\ee_j,
    \end{equation}
    then $ \bar  T_{\rr^\star,\aaa}=0$ for all $\aaa\in\{0,1\}^K$ and hence $ \bar  T_{\rr^\star,\bcolon}\bar\pp = 0$. Based on Step 2 and 3, we have $\bar \theta^+_j = \theta^+_j$ and $\bar \theta^-_k = \theta^-_k$ for the $j$ and any $k$ with $\bar q_{j,k}=1$ used in \eqref{eq-step4}. Therefore, due to the first summation term in \eqref{eq-step4}, $T_{\rr^\star,\aaa}\neq 0$ only if $\aaa$ satisfies $\alpha_k=1$ for all $k$ s.t. $\bar q_{j,k}=1$ (i.e., $\aaa\succeq\bar \bq_j$ and $\bar\Gamma_{j,\aaa}=1$); and due to the second term $\theta^+_j\ee_j$ in \eqref{eq-step4}, $T_{\rr^\star,\aaa}\neq 0$ only if $\theta_{j,\aaa}=\theta^-_j$ (i.e., $\Gamma_{j,\aaa}=0)$. In summary, $T_{\rr^\star,\aaa}\neq 0$ only if $\aaa\in\mca_{\star}$, so
    \begin{align*}
        T_{\rr^\star,\bcolon}\pp = 
        \prod_{k\leq K:\, q_{j,k}=1}(\theta^+_k-\theta^-_k)(\theta^-_j-\theta^+_j)\Big( \sum_{\aaa\in\mca_*}p_{\aaa} \Big)\neq 0, 
    \end{align*}
    which contradicts $T_{\rr^\star,\bcolon}\bar \pp=0$.
    This implies $\mca_{\star}=\varnothing$ must hold.
    
    \item[(ii)] We also use proof by contradiction and suppose $\mca_{\star\star} = \{\aaa\in\mca:\, \bar\Gamma_{j,\aaa}=0,\, \Gamma_{j,\aaa}=1\}\neq\varnothing$ for some $j\in\{K+1,\ldots,J\}$. Then there exists $\aaa\in\mca$ with $\Gamma_{j,\aaa}=1$ but $\bar\Gamma_{j,\aaa}=0$, which implies there exists some attribute $k\in[K]$ s.t. $\bar q_{j,k}=1$ and $q_{j,k}=0$.
    Based on the above relation, we apply Lemma \ref{lem-ch} with $\mathcal K=\{j\}$ and $m=k$ to obtain $\bar \theta^+_k = \theta^+_k$.
    Define
    \begin{equation}
        \label{eq-step4t}
        \ttt^\star = \bar \theta^-_j\ee_j
        %+ \sum_{\ell\leq K:\atop \ell\to k} \bar \theta^-_\ell\ee_\ell 
        +\bar \theta^+_k\ee_k 
        + \sum_{m\leq K:\, k\not\to m}\theta^-_m\ee_m,
    \end{equation}
    then based on the first two terms in \eqref{eq-step4t}, we have $ \bar  T_{\rr^\star,\aaa}=0$ for all $\aaa\in\{0,1\}^K$. So $ \bar  T_{\rr^\star,\bcolon}\bar\pp=0$ and further $T_{\rr^\star,\bcolon}\pp=0$. Now consider $T_{\rr^\star,\aaa}$, then $T_{\rr^\star,\aaa}\neq 0$ only if $\aaa$ belongs to the set $\mca_7$ defined as
    \begin{align}
        \label{eq-mca7}
        \mca_7 = \{\aaa\in\mca:\,
        \alpha_k=0;~\alpha_m=1~\forall~k\not\to m\},
        %~\aaa\nsucceq\bq_j\},
    \end{align}
    then this $\mca_7\neq\varnothing$ because the $\aaa^*=(\alpha^*_1,\ldots,\alpha^*_K)$ defined as follows belongs to $\mca_7$. The $\aaa^*$ takes the form $\alpha^*_k=0$, $\alpha^*_{\ell}=0$ for all $k\to \ell$, and $\alpha^*_m=1$ for all $k\not\to m$. The $\aaa^*$  also satisfies $\aaa^*\succeq\bq_j$ for the following reason.
    Since $q_{j,k}=0$, then under the attribute hierarchy this $\bq_j$ is equivalent to a $\tilde\bq_j$ with $\tilde q_{j,k}=0$ and $\tilde q_{j,\ell}=0$ for all $\ell$ s.t. $k\to \ell$.
    Therefore for the defined $\aaa^*\in\mca$ that respects the attribute hierarchy, there must be $\aaa^*\succeq\tilde\bq_j$, so $\Gamma_{j,\aaa}=1$.
    %Since Lemma \ref{lem-basic} establishes that we can consider without loss of generality the case where each row vector of $\QQ$ respects the attribute hierarchy, we have the conclusion that equivalently, $\aaa^*\succeq\bq_j$. 
     So there is $\sum_{\aaa\in\mca_7}p_{\aaa}\geq p_{\aaa^*}>0$.
    Now we have
    \begin{equation*}
        0=T_{\rr^\star,\bcolon}\pp
         =(\theta^+_j-\bar \theta^-_j)(\theta^-_k-\bar \theta^+_k)\prod_{m\leq K:\, k\not\to m}(\theta^+_m-\theta^-_m)
         \Big(\sum_{\aaa\in\mca_7}p_{\aaa} \Big),
    \end{equation*}
    which leads to a contradiction since each factor in the above term is nonzero. So we have proved the $\mca_{\star\star}$ defined earlier must also be an empty set.
\end{enumerate}

As stated before, based on the (i) and (ii) shown above, we obtain $\bar\bq_j \stackrel{\mce}{\sim} \bq_j$ for every item $j\in\{K+1,\ldots,J\}$.
In summary, by far we have obtained $\bar \theta^-_k = \theta^-_k$ for all $k\in[K]$, $\bar \theta^+_j = \theta^+_j$ for all $j\in\{K+1,\ldots,J\}$, and $\bar\QQ\stackrel{\mce}{\sim} Q$.

\medskip
\noindent
\textbf{Step 5.}
We next show $\bar \theta^+_k = \theta^+_k$ for all $k\in[K]$ and $\bar \theta^-_j = \theta^-_j$ for all $j\in\{K+1,\ldots,J\}$, and $\bar\pp = \pp$.

\smallskip\noindent
\textbf{Step 5.1.}
In this step, we show $\bar \theta^+_k = \theta^+_k$ for all $k\in[K]$. By Condition $B$, there exists some item $j>K$ s.t. $q^{\sparse}_{j,k}=1$, and we denote this item by $j_k$. Define
\begin{equation}
    \label{eq-step51}
    \ttt^\star = \sum_{h\leq K:\, h\neq k}\bar \theta^-_h\ee_h
    + \bar \theta^-_{j_k}\ee_{j_k}
    + \sum_{j>K:\, j\neq j_k}\theta^-_j\ee_j,
\end{equation}
then $T_{\rr^\star,\aaa}\neq 0$ if and only if $\aaa=\one_K$. This is because considering the the last term of summation in \eqref{eq-step51}, we have $T_{\rr^\star,\aaa}\neq 0$ only if $\aaa\succeq\bq_{\mathcal J}$ where $\mathcal J:=\{K+1,\ldots,J\}\setminus\{j_k\}$; and  by Condition $B$ there is $\bq^{\sparse}_{\mathcal J}=\one_K$ so $\bq_{\mathcal J}=\one_K$ as well. Specifically, 
$$
T_{\rr^\star,\one_K} = \prod_{h\leq K:\, h\neq k}(\theta^+_h-\bar \theta^-_h)
(\theta^+_{j_k} - \bar\theta^-_{j_k})
\prod_{j>K:\, j\neq j_k}(\theta^+_j-\theta^-_j),
$$
and there is $T_{\rr^\star,\bcolon}\pp = T_{\rr^\star,\one_K}p_{\one_K}\neq 0$. So by \eqref{eq-tra} we have $ \bar  T_{\rr^\star,\bcolon}\bar\pp \neq 0$.
Further,  the element $\bar T_{\rr^\star,\aaa}$ could be potentially nonzero only if $\aaa=\one_K$. This is because considering the first two terms ($\sum_{h\leq K:\, h\neq k}\bar \theta^-_h\ee_h$ and $\bar \theta^-_{j_k}\ee_{j_k}$) in $\ttt^\star$ defined in \eqref{eq-step51}, there is $ \bar  T_{\rr^\star,\aaa}\neq 0$ only if $\aaa\succeq\max(\max_{h\leq K:\, h\neq k}\bar\bq_h,\bar\bq_{j_k})$; and since $\bar q_{j_k,k}=1$ there must be  $\max(\max_{h\leq K:\atop h\neq k}\bar\bq_h,\bar\bq_{j_k})=\one_K$.
Therefore,
\begin{equation*}
    \bar \theta^+_k = 
    \frac{ \bar  T_{\rr^\star+\ee_k,\bcolon}\bar\pp}{ \bar  T_{\rr^\star,\bcolon}\bar\pp}=
    \frac{T_{\rr^\star+\ee_k,\bcolon}\pp}{T_{\rr^\star,\bcolon}\pp}
    =\theta^+_k.
\end{equation*}

\smallskip\noindent
\textbf{Step 5.2.}
In this step we show $\bar \theta^-_j = \theta^-_j$ for all $j\in\{K+1,\ldots,J\}$.
Consider an arbitrary $j>K$, then there exists an attribute $k$ such that $q_{j,k}=1$. Define $\ttt^\star=\theta^+_k\ee_k$, and note that in Step 5.1 we  obtained $\bar \theta^+_k = \theta^+_k$ and in Step 3 we obtained $\bar \theta^-_k =\theta^-_k$. Then with this $\ttt^\star$, there is 
$$
 0\neq(\theta^-_k-\theta^+_k)\Big(\sum_{\aaa\in\mca:\, \aaa\nsucceq\bq_k} p_{\aaa}\Big)
 =T_{\rr^\star,\bcolon}\pp = \bar  T_{\rr^\star,\bcolon}\bar\pp
 =(\theta^-_k-\theta^+_k)\Big(\sum_{\aaa\in\mca:\, \aaa\nsucceq\bq_k}\bar p_{\aaa}\Big),
$$
and note that for any $\aaa\neq\bq_k$, there must be $\aaa\neq\bq_j$ since $q_{j,k}=1$.
Now consider the item $j$, we have
$$
\bar \theta^-_j = \frac{ \bar  T_{\rr^\star+\ee_j,\bcolon}\bar\pp}{ \bar  T_{\rr^\star,\bcolon}\bar\pp}
=\frac{ \bar  T_{\rr^\star+\ee_j,\bcolon}\pp}{ \bar  T_{\rr^\star,\bcolon}\pp}
=\theta^-_j.
$$
Since $j$ is arbitrary from $\{K+1,\ldots,J\}$, we have obtained $\bar \theta^-_j = \theta^-_j$ for all $j\in\{K+1,\ldots,J\}$.

\medskip
\noindent
\textbf{Step 6.}
In this step we show that for $\Gamma(\QQ,\mce)$ and the alternative $\Gamma$-matrix $\bar \Gamma$ (also denoted by $ \Gamma(\bar\QQ,\bar\mce)$ where $\mca(\bar\mce)$ is the set corresponding to those columns in $\bar \Gamma$ with nonzero proportion parameters in $\bar\pp$), the column vectors in $\Gamma(\bar\QQ,\bar\mce)$ that correspond to $\bar p_{\aaa}>0$ are identical to $\Gamma(\QQ,\mce)$; furthermore, $\bar p_{\pi(\aaa)} = p_{\aaa}$ for $\aaa\in\mca(\mce)$, where $\pi:\mca(\mce) \to \mca(\bar\mce)$ is a one-to-one map.
For an arbitrary $\aaa\in\mca(\mce)$, define
\begin{equation}
    \label{eq-final}
    \ttt^\star = \sum_{k\leq K:\, \aaa\succeq\bq_k}\theta^-_k\ee_k
    + \sum_{m\leq K:\, \aaa\nsucceq\bq_m}\theta^+_m\ee_m.
\end{equation}
Then for any $\aaa^*\in\mca$, the $T_{\rr^\star,\aaa^*}\neq 0$ (equivalently, $ \bar  T_{\rr^\star,\aaa^*}\neq 0$) if and only if $\aaa^*=\aaa$, because $\QQ^0 = Q_{1:K,\bcolon}\stackrel{\mce}{\sim} I_K$.
Then $T_{\rr^\star,\bcolon}\pp = \bar  T_{\rr^\star,\bcolon}\bar\pp$ gives
\begin{eqnarray*}
&& \prod_{k\leq K:\, \alpha_k=1}(\theta^+_k-\theta^-_k)
\prod_{m\leq K:\, \alpha_m=0}(\theta^-_m-\theta^+_m)\, p_{\aaa}\\
&=& \prod_{k\leq K:\, \alpha_k=1}(\theta^+_k-\theta^-_k)
\prod_{m\leq K:\, \alpha_m=0}(\theta^-_m-\theta^+_m)\, \bar p_{\pi(\aaa)},		
\end{eqnarray*}
and we obtain $\bar p_{\pi(\aaa)} = p_{\aaa}$.
Since $\sum_{\aaa\in\{0,1\}^K}\bar p_{\pi(\aaa)}=\sum_{\aaa\in\mca}p_{\aaa}=1$, the equality $\bar p_{\pi(\aaa)} = p_{\aaa}$ for any $\aaa\in\mca$ also implies $\bar p_{\pi(\aaa)}=0$ for all $\aaa\in\{0,1\}^K\setminus\mca$. 
So $\Gamma(\bar\QQ,\bar\mce)=\Gamma(\QQ,\mce)$ also holds.
This completes the proof of Step 6.
Now we have shown $\Gamma(\QQ,\mce)=\Gamma(\bar\QQ,\bar\mce)$, $\bar \ttt^+=\ttt^+$, $\bar \ttt^-=\ttt^-$,  $\bar\pp=\pp$. %and $\bar\mca =\mca$.
This completes the proof of the sufficiency of Conditions $A$, $B$ and $C$.

\textit{As for the last claim in the theorem} that Conditions $A$, $B$ and $C$ are necessary and sufficient for identifiability of $(\QQ,\,\pp,\,\ttt^+,\,\ttt^-)$ where there is no hierarchy, it directly follows from the result in Theorem 1 in \cite{id-Q}. 
%exactly established this argument. 
\end{proof}

\section{Proofs of Propositions 1--4}\label{pf-prop14}

%\paragraph{Proof of part (ii).}
\begin{proof}[Proof of Proposition \ref{prop}]
\textit{We first show that if $\QQ$ contains a submatrix $I_K$ in addition to satisfying $A$, $B$ and $C$, then $(\mce,\ttt^+,\ttt^-,\pp)$ are jointly identifiable.} Based on the conclusion of part (i), it suffices to show that if $\QQ$ contains an $I_K$, then $\mce$ is identifiable from $\Gamma(\QQ,\mce)$. That is, we will show that if $\Gamma(\QQ,\mce) = \Gamma(\bar \QQ,\bar\mce)$ with both $\QQ$ and $\bar\QQ$ containing a submatrix $I_K$, then $\mce =\bar\mce$. 
Given a $\QQ$-matrix, denote the ideal response matrix corresponding to an empty attribute hierarchy $\mce=\varnothing$ (which leads to a saturated latent pattern space $\mca(\varnothing)=\{0,1\}^K$) by $\Gamma(\QQ,\varnothing)$.
Note that when $\QQ$ contains an $I_K$, the $J\times 2^K$ matrix $\Gamma(\QQ,\varnothing)$ has $2^K$ distinct column vectors \cite{partial}. Without loss of generality, suppose the first $K$ rows of $\QQ$ and $\bar\QQ$ are both $I_K$. Then due to this distinctiveness of the $2^K$ ideal response vectors of the $2^K$ latent patterns under an identity matrix, $\Gamma_{1:K, \bcolon}(\QQ,\mce) = \Gamma(I_K,\mce) = \Gamma(I_K,\bar\mce)= \Gamma_{1:K,\bcolon}(\bar \QQ,\bar\mce)$ exactly implies $\mca(\mce) = \mca(\bar\mce)$, which further gives $\bar\mce=\mce$.

\textit{We next show that in order to identify an arbitrary $\mce$, it is necessary for $\QQ$ to contain an $I_K$.} Suppose $\QQ$ does not contain an $I_K$, then based on the concept of $\pp$-partial identifiability in \cite{partial}, certain patterns would become equivalent in that they lead to the same column vectors in $\Gamma(\QQ,\varnothing)$, hence there must exist some $\mca$ that is not identifiable.
This completes the proof of Proposition \ref{prop}.
\end{proof}

\begin{proof}[Proof of Proposition \ref{prop-Q}]
Theorem \ref{thm-main} already shows under Conditions A, B, and C, there is $\Gamma(\QQ,\mce) = \Gamma(\bar \QQ,\bar\mce)$.
Proposition \ref{prop} shows that under the additional condition that $\QQ$ (and $\bar\QQ$) contains a submatrix $I_K$, there is $\mce = \bar\mce$.
Therefore, we obtain $\Gamma(\QQ, \mce) = \Gamma(\bar \QQ, \mce)$, which exactly means that $\QQ$ is identifiable up to the equivalence class defined by the attribute hierarchy graph $\mce$. Therefore $\mc D^{\mce}(\QQ)$ and $\mc S^{\mce}(\QQ)$ are identifiable. 
By the definition of the $\mce$-equivalence class of $\QQ$, the specific $\QQ$ is not identifiable for any nonempty attribute hierarchy $\mce \neq \varnothing$.
This  proves the conclusion of Proposition \ref{prop-Q}.
\end{proof}

\begin{proof}[Proof of Proposition \ref{prop-A}]
\textit{We next show that Condition $A$ is necessary for identifying} $(\Gamma(\QQ,\mce),\,\allowbreak\ttt^+,\,\ttt^-,\,\pp)$. We use proof by contradiction and assume that Condition $A$ does not hold. Recall that the type of modification of $\QQ$ described in Condition $A$ is the sparsifying operation, which sets every $q_{j,k}$ to zero if $q_{j,h}=1$ and $k\to h$. The resulting matrix is denoted by $\mathcal S^{\mce}(\QQ)$. If Condition $A$ fails to hold, then $\mathcal S^{\mce}(\QQ)$ lacks an identity submatrix $I_K$. Without loss of generality, suppose $\mathcal S^{\mce}(\QQ)$ does not contain any row vector in the form $\ee_h$ for some $h\in[K]$. Combined with the definition of the sparsifying operation, this means for any $\bq$-vector with $q_{j,h}=1$, in the original $\QQ$ there must be $q_{j,\ell}=1$ for some $\ell\not\to h$.
Then the following two attribute patterns in $\mca$ will lead to the same column vectors in $\Gamma(\QQ,\mce)$: $\aaa_1:=\zero_K$ and $\aaa_2:=(\alpha_{2,1},\ldots,\alpha_{2,K})$ where $\alpha_{2,h}=1$, $\alpha_{2,k}=1$ for all $k\to h$, and $\alpha_{2,\ell}=0$ for all $\ell\not\to h$. Then $\aaa_1,\aaa_2\in\mca(\mce)$. Under the current assumption on the structure of $\QQ$, there is $\Gamma_{:,\aaa_1}(\QQ,\mce)=\Gamma_{:,\aaa_2}(\QQ,\mce)=\zero$, which directly results in that $p_{\aaa_1}$ and $p_{\aaa_2}$ can be at best identified up to their sum, even if all the item parameters $\ttt^+$ and $\ttt^-$ are identified and known. In other words, the two separate proportions $p_{\aaa_1}$ and $p_{\aaa_2}$ are not identifiable. This proves the necessity of Condition $A$.
\end{proof}

We first introduce a new definition and a new proposition, the proof of which paves the way for later proofs.
\begin{definition}[Common-Ancestor Hierarchy]
An attribute hierarchy $\mce$ is said to be a {common-ancestor} hierarchy if there exists some latent attribute $k$ that serves as a direct or indirect prerequisite for all the other attributes.
\end{definition}
\noindent This family of the {common-ancestor} attribute hierarchy is quite general and includes many specific attribute structures. 
Indeed, the \textit{linear} hierarchy, \textit{convergent} hierarchy, \textit{divergent} hierarchy and the so-called \textit{unstructured} hierarchy presented in \cite{templin2014hierarchical} (shown in our Fig \ref{fig-leighton-hier} in Example \ref{exp-leighton}) all belong to the common-ancestor hierarchy family.

\begin{proposition}[Identifiability for Common-Ancestor Hierarchy]\label{prop-two-e}
%Consider those $\QQ$-matrices whose row vectors respect the hierarchy $\mce$. 
Consider a DINA-based HLAM with a fixed $\QQ$-matrix.
Under a common-ancestor hierarchy, if $\QQ$ contains two copies of the reachability matrix $\EE$ as a submatrix, then $(\ttt^+, \ttt^-, [\mathcal E],  [\pp])$ are identifiable.
\end{proposition}

\begin{proof}[Proof of Proposition \ref{prop-two-e}]
	Without loss of generality, assume attribute $k=1$ serves as a direct or indirect prerequisite for all the other attributes $h=2,\ldots, K$ and that $\QQ$ takes the following form,
	\begin{align}\label{eq-twoe-pf}
		\QQ = \begin{pmatrix}
			\EE \\
			\EE \\
			\QQ^\star
		\end{pmatrix}
		=
		\begin{pmatrix}
			1 & 0 & \cdots & 0\\
			1 & 1 & \cdots & 0\\
			\vdots & \vdots & \ddots & \vdots \\
			1 & * & \cdots & 1\\
			\hline
			1 & 0 & \cdots & 0\\
			1 & 1 & \cdots & 0\\
			\vdots & \vdots & \ddots & \vdots \\
			1 & * & \cdots & 1\\
			\hline
			  &  &\QQ^\star~~ & 
		\end{pmatrix}.
	\end{align}
	First, let 
	$$
	\ttt^\star = \bar \theta^-_1 \ee_1 + \bar \theta^+_{K+1} \ee_{K+1} + \sum_{k=2}^K \theta^-_k\ee_k,
	$$
	then $\bar T_{\rr^\star,\bcolon}=\zero$ and hence $\bar T_{\rr^\star,\bcolon} \bar \pp = 0 = T_{\rr^\star,\bcolon} \pp$. Thanks to the third group of terms in the defined $\ttt^\star$, using the fact that under the true model parameters we have $1\to h$ for all attribute $h=2,\ldots,K$, we obtain that $T_{\rr^\star,\aaa}$ is potentially nonzero only for the all-one attribute pattern $\aaa=\one_K$, therefore
	$$
	\bar T_{\rr^\star,\bcolon} \bar \pp = 0 = T_{\rr^\star,\bcolon} \pp = p_{\one_K}(\theta^+_k - \theta^-_k)\prod_{k=1}^K(\theta^+_k - \theta^-_K) (\theta^+_{K+1} - \bar \theta^+_{K+1}),
	$$
	which implies $\theta^+_{K+1}=\bar \theta^+_{K+1}$. By symmetry we also obtain $\theta^+_1 = \bar \theta^+_1$.  
	%%%
	Second, given $\theta^+_{K+1} = \bar \theta^+_{K+1}$, let
	$$
	\ttt^\star =  \bar \theta^-_1 \ee_1 + \bar \theta^+_{K+1} \ee_{K+1},
	$$
	then  $\bar T_{\rr^\star,\bcolon}=\zero$ and hence
	$$
	\bar T_{\rr^\star,\bcolon} \bar \pp = 0 = T_{\rr^\star,\bcolon} \pp = p_{\zero_K}(\theta^-_1 - \bar \theta^-_1)(\theta^-_{K+1} - \bar \theta^+_{K+1}),
	$$
	which implies $\theta^-_1 = \bar \theta^-_1$. By symmetry we also obtain $\theta^-_{K+1} = \bar \theta^-_{K+1}$.  
	%%%
	Third, for some $k\in\{2,\ldots, K\}$ we define
    $$
	\ttt^\star =  \bar \theta^+_1 \ee_1 + \bar \theta^-_k\ee_k,
	$$
	then $\bar T_{\rr^\star,\aaa}=0$ for any $\aaa\in\{0,1\}^K$. This is because $\bar T_{\rr^\star,\aaa}\neq 0$ only potentially for those $\aaa$ such that $\aaa\nsucceq\bq_1$ and $\aaa\succeq\bq_k$; but $\bq_j\succeq\bq_1$ so such $\aaa$ does not exist. Based on $\bar T_{\rr^\star,\bcolon}=\zero$, we have
	$$
	\bar T_{\rr^\star,\bcolon} \bar \pp = 0 = T_{\rr^\star,\bcolon} \pp = p_{\zero_K}(\theta^-_1 - \bar \theta^+_1)(\theta^-_k - \bar \theta^-_k),
	$$
	which gives $\theta^-_k = \bar \theta^-_k$ for $k\in\{2,\ldots, K\}$. By symmetry we can obtain $\theta^-_k = \bar \theta^-_k$ for $k\in\{K+2,\ldots, 2K\}$.
	%%%
	Last, given that $\theta^-_1 = \bar \theta^-_1$, for some $k\in\{2,\ldots, K\}$ we define 
	$$
	\ttt^\star = \bar \theta^+_k 
	     + \sum_{h\in[K],\, h\neq k} \theta^-_h\ee_h + \theta^-_{K+k}\ee_{K+k},
	$$
	then $\bar T_{\rr^\star,\aaa}=0$ for any $\aaa\in\{0,1\}^K$. 
	Therefore,
	$$
	\bar T_{\rr^\star,\bcolon} \bar \pp = 0 = T_{\rr^\star,\bcolon} \pp = p_{\one_K} (\theta^+_k - \bar \theta^+_k) \prod_{h\in[K],\, h\neq k} (\theta^+_h-\theta^-_h) (\theta^+_{K+k}-\theta^-_{K+k}),
	$$
	therefore we obtain $\theta^+_k=\bar \theta^+_k$ for $k\in\{2,\ldots, K\}$. By symmetry we also have $\theta^+_k=\bar \theta^+_k$ for $k\in\{K+2,\ldots, 2K\}$.

	Thus far we have proved $\theta^-_j=\bar \theta^-_j$ and $\theta^+_j=\bar \theta^+_j$ for all $j\in\{1,\ldots,2K\}$. Based on this, we next show  $\theta^-_j=\bar \theta^-_j$ and $\theta^+_j=\bar \theta^+_j$ for any $j\in\{2K+1,\ldots,J\}$. In particular, for $j\in\{2K+1,\ldots,J\}$ and any $\aaa\in\mca$, define 
	\begin{align*}
     \ttt^\star = 
	 &~\sum_{m\leq K:\,\alpha_{m}=1}\theta_m^-\ee_m + 
	 \sum_{k\leq K:\,\alpha_{m}=0}\theta_k^+\ee_k +\\
	 &~\sum_{m\leq K:\,\alpha_{m}=1}\theta_{K+m}^-\ee_{K+m} + 
	 \sum_{k\leq K:\,\alpha_{m}=0}\theta_{K+k}^+\ee_{K+k}.
	\end{align*}
	Then 
	\begin{align*}
	T_{\rr^\star}\pp=
	&~p_{\aaa}\times\prod_{m\leq K:\,\alpha_{m}=1}(\theta_m^+-\theta_j^-)
	\prod_{k\leq K:\,\alpha_{m}=0}(\theta_k^- - \theta_k^+)\\
	&~\quad \prod_{m\leq K:\,\alpha_{m}=1}(\theta_{K+m}^+ - \theta_{K+m}^-)
	\prod_{k\leq K:\,\alpha_{m}=0}(\theta_{K+k}^- - \theta_{K+k}^+)\neq 0,
	\end{align*}
	so $\bar T_{\rr^\star}\bar\pp = T_{\rr^\star}\pp \neq 0$. 
Therefore we have
	\begin{align*}
\theta_{j,\aaa}=\frac{T_{\rr^\star+\ee_j}\pp}{T_{\rr^\star}\pp}
=\frac{\bar T_{\rr^\star+\ee_j}\bar\pp}{\bar T_{\rr^\star}\bar\pp}=\bar\theta_{j,\aaa}.
	\end{align*}	
	Now that $\aaa\in\mca$ is arbitrary, we have obtained $\bar\theta^+_j=\theta^+_j$ and $\bar\theta^-_j=\theta^-_j$ for this $j$. Since $j\in\{2K+1,\ldots,J\}$ is also arbitrary, we have establishes that $\bar\ttt^+=\ttt^+$ and $\bar\ttt^-=\ttt^-$.

	Next we show $\bar p_{\aaa} = p_{\aaa}$ for all $\aaa\in\mathcal A(\mathcal E)$, which will naturally establish the identifiability of the attribute hierarchy $\mathcal E$. 
	First, define $\ttt^\star = \sum_{k=1}^K \theta^+_k\ee_k$, then  $\bar T_{\rr^\star,\bcolon} \bar \pp = T_{\rr^\star,\bcolon} \pp$ gives 
	$$
	\bar p_{\zero_K}\prod_{k=1}^K (\theta^-_k - \theta^+_k) =  p_{\zero_K}\prod_{k=1}^K (\theta^-_k - \theta^+_k),
	 %\quad\Longrightarrow\quad	\bar p_{\zero_K} = p_{\zero_K}.
	$$
	which implies $\bar p_{\zero_K} = p_{\zero_K}$.
Second, for some $\aaa=(1,\aaa')\in\mca(\mce)$ where $\aaa'\in\{0,1\}^K$, we have $p_{\aaa}>0$; define
$$
\ttt^\star =  \theta^-_1\ee_1 + \sum_{2\leq k\leq K:\, \alpha_k=1} \theta^-_k\ee_k
+\sum_{2\leq k\leq K:\, \alpha_k=0} \theta^+_k\ee_k.
$$
Then $\bar T_{\rr^\star,\bcolon} \bar \pp = T_{\rr^\star,\bcolon} \pp$ gives
\begin{align*}
&~\bar p_{\aaa} (\theta^+_1-\theta^-_1)\prod_{k\geq 2,\, \alpha_k=1}(\theta^+_k-\theta^-_k)\prod_{k\geq 2,\, \alpha_k=0}(\theta^-_k-\theta^+_k)\\
=&~
p_{\aaa} (\theta^+_1-\theta^-_1)\prod_{k\geq 2,\, \alpha_k=1}(\theta^+_k-\theta^-_k)\prod_{k\geq 2,\, \alpha_k=0}(\theta^-_k-\theta^+_k),	
\end{align*}
which gives $\bar p_{\aaa} = p_{\aaa}$ for all $\aaa=(1,\aaa')\in\mca(\mce)$. Thus far we established $\bar p_{\aaa} = p_{\aaa}$ for all $\aaa\in\mca(\mce)$. This implies $\sum_{\aaa\in\mca(\mce)} \bar p_{\aaa} = \sum_{\aaa\in\mca(\mce)} p_{\aaa} = 1$. Now from $\sum_{\aaa\in\{0,1\}^K} \bar p_{\aaa} = 1$ and $\bar p_{\aaa} \geq 0$ for any $\aaa\in\{0,1\}^K$, we obtain that $\bar p_{\aaa} =0$ for any $\aaa\in\{0,1\}^K\setminus \mca(\mce)$.
Now we have proved $\bar\pp = \pp$. This establishes the identifiability of  $\mce$ and $(\ttt^+,\ttt^-,\pp)$.
The proof is complete.
\end{proof}

\bigskip

%\section{Proofs of Results in Section \ref{sec-nece}}

\section{Proofs of the Individual Necessary Conditions in Propositions 5--8 in Section 4.1}
\label{pf-prop58}

\begin{proof}[Proof of Proposition \ref{prop-nece-sing}]
	We first prove part (a) using proof by contradiction. Suppose attribute $k$ is not connected to any other attribute in the DAG and $\sum_{k=1}^K q_{j,k} = 2$, and there is $T(\bar\ttt^+,\bar\ttt^-)\bar\pp = T(\ttt^+,\ttt^-)\pp$. We next construct $(\bar\ttt^+,\bar\ttt^-,\bar\pp)\neq (\ttt^+,\ttt^-,\pp)$ to show the contradiction. Without loss of generality, suppose $k=1$ is the attribute that is not connected to any others and $\QQ$ takes the following form,
\begin{align*}
	\QQ = \begin{pmatrix}
		1 & \zero^\top\\
		1 & \vv_2^\top\\
		\hline
		\zero & \QQ^\star
	\end{pmatrix}.
\end{align*}
Let $\bar \theta^+_j=\theta^+_j$ and $\bar \theta^-_j=\theta^-_j$ for all $j=3,\ldots, J$.
Given this, it is not hard to see that to guarantee \eqref{eq-def} holds, we only need to ensure the following set of equations hold for certain $\aaa'\in\{0,1\}^{K-1}$ to be specified later,
\begin{align}\label{eq-set4}
	\begin{cases}
		\bar p_{(0,\aaa')} + \bar p_{(1,\aaa')} = p_{(0,\aaa')} + p_{(1,\aaa')}, \\
		\bar \theta^-_1 \bar p_{(0,\aaa')} + \bar \theta^+_1\bar p_{(1,\aaa')} = \theta^-_1 p_{(0,\aaa')} + \theta^+_1 p_{(1,\aaa')},
		 \\
		\bar \theta^-_2 \bar p_{(0,\aaa')} + \bar \theta^+_2\bar p_{(1,\aaa')} \\
 \qquad \qquad = \theta^-_2 p_{(0,\aaa')} + \theta^+_2 p_{(1,\aaa')},&
  (\forall \aaa'\succeq\vv_2), \\
		\bar \theta^-_1 \bar \theta^-_2\bar p_{(0,\aaa')} + \bar \theta^+_1 \bar \theta^+_2 \bar p_{(1,\aaa')} \\
 \qquad \qquad = \theta^-_1 \theta^-_2 p_{(0,\aaa')} + \theta^+_1 \theta^+_2 p_{(1,\aaa')},& 
  (\forall \aaa'\succeq\vv_2).
	\end{cases}
\end{align}
Since attribute $k$ does not have any prerequisite nor serve as the prerequisite for any other attribute, we claim that for any $\aaa'\in\{0,1\}^{K-1}$, there are only the following two cases: Case (1), $p_{(0,\aaa')} = p_{(1,\aaa')} = 0$; or Case (2), $p_{(0,\aaa')}\neq 0$ and  $p_{(1,\aaa')}\neq 0$. This is because $p_{(0,\aaa')} \neq 0$ and $p_{(1,\aaa')}=0$ would indicate $k\to 1$ for some attribute $k\in\{2,\ldots,K\}$; also $p_{(0,\aaa')} = 0$ and $p_{(1,\aaa')}\neq 0$ would indicate $1\to k$ for some attribute $k\in\{2,\ldots,K\}$. These two cases violate the assumption that $k$ does not have any prerequisite nor serve as the prerequisite for any other attribute. So we have proved the claim that either $p_{(0,\aaa')} = p_{(1,\aaa')} = 0$ or $p_{(0,\aaa')}\neq 0$ and  $p_{(1,\aaa')}\neq 0$. 
With this observation, we take a specific type of true proportion parameters $\pp$ such that $p_{(1,\aaa')} / p_{(0,\aaa')} = r$ for all $\aaa'$ satisfying \eqref{eq-set4}. We next construct the alternative parameters $\bar\pp$ with $\bar p_{(0,\aaa')} = p_{(0,\aaa')} / f$ and $\bar p_{(1,\aaa')} =\bar p_{(0,\aaa')} \cdot \rho$. Then it is straightforward to see that \eqref{eq-set4} can be transformed into the following equations without losing any constraints,
\begin{align*}
	\begin{cases}
		(1 + \rho)f = (1+r)p_{(0,\aaa')},\\
		(\bar \theta^-_1 + \rho \bar \theta^+_1)f =  (\theta^-_1+r \theta^+_1)p_{(0,\aaa')},\\
		(\bar \theta^-_1 + \rho \bar \theta^+_1)f =  (\theta^-_2+r \theta^+_2)p_{(0,\aaa')},\\
		(\bar \theta^-_1 \bar \theta^-_2 + \rho \bar \theta^+_1 \bar \theta^+_1)f =  (\theta^-_1 \theta^-_2 +r \theta^+_1 \theta^+_2)p_{(0,\aaa')}.\\
	\end{cases}
\end{align*}
Then the above set of equations have four constraints for six free variables $(\rho,f, \bar \theta^-_1,\bar \theta^-_2,\bar \theta^+_1,\bar \theta^+_2)$, so there are infinitely many different sets of solutions to it. This shows the non-identifiability and proves $\sum_{k=1}^K q_{j,k}\geq 3$ is necessary requirement for identifiability.

We next prove part (b)   by construction. Consider the case where attribute 1 is a single attribute and attribute 2 is the common ancestor for all the remaining attributes $3,\ldots,K$. That is, $\mce=\{2\to 3,~2\to 4,~\ldots,~2\to K\}$. Under this hierarchy, consider the following $\QQ$,
\begin{align*}
%	\QQ=\begin{pmatrix}
%	    1 & 0 & 0 \\
%		1 & * & * \\
%		1 & * & * \\
%		0 & 1 & 0 \\
%		0 & 1 & 0 \\
%		0 & 1 & 1 \\
%		0 & 1 & 1 \\
%	\end{pmatrix}.
	\QQ=\begin{pmatrix}
	    1 & \zero^\top \\
		1 & \vv_1^\top \\
		1 & \vv_2^\top \\
	    \zero & \EE^\star \\
	    \zero & \EE^\star
	\end{pmatrix},
\end{align*}
where $\EE^\star$ is a $(K-1)\times(K-1)$ reachability matrix among attributes $2,3,\ldots,K$.
Note that the bottom right $2(K-1)\times(K-1)$ submatrix of $\QQ$ contains two copies of the reachability matrix among the last $K-1$ attributes.
Therefore, using a similar argument as that in the proof of Proposition  \ref{prop-two-e} can establish $\theta^+_j = \bar \theta^+_j$ and $\theta^-_j=\bar \theta^-_j$ for $j=\{4,\ldots,J\}$. Now define
$$
\ttt^\star = \theta^+_1\ee_1 + \bar \theta^-_2\ee_2 + \bar \theta^+_3\ee_3+ \sum_{j=4}^J \theta^-_j\ee_j,
$$
then 
\begin{align*}
\bar T_{\rr^\star,\bcolon}\bar\pp = 0 &  = T_{\rr^\star,\bcolon}\pp  \\
&= \left(\sum_{
%\aaa'\in\{0,1\}^{K-1}, \atop 
(0,\aaa')\in\mca} p_{(0,\aaa)}\right)(\theta^-_1-\theta^+_1)(\theta^-_2-\bar \theta^-_2)(\theta^-_3-\bar \theta^+_3)\prod_{j=4}^J (\theta^+_j-\theta^-_j).
\end{align*}
Since the factor $\sum_{(0,\aaa')\in\mca} p_{(0,\aaa)}$ in the above display is nonzero due to the assumption that attribute 1 is a singleton attribute, we have $\theta^-_2=\bar \theta^-_2$. By symmetry we also obtain $\theta^-_3=\bar \theta^-_3$. Then define
$$
\ttt^\star = \sum_{j=3}^J \theta^-_j\ee_j,
$$
then $T_{\rr^\star} \pp = \prod_{j=3}^J(\theta^+_j-\theta^-_j) p_{\one_K} \neq 0$, so there is $T_{\rr^\star} \pp =\bar T_{\rr^\star}\bar \pp \neq 0$.
Therefore we have 
\begin{align*}
	\frac{\bar T_{\rr^\star+\ee_2} \bar\pp}{\bar T_{\rr^\star} \bar\pp} =&~\frac{\bar \theta^+_2(\bar \theta^+_3-\theta^-_3)\prod_{j=4}^J( \theta^+_j-\theta^-_j) \bar p_{\one_K}}{(\bar \theta^+_3-\theta^-_3)\prod_{j=4}^J( \theta^+_j-\theta^-_j) \bar p_{\one_K}}\\
	= \frac{ T_{\rr^\star+\ee_2} \pp}{ T_{\rr^\star} \pp}
	=&~
	\frac{ \theta^+_2\prod_{j=3}^J(\theta^+_j-\theta^-_j)  p_{\one_K}}{\prod_{j=3}^J(\theta^+_j-\theta^-_j) p_{\one_K}},
	%\quad \Longrightarrow\quad \bar \theta^+_2=\theta^+_2.
\end{align*}
which gives $\bar \theta^+_2=\theta^+_2$.
Similarly we can obtain $\theta^+_1=\bar \theta^+_1$ and $\theta^+_3=\bar \theta^+_3$.  Define 
$$\ttt^\star = \theta^+_2\ee_2 + \sum_{j=4}^J \theta^-_j\ee_j,$$ 
then
\begin{align*}
	\frac{\bar T_{\rr^\star+\ee_1} \bar\pp}{\bar T_{\rr^\star} \bar\pp} 
	&=~\frac{\bar \theta^-_1(\theta^-_2-\theta^+_2)\prod_{j=4}^J( \theta^+_j-\theta^-_j) (\sum_{(0,\aaa')\in\mca}\bar p_{(0,\aaa)})}{\prod_{j=4}^J( \theta^+_j-\theta^-_j) (\sum_{(0,\aaa')\in\mca}\bar p_{(0,\aaa)})}\\
	=  \frac{ T_{\rr^\star+\ee_1} \pp}{ T_{\rr^\star} \pp}
	&=~\frac{ \theta^-_1(\theta^-_2-\theta^+_2)\prod_{j=4}^J(\theta^+_j-\theta^-_j) (\sum_{(0,\aaa')\in\mca} p_{(0,\aaa)})}{\prod_{j=4}^J(\theta^+_j-\theta^-_j) (\sum_{(0,\aaa')\in\mca} p_{(0,\aaa)})},
%	&~\frac{\bar \theta^-_1(\theta^-_2-\theta^+_2)\prod_{j=4}^J( \theta^+_j-\theta^-_j) (\sum_{(0,\aaa')\in\mca}\bar p_{(0,\aaa)})}{\prod_{j=4}^J( \theta^+_j-\theta^-_j) (\sum_{(0,\aaa')\in\mca}\bar p_{(0,\aaa)})}\\
%	=&~
%	\frac{ \theta^-_1(\theta^-_2-\theta^+_2)\prod_{j=4}^J(\theta^+_j-\theta^-_j) (\sum_{(0,\aaa')\in\mca} p_{(0,\aaa)})}{\prod_{j=4}^J(\theta^+_j-\theta^-_j) (\sum_{(0,\aaa')\in\mca} p_{(0,\aaa)})},
	%\quad \Longrightarrow\quad \bar \theta^-_1=\theta^-_1.
\end{align*}
which implies $\bar \theta^-_1 = \theta^-_1$.
Thus far we have shown $\bar\ttt^+=\ttt^+$ and $\bar\ttt^-=\ttt^-$. Following a similar argument as that in the end of the proof of Proposition \ref{prop-two-e} establishes $\bar\pp=\pp$ and $\bar{\mce}=\mce$. This completes the proof of part (b).
\end{proof}

\begin{proof}[Proof of Proposition \ref{prop-nece-anc}] 
	We first prove part (a) for an ancestor attribute $k$, then prove part (a) for a leaf attribute $k$, and finally prove part (b).
	
\bigskip
\noindent\textbf{Part (a) for an ancestor attribute.}	Suppose attribute 1 is an ancestor for which there is some attribute $h$ such that $1\to h$ but there does not exist any attribute $\ell$ such that $\ell\to 1$. Then clearly a valid topological order of the attributes can start with this attribute 1, and hence the corresponding reachability matrix $\EE$ is lower-triangular.
 Assume Condition A is satisfied, that is, $\QQ_{1:K,\bcolon}\stackrel{\mce}{\sim}I_K$.
Then if $\sum_{j=1}^J q^{\sparse}_{j,1}=1$ holds, the $\QQ$ and the corresponding $\mathcal S^{\mce}(\QQ)$ can be written in the following forms,
	\begin{align}\label{eq-only1}
%		\QQ=
%		\begin{pmatrix}
%			1 & 0 & \cdots & 0\\
%			1 & 1 & \cdots & 0\\
%			\vdots & \vdots & \ddots & \vdots \\
%			1 & * & \cdots & 1\\
%			\hline
%			1 & 1 & \cdots & 0\\
%			\vdots & \vdots & \ddots & \vdots \\
%			1 & * & \cdots & 1\\
%		\end{pmatrix}
%		\quad \Longrightarrow\quad
		%
		\QQ
        =\begin{pmatrix}
			\QQ_{1:K,\,\bcolon}\\
			\hline
			\QQ^{\star}
		\end{pmatrix}
		=
		\begin{pmatrix}
			1 & * & \cdots & *\\
			* & 1 & \cdots & *\\
			\vdots & \vdots & \ddots & \vdots \\
			* & * & \cdots & 1\\
			\hline
			\zero & \vdots & \vdots & \vdots \\
		\end{pmatrix}
		\quad \stackrel{\mce}{\Longrightarrow}\quad
		\mathcal S^{\mce}(\QQ)=
		\begin{pmatrix}
			1 & 0 & \cdots & 0\\
			0 & 1 & \cdots & 0\\
			\vdots & \vdots & \ddots & \vdots \\
			0 & 0 & \cdots & 1\\
			\hline
			\zero & \vdots & \vdots & \vdots \\
		\end{pmatrix}.
	\end{align}
We next construct $(\bar\ttt^+,\bar\ttt^-,\bar\pp)\neq (\ttt^+,\ttt^-,\pp)$ to show the nonidentifiability.
Given a set of valid model parameters $(\ttt^+,\ttt^-,\pp)$, we first take $\bar \theta^+_j=\theta^+_j$ and $\bar \theta^-_j=\theta^-_j$ for $j=2,\ldots,J$. 
Based on this, it is not hard to see that to ensure \eqref{eq-def} holds, we only need to ensure the following equations hold for any $\aaa'\in\{0,1\}^{K-1}$,
\begin{align*}
    \begin{cases}
		\bar p_{(0,\aaa')} + 
		\bar p_{(1,\aaa')} = p_{(0,\aaa')} + p_{(1,\aaa')},\\
		\bar \theta^-_1\bar p_{(0,\aaa')} + \bar \theta^+_1\bar p_{(1,\aaa')} = \theta^-_1 p_{(0,\aaa')} + \theta^+_1 p_{(1,\aaa')}.\\
	\end{cases}
\end{align*} 
%Now note that attribute 1 is a common ancestor attribute for all the other attributes, so $p_{(1,\aaa')}=0$ for all $\aaa'\neq\zero_{K-1}$. In other words, $p_{(1,\aaa')}>0$ only if $\aaa'=\zero_{K-1}$.
Now note that attribute 1 is an ancestor attribute for some attribute $h$, so $\aaa=\ee_h\neq\mca(\mce)$ and $p_{\ee_h}=0$.
%so $p_{\ee_1+\ee_h}=0$. 
Also by Definition \ref{def-anc}, attribute 1 is an ancestor attribute also implies that no attribute is a prerequisite for attribute 1, therefore 
%$p_{(0,\zero_{K-1})}>0$ and 
$p_{(1,\zero_{K-1})}>0$; also $p_{(0,\zero_{K-1})}>0$ under any hierarchy.
% More generally, 
Therefore the above set of equations equivalently become 
\begin{align}\notag
  & \text{(i) for}~\aaa'=\zero_{K-1},~\\
  \label{eq-3case1}
&  \qquad  \begin{cases}
    	\bar p_{(0,\zero_{K-1})} + \bar p_{(1,\zero_{K-1})} = p_{(0,\zero_{K-1})} + p_{(1,\zero_{K-1})}, \\
    	\bar \theta^-_1 \bar p_{(0,\zero_{K-1})} + \bar \theta^+_1 p_{(1,\zero_{K-1})} =  \theta^-_1 p_{(0,\zero_{K-1})} + \theta^+_1 p_{(1,\zero_{K-1})}; 
    \end{cases}
    \\ \notag 
    	%%% this one is additional; if 1 is the common-ancestor then the following can be dropped
	&\text{(ii) for}~\aaa'\neq\zero_{K-1}~\text{and}~(0,\aaa')\in\mca(\mce),~\\ 
	\label{eq-3case2}
	& \qquad\begin{cases}
		\bar p_{(0,\aaa')} + \bar p_{(1,\aaa')} = p_{(0,\aaa')}+p_{(1,\aaa')},\\
		\bar \theta^-_1\bar p_{(0,\aaa')} + \bar \theta^+_1 \bar p_{(1,\aaa')} = \theta^-_1 p_{(0,\aaa')} + \theta^+_1 p_{(1,\aaa')}; \\
	\end{cases}
    \\ \notag 
	%%%
	&\text{(iii) for}~\aaa'\neq\zero_{K-1}~\text{and}~(0,\aaa')\not\in\mca(\mce),~\\
 \label{eq-3case3}
 &\qquad	\begin{cases}
		\bar p_{(1,\aaa')} = p_{(1,\aaa')},\\
		\bar \theta^+_1\bar p_{(1,\aaa')} = \theta^+_1 p_{(1,\aaa')}. \\
	\end{cases}
\end{align}
In the above system of equations, we first point out that  Eq. \eqref{eq-3case3} are not empty constraints due to the assumption that attribute 1 is an ancestor attribute. This is because for some attribute $h$ such that $1\to h$, as stated earlier there is $\aaa=\ee_h=(0,\aaa')\not\in\mca(\mce)$ belongs to case (iii) in Eq. \eqref{eq-3case3}. 
Therefore the constraints in Eq. \eqref{eq-3case3} gives $\bar \theta^+_1 = \theta^+_1$. Given this, the set of equations in Eq. \eqref{eq-3case1}--\eqref{eq-3case3} can be further equivalently written as
\begin{align}\notag
	&\text{for}~\aaa=(0,\aaa')\in\mca(\mce),\\
	\label{eq-ansfi}
	&	\qquad \begin{cases}
		\bar p_{(0,\aaa')} + \bar p_{(1,\aaa')} = p_{(0,\aaa')}+p_{(1,\aaa')},\\
		\bar \theta^-_1\bar p_{(0,\aaa')} + \theta^+_1 \bar p_{(1,\aaa')} = \theta^-_1 p_{(0,\aaa')} + \theta^+_1 p_{(1,\aaa')}. \\
	\end{cases}
\end{align}
Recall that \eqref{eq-def} hold if the above equations \eqref{eq-ansfi} hold. Define the set $$\mca_{\text{sub}}=\{\aaa\in\{0,1\}^K:\,\aaa=(0,\aaa')\in\mca(\mce)\}.$$
Then the above Eq. \eqref{eq-ansfi} involve $2|\mca_{\text{sub}}|$ number of constraints for $2|\mca_{\text{sub}}|+1$ number of free variables in $$\{\bar \theta^-_1\}\cup\{\bar p_{\aaa}:\,\aaa\in\mca_{\text{sub}}\}\cup\{\bar p_{\aaa+\ee_1}:\,\aaa\in\mca_{\text{sub}}\}.$$ Therefore there are infinitely many different solutions to Eq. \eqref{eq-ansfi}. This proves the nonidentifiability under $\mce$ and the $\QQ$ in \eqref{eq-only1} and concludes the proof of part (a) for an ancestor attribute.

%%%%%%%%%%%%%%%%%%%%%%%%%%%%%%%%%%%%%%%%%%%%%%%%
\bigskip
\noindent\textbf{Part (a) for a leaf attribute.} 
Suppose attribute $K$ is a leaf attribute for which there is some attribute $\ell$ such that $\ell\to K$ but there does not exist any attribute $h$ such that $K\to h$. Then clearly a valid topological order of the attributes can end with this attribute $K$, and hence the corresponding reachability matrix $\EE$ is lower-triangular.
 Assume Condition A is satisfied, that is, $\QQ_{1:K,\bcolon}\stackrel{\mce}{\sim}I_K$.
Then if $\sum_{j=1}^J q^{\sparse}_{j,K}=1$ holds, the $\QQ$ and the corresponding $\mathcal S^{\mce}(\QQ)$ can be written in the following forms,
	\begin{align}\label{eq-only1}
		\QQ
	    =\begin{pmatrix}
			\QQ_{1:K,\,\bcolon}\\
			\hline
			\QQ^{\star}
		\end{pmatrix}
		=
		\begin{pmatrix}
			1 & \cdots & 0 & 0\\
			\vdots & \ddots & \vdots  & \vdots \\
			* & \cdots & 1 & 0\\
			* & \cdots & * &  1\\
			\hline
			\vdots & \vdots & \vdots & \vdots \\
		\end{pmatrix}
		\quad \stackrel{\mce}{\Longrightarrow}\quad
		\mathcal S^{\mce}(\QQ)=
		\begin{pmatrix}
			1 & \cdots & 0 & 0\\
			\vdots & \ddots & \vdots  & \vdots \\
			0 & \cdots & 1 & 0\\
			0 & \cdots & 0 &  1\\
			\hline
			\vdots & \vdots & \vdots & \zero \\
		\end{pmatrix}.
	\end{align}
We next construct $(\bar\ttt^+,\bar\ttt^-,\bar\pp)\neq (\ttt^+,\ttt^-,\pp)$ to prove the nonidentifiability.
Given a set of valid parameters $(\ttt^+,\ttt^-,\pp)$, we first take $\bar \theta^+_j=\theta^+_j$ and $\bar \theta^-_j=\theta^-_j$ for $j\in[J]\setminus\{K\}$. 
Proceeding in a similar spirit to the earlier proof for an ancestor attribute, we can get the following: in order to ensure \eqref{eq-def} holds, it suffices to ensure the following equations hold for any $\aaa'\in\{0,1\}^{K-1}$,
\begin{align}\label{eq-leaf1}
    \begin{cases}
		\bar p_{(\aaa',0)} + 
		\bar p_{(\aaa',1)} = p_{(\aaa',0)} + p_{(\aaa',1)},\\
		\bar \theta^-_K\bar p_{(\aaa',0)} + \bar \theta^+_K\bar p_{(\aaa',1)} = \theta^-_K p_{(\aaa',0)} + \theta^+_K p_{(\aaa',1)}.\\
	\end{cases}
\end{align} 
Because attribute $K$ is a leaf attribute, there exists some attribute $\ell$ such that $\ell\to K$ and hence $\aaa=\ee_\ell\not\in\mca(\mce)$ and $p_{\ee_\ell}=0$.
Also by Definition \ref{def-leaf}, attribute K is a leaf attribute also implies that it does not serve as a prerequisite for any other attribute, therefore 
$\aaa=(\one_{K-1},0)\in\mca(\mce)$ and
$p_{(\one_{K-1},0)}>0$; also $p_{(\one_{K-1},1)}>0$ under any hierarchy.
% More generally, 
Therefore the above set of equations \eqref{eq-leaf1} equivalently become 
\begin{align}\notag
   &\text{(i) for}~\aaa'=\one_{K-1},
   \\
   \label{eq-leafc1}
   &\qquad
    \begin{cases}
    	\bar p_{(\one_{K-1},0)} + \bar p_{(\one_{K-1},1)} =  p_{(\one_{K-1},0)} + p_{(\one_{K-1},1)},  \\
    	\bar \theta^-_K\bar p_{(\one_{K-1},0)} + \bar \theta^+_K\bar p_{(\one_{K-1},1)} = \theta^-_K p_{(\one_{K-1},0)} + \theta^+_K p_{(\one_{K-1},1)};
    \end{cases}
    \\ \notag
	%%% this one is additional; if 1 is the common-ancestor then the following can be dropped
	&\text{(ii) for}~\aaa'\neq\one_{K-1}~\text{and}~(\aaa',0)\in\mca(\mce),\\
	\label{eq-leafc2}
	&\qquad
	\begin{cases}
		&\bar p_{(\aaa',0)} + \bar p_{(\aaa',1)} =  p_{(\aaa',0)} + p_{(\aaa',1)},  \\
    	&\bar \theta^-_K\bar p_{(\aaa',0)} + \bar \theta^+_K\bar p_{(\aaa',1)}   = \theta^-_K p_{(\aaa',0)} + \theta^+_K p_{(\aaa',1)};
	\end{cases}
    \\ \notag
	%%%
	&\text{(iii) for}~\aaa'\neq\one_{K-1}~\text{and}~(\aaa',0)\not\in\mca(\mce),\\
	\label{eq-leafc3}
	&\qquad
	\begin{cases}
		\bar p_{(\aaa',1)} = p_{(\aaa',1)},\\
		\bar \theta^+_K\bar p_{(\aaa',1)} = \theta^+_K p_{(\aaa',1)}.\\
	\end{cases}
\end{align}
We point out that the above Eq. \eqref{eq-leafc3} are not empty constraints. This is because as stated earlier, there is $\aaa=\ee_\ell=(\aaa',0)\not\in\mca(\mce)$ with $\aaa'\neq\one_{K-1}$. Therefore this particular $\aaa'$ falls into the case (iii) in Eq. \eqref{eq-leafc3}. Based on this, Eq. \eqref{eq-leafc3} gives $\bar \theta^+_K=\theta^+_K$ and the set of equations \eqref{eq-leaf1} equivalently becomes
\begin{align}\notag
		& \text{for}~\aaa=(\aaa',0)\in\mca(\mce),\\
		\label{eq-leafi}
	&\qquad
	\begin{cases}
		\bar p_{(\aaa',0)} + \bar p_{(\aaa',1)} =  p_{(\aaa',0)} + p_{(\aaa',1)}, \\
    	\bar \theta^-_K\bar p_{(\aaa',0)} + \theta^+_K\bar p_{(\aaa',1)} = \theta^-_K p_{(\aaa',0)} + \theta^+_K p_{(\aaa',1)}.
	\end{cases}
\end{align}
Recall that \eqref{eq-def} hold if the above equations \eqref{eq-leafi} hold. Define the set $$\mca'_{\text{sub}}=\{\aaa\in\{0,1\}^K:\,\aaa=(\aaa',0)\in\mca(\mce)\}.$$
Then the above Eq.~\eqref{eq-leafi} involve $2|\mca'_{\text{sub}}|$ number of constraints for $2|\mca'_{\text{sub}}|+1$ number of free variables in 
$$\{\bar \theta^-_K\}\cup\{\bar p_{\aaa}:\,\aaa\in\mca'_{\text{sub}}\}\cup\{\bar p_{\aaa+\ee_K}:\,\aaa\in\mca'_{\text{sub}}\}.$$ 
So there are infinitely many different solutions to Eq.~\eqref{eq-ansfi}. This proves the nonidentifiability under $\mce$ and the $\QQ$ in \eqref{eq-only1} and completes the proof of part (a) for a leaf attribute.

%%%%%%%%%%%%%%%%%%%%%%%%%%%%%%%%%%%%%%%%%%%%%%%%
\bigskip
\noindent\textbf{Part (b) for either an ancestor attribute or a leaf attribute.} To prove the conclusion in part (b) for an ancestor attribute, we only need to note that the previous Proposition  \ref{prop-two-e} established identifiability when some attribute is a common ancestor for all other attributes and the $\mathcal S^{\mce}(\QQ)$ there includes only two ``1'' in the column corresponding to this ancestor attribute.

%On the other hand, the proof of the sufficiency part of Proposition \ref{prop-linear} establishes identifiability when some attribute is a common leaf for all other attributes and $\mathcal S^{\mce}(\QQ)$ includes only two ``1'' in the column corresponding to this leaf attribute.
We next prove the conclusion in part (b) for a leaf attribute.
To this end, we prove the following $(K+2)\times K$ matrix $\QQ$ under the $\mce=\{1\to 2\to \cdots \to K\}$ makes the model identifiable.
\begin{align}\label{eq-linearq}
	\QQ=\begin{pmatrix}
        & &\EE & & \\
		\hline
		1 & 0 & \cdots & 0 & 0\\
		1 & 1 & \cdots & 1 & 1
	\end{pmatrix}
	\quad\Longrightarrow\quad
	\mathcal S^{\mce}(\QQ) =\begin{pmatrix}
%		1 & 0 & \cdots & 0\\
%		0 & 1 & \cdots & 0\\
%		\vdots & \vdots & \ddots & \vdots\\
%		0 & 0 & \cdots & 1\\
        & & I_K & & \\
		\hline
		1 & 0 & \cdots & 0 & 0\\
		0 & 0 & \cdots & 0 & 1
	\end{pmatrix}
\end{align}
where the first $K$ rows of $\QQ$ form a lower-triangular matrix with all the lower triangular entries equal to one, the $(K+1)$th row equals $\ee_1$ and the $(K+2)$th row equals $\one_K$. Then its corresponding $\mathcal S^{\mce}(\QQ)$ also shown in \eqref{eq-linearq} satisfies Condition D in the proposition. We next prove $\bar\ttt^+=\ttt^+$, $\bar\ttt^-=\ttt^-$, and $\bar\pp=\pp$ from $T(\bar\ttt^+,\bar\ttt^-)\bar\pp=T(\ttt^+,\ttt^-)\pp$ to establish identifiability.

First, define
$$
\ttt^\star=\bar \theta^-_1\ee_1 + \bar \theta^+_{K+1}\ee_{K+1} + \sum_{k=2}^K \theta^-_k\ee_k,
$$
then the form of $\QQ$ in \eqref{eq-linearq} ensures $\bar T_{\rr^\star,\aaa}=0$ for all $\aaa\in\{0,1\}^K$.
On the other hand, since $1\to k$ for all $k=2,\ldots,K$, we have that if $\aaa\succeq \bq_k$ for all  $k=2,\ldots,K$, there must be $\aaa=\one_K$.
Therefore
$$
\bar T_{\rr^\star,\bcolon} \bar\pp=0= T_{\rr^\star,\bcolon} \pp
=p_{\one_K}(\theta^+_1-\bar \theta^-_1)(\theta^+_{K+1}-\bar \theta^+_{K+1})\prod_{k=2}^K(\theta^+_k-\theta^-_k),
$$
which gives $\theta^+_{K+1}=\bar \theta^+_{K+1}$. By symmetry we also obtain $\theta^+_1=\bar \theta^+_1$. 

Second, based on this, we define $$\ttt^\star=\bar \theta^-_1\ee_1 +  \theta^+_{K+1}\ee_{K+1}.$$ 
We still have $\bar T_{\rr^\star,\aaa}=0$ for all $\aaa\in\{0,1\}^K$. Again since $1\to k$ for all $k=2,\ldots,K$, we have that if $\aaa\nsucceq\bq_1$, there must be $\aaa=\zero_K$. Therefore
$$
\bar T_{\rr^\star,\bcolon} \bar\pp=0= T_{\rr^\star,\bcolon} \pp=
p_{\zero_K}(\theta^-_1-\bar \theta^-_1)(\theta^-_{K+1}-\theta^+_{K+1}),
$$ 
which gives $\theta^-_1=\bar \theta^-_1$. By symmetry we also get $\theta^-_{K+1}=\bar \theta^-_{K+1}$.

Third, for $j\in\{2,\ldots,K\}\cup\{K+2\}$ we define
$$
\ttt^\star = \theta^+_1\ee_1+\bar \theta^-_j\ee_j,
$$
then based on the form of $\QQ$ and $\bar \theta^+_1=\theta^+_1$ we have $\bar T_{\rr^\star,\aaa}=0$ for all $\aaa\in\{0,1\}^K$. So
$$
\bar T_{\rr^\star,\bcolon} \bar\pp=0= T_{\rr^\star,\bcolon} \pp=
p_{\zero_K}(\theta^-_1-\theta^+_1)(\theta^-_j-\bar \theta^-_j),
$$
which implies $\theta^-_j=\bar \theta^-_j$ for $j\in\{2,\ldots,K\}\cup\{K+2\}$. Note that thus far we have shown $\bar\ttt^-=\ttt^-$.

Forth, for $k\in\{2,\ldots,K-1\}$ we define
$$
\ttt^\star = \bar \theta^+_k\ee_k + \theta^-_K\ee_K.
$$
Note that we have already established $\theta^-_K=\bar \theta^-_K$, so $\bar T_{\rr^\star,\aaa}=0$ for all $\aaa$ and $T_{\rr^\star,\aaa}\neq 0$ only potentially for $\aaa=\one_K$ under the hierarchy. Therefore
$$
\bar T_{\rr^\star,\bcolon} \bar\pp=0= T_{\rr^\star,\bcolon} \pp=p_{\one_K}(\theta^+_k-\bar \theta^+_k)(\theta^+_K-\theta^-_K),
$$
which gives $\theta^+_k=\bar \theta^+_k$ for $k\in\{2,\ldots,K-1\}$.

Now it remains to show $\theta^+_{K}=\bar \theta^+_{K}$, $\theta^+_{K+2}=\bar \theta^+_{K+2}$ and $\bar\pp=\pp$. We define
$$
\ttt^\star=\theta^-_K\ee_K.
$$
Note that $\bar \theta^-_K=\theta^-_K$, so for this $\ttt^\star$, the $\bar T_{\ee_K,\bcolon}\bar\pp=T_{\ee_K,\bcolon}\pp$ gives
\begin{equation}\label{eq-ns1}
\bar p_{\one_K} (\bar \theta^+_K-\theta^-_K)
=
p_{\one_K} (\theta^+_K-\theta^-_K)\neq 0.
\end{equation}
We further define 
$$
\ttt^\star=\theta^-_K\ee_K+\theta^-_{K+2}\ee_{K+2},
$$
then $\bar T_{\ee_K+\ee_{K+2},\bcolon}\bar\pp=T_{\ee_K+\ee_{K+2},\bcolon}\pp$ gives
\begin{equation}\label{eq-ns2}
\bar p_{\one_K} (\bar \theta^+_K-\theta^-_K)(\bar \theta^+_{K+2}-\theta^-_{K+2})
=
p_{\one_K} (\theta^+_K-\theta^-_K)( \theta^+_{K+2}-\theta^-_{K+2}).
\end{equation}
Taking the ration of \eqref{eq-ns2} and \eqref{eq-ns1} gives $\bar \theta^+_{K+2}=\theta^+_{K+2}$. By symmetry we also have $\bar \theta^+_K=\theta^+_K$. 
Thus far we have shown $\bar\ttt^-=\ttt^-$ and $\bar\ttt^+=\ttt^+$. Then following a similar argument as that in the end of the proof of Theorem \ref{prop-two-e} gives $\bar\pp=\pp$, which establishes the identifiability of the hierarchy $\mce$ and all the model parameters under the linear hierarchy. 

The above proof establishes identifiability when some attribute is a common leaf for all other attributes and $\mathcal S^{\mce}(\QQ)$ includes only two ``1'' in the column corresponding to this leaf attribute.
This proves the conclusion of part (b).
\end{proof}

\begin{proof}[Proof of Proposition \ref{prop-nece-int}]
For part (a), the necessity of $\sum_{j=1}^J q^{\sparse}_{j,k}\geq 1$ is guaranteed by the statement in Theorem \ref{thm-main} that Condition A is necessary for identifiability. This is because if $\sum_{j=1}^J q^{\sparse}_{j,k} = 0$, then Condition A is violated. This proves part (a).

For part (b), we next construct a case where the model is proved to be identifiable but $\sum_{j=1}^J q^{\sparse}_{j,k}= 1$ for an intermediate attribute $k$.
In particular, still consider the case that attribute 1 is a common ancestor for all other attributes 2 through $K$. 
	Suppose attribute 2 is an intermediate attribute such that $2\to k$ for any $k\in\{3,\ldots,K\}$.
	That is, 
	$$\mce=\{1\to 2\}\cup\{2\to k:\,k=3,\ldots,K\}.$$
	Consider the following $\QQ$ and the corresponding $\mathcal S^{\mce}(\QQ)$,
\begin{align}\label{eq-int}
		\QQ=
		\begin{pmatrix}
			1 & 0 & 0 & \cdots & 0\\
			1 & 0 & 0 & \cdots & 0\\
			0 & 1 & 0 & \cdots & 0\\
%			\one_{K-1} & \one_{K-1} &   & \EE^{\star} & \\
%			\one_{K-1} & \one_{K-1} &   & \EE^{\star} & \\
			\one & \one &   & \EE^{\star} & \\
			\one & \one &   & \EE^{\star} & \\
		\end{pmatrix}
		\quad \stackrel{\mce}{\Longrightarrow}\quad
		\mathcal S^{\mce}(\QQ)=
		\begin{pmatrix}
			1 & 0 & 0 & \cdots & 0\\
			1 & 0 & 0 & \cdots & 0\\
			0 & 1 & 0 & \cdots & 0\\
%			\one_{K-1} & \one_{K-1} &   & \EE^{\star} & \\
%			\one_{K-1} & \one_{K-1} &   & \EE^{\star} & \\
			\zero & \zero &   & \EE^{\star} & \\
			\zero & \zero &   & \EE^{\star} & \\
		\end{pmatrix}
\end{align}
We next establish identifiability of $(\mce,\ttt^+,\ttt^-,\pp)$ under the $\QQ$ in \eqref{eq-int}.
First note that if deleting the third row and the second column in $\QQ$, the remaining $2(K-1)\times (K-1)$ submatrix can be viewed as containing two copies of the reachability matrix under the subgraph $\mce^\star=\{1\to k:\, k=3,4,\ldots,K\}$ among the subset of attributes $\{1,3,4,\ldots,K\}$. Therefore, following a similar argument as the proof of Proposition  \ref{prop-two-e}, we obtain $\bar \theta^+_j=\theta^+_j$ and $\bar \theta^-_j = \theta^-_j$ for $j\in\{1,2\}\cup\{4,\ldots,J\}$. It remains to show $\bar \theta^+_3=\theta^+_3$, $\bar \theta^-_3=\theta^-_3$ and $\bar\pp=\pp$.
First, define
$$
\ttt^\star= \bar \theta^+_1\ee_1+\bar \theta^-_3\ee_3,
$$
then $\bar T_{\rr^\star,\bcolon}=\zero$ under $\QQ$. So 
$$
\bar T_{\rr^\star,\bcolon} \bar\pp = 0 = T_{\rr^\star,\bcolon} \pp
= p_{\zero_K}(\theta^-_1-\bar \theta^+_1)(\theta^-_3-\bar \theta^-_3),
$$
which implies $\theta^-_3=\bar \theta^-_3$. Second, define
$$
\ttt^\star=\theta^-_1\ee_1+ \bar \theta^+_3\ee_3+\sum_{j=4}^J\bar \theta^-_j\ee_j,
$$
then we have $\bar T_{\rr^\star,\bcolon}=\zero$. This is because there is no such $\aaa$ for which $\aaa\succeq\bq_1$ and $\aaa\nsucceq \bq_3$, so $\bar T_{\rr^\star,\aaa}$ can not contain a nonzero factor $(\theta^+_1-\theta^-_1)(\bar \theta^-_3-\bar \theta^+_3)$ and must be always zero. Given this, we have
$$
\bar T_{\rr^\star,\bcolon} \bar\pp = 0 = T_{\rr^\star,\bcolon} \pp
= p_{\one_k}(\theta^+_1-\theta^-_1)(\theta^+_3-\bar \theta^+_3)\prod_{j=4}^J(\theta^+_j-\theta^-_j),
$$
which gives $\theta^+_3=\bar \theta^+_3$. Thus far we have shown $\bar\ttt^+=\ttt^+$ and $\bar\ttt^-=\ttt^-$. Now following a similar proof as those in the end of the proof of Proposition  \ref{prop-two-e} we obtain $\bar\pp=\pp$. This also establishes that $\mce$ is identifiable. The proof is complete.
\end{proof}

\begin{proof}[Proof of Proposition \ref{thm-nec-dist}]
Suppose the first $K$ rows of $\QQ$ is equivalent to $I_K$ under the hierarchy $\mce$, that is, Condition A is satisfied. 
%The proof of part (c) is implied by the proof of Proposition \ref{prop-linear}, where we  established identifiability for a case where $\QQ_{(K+1):J,\,k}=\QQ_{(K+1):J,\,\ell}$ for some attributes $k$ and $\ell$ falling in case (c); that will prove the Condition C in this scenario is not necessary.
%In the following, we prove part (a) and part(b), respectively.
%The following statements hold. In the following, we consider two cases separately: case (1), the two attributes $k$ and $\ell$ are not connected in the attribute hierarchy; and case (2), the two attributes are connected in the hierarchy. %, e.g., $k\to\ell$.
%When there is no path from $k$ to $\ell$ nor path from $\ell$ to $k$. 
Without loss of generality, suppose attributes 1 and 2 are singleton attributes, so $1\not\to 2$ and $2\not\to 1$. To prove by contradiction, consider the following $\QQ$ with the first $K$ rows forming the reachability matrix $\EE$,
\begin{align*}
	\QQ=\begin{pmatrix}
 	   & \EE &     \\
  \hline
 \vv & \vv \quad \vdots \quad  \vdots & \vdots 
	\end{pmatrix}.
\end{align*}
Since $1\not\to 2$ and $2\not\to 1$, the first two rows of the above $\QQ$ takes the following form,
$$
\QQ_{1:2,\,\bcolon}=\EE_{1:2,\,\bcolon}=\begin{pmatrix}
	1 & 0 & * & \cdots & * \\
	0 & 1 & * & \cdots & * \\
\end{pmatrix}.
$$
We next construct $(\ttt^+,\ttt^-,\pp)\neq(\bar\ttt^+,\bar\ttt^-,\bar\pp)$ such that \eqref{eq-def} holds. For any valid set of model parameters $(\ttt^+,\ttt^-,\pp)$, we first set $\bar \ttt^+=\ttt^+$ and $\bar \theta^-_j=\theta^-_j$ for all $j\geq 3$.
Then following a similar argument as that in the proof of Theorem 1 in \cite{id-dina} we obtain that, $T(\bar\ttt^+,\bar\ttt^-)\bar \pp=T(\ttt^+,\ttt^-)\pp$ hold as long as the following set of equations hold for any $\aaa'$
\begin{align}\label{eq-dist}
	\begin{cases}
& p_{(0,0,\aaa')} + p_{(1,0,\aaa')} + p_{(0,1,\aaa')}\\
&\qquad \qquad 
  = \bar p_{(0,0,\aaa')} +\bar p_{(1,0,\aaa')} +\bar p_{(0,1,\aaa')}, \\ 
& \theta^-_1   [p_{(0,0,\aaa')} + p_{(0,1,\aaa')}] + \theta^+_1  p_{(1,0,\aaa')} \\
&\qquad \qquad 
 = \bar \theta^-_1   [\bar p_{(0,0,\aaa')} + \bar p_{(0,1,\aaa')}] + \theta^+_1 \bar p_{(1,0,\aaa')} ,\\ 
& \theta^-_2   [p_{(0,0,\aaa')} + p_{(1,0,\aaa')}] + \theta^+_2  p_{(0,1,\aaa')}\\
&\qquad\qquad 
  = \bar \theta^-_2   [\bar p_{(0,0,\aaa')} + \bar p_{(1,0,\aaa')}] + \theta^+_2 \bar p_{(0,1,\aaa')} ,\\ 
& \theta^-_1  \theta^-_2   p_{(0,0,\aaa')} + \theta^+_1  \theta^-_2  p_{(1,0,\aaa')} + \theta^-_1  \theta^+_2  p_{(0,1,\aaa')}  
\\
&\qquad \qquad 
=\bar \theta^-_1 \bar \theta^-_2  \bar  p_{(0,0,\aaa')} + \theta^+_1  \bar \theta^-_2  \bar p_{(1,0,\aaa')} +\bar  \theta^-_1 \theta^+_2  \bar p_{(0,1,\aaa')}.
	\end{cases}
\end{align}
An important observation before proceeding with the proof is the following. The assumption $1\not\to 2$ and $2\not\to 1$ ensures that for any $\aaa'\in\{0,1\}^{K-2}$, we must have either (1) $p_{(0,0,\aaa')}>0$, $p_{(1,0,\aaa')}>0$, $p_{(0,1,\aaa')}>0$ holds simultaneously; or (2) $p_{(0,0,\aaa')}=p_{(1,0,\aaa')}=p_{(0,1,\aaa')}=0$.
To show nonidentifiability, we next focus on those $(\ttt^+,\ttt^-,\pp)$ under which for  any $\aaa^*\in\{0,1\}^{K-2}$, there is 
${p_{(0,1,\aaa^*)}}/{p_{(0,0,\aaa^*)}} = u$ and
${p_{(1,0,\aaa^*)}}/{p_{(0,0,\aaa^*)}} = v$,
where $u,v>0$ are some constants. 
Then we take $ \bar\pp$ such that for any $\aaa^*\in\{0,1\}^{K-2}$
$$
p_{(1,1,\aaa^*)} =   \bar p_{(1,1,\aaa^*)},~~
\bar p_{(0,0,\aaa^*)} = \bar\rho \cdot p_{(0,0,\aaa^*)},~~
\frac{\bar p_{(0,1,\aaa^*)}}{\bar p_{(0,0,\aaa^*)}} = \bar u, ~~
\frac{\bar p_{(1,0,\aaa^*)}}{\bar p_{(0,0,\aaa^*)}} = \bar v,
$$
for some constants $\bar\rho,\bar u,\bar v>0$.
In particular, we take $\bar\rho$ close to 1, then \eqref{eq-dist} equivalently becomes
\begin{align}\label{eq-t3}
\begin{cases}
& (1+u+v) = \bar\rho  (1+\bar u+\bar v), \\ 
& \theta^-_1  (1+u) + \theta^+_1 v = \bar\rho ~[~\bar \theta^-_1  (1+\bar u) + \theta^+_1  \bar v~],\\ 
& \theta^-_2 (1+v) + \theta^+_2  u = \bar \rho~ [~\bar \theta^-_2 (1+\bar v) + \theta^+_2  \bar u~],\\ 
& \theta^-_1  \theta^-_2+\theta^-_1  \theta^+_2  u+\theta^+_1  \theta^-_2  v  =\bar\rho ~ [~\bar \theta^-_1  \bar \theta^-_2+\bar \theta^-_1  \theta^+_2  \bar u+ \theta^+_1  \bar \theta^-_2  \bar v~].
\end{cases}
\end{align}
The above system of equations involve 4 constraints for 5 free variables $\bar\rho$, $\bar u$, $\bar v$, $\bar \theta^-_1$ and $\bar \theta^-_2$, so there are infinitely many sets of solutions of $(\bar\rho, \bar u, \bar v, \bar \theta^-_1,  \bar \theta^-_2)$. 
This shows the non-identifiability and concludes the proof of the proposition.

\end{proof}

%%%%%%%%%%%%% Section 4.2 %%%%%%%%%%%%%
%\color{blue!70!black}

\section{Proofs of Results in Section 4.2 Bridging the Necessary and Sufficient Identifiability Conditions}\label{pf-bridge}

In the following, we first prove Corollary \ref{thm-connected}, and then prove Theorem \ref{thm-general} building upon the some intermediate steps in the proof of Corollary \ref{thm-connected}.

\begin{proof}[Proof of Corollary \ref{thm-connected}]
We only need to prove the sufficiency of the conditions for identifiability, because the necessity of Condition A was shown in the proof of Proposition \ref{prop-Q} and the necessity of Condition D was established in Propositions \ref{prop-nece-anc}--\ref{prop-nece-int}.

Under Condition A, without loss of generality, suppose the first $K$ rows of the $\QQ$-matrix equals the reachability matrix $\EE$ with $\QQ_{1:K,\bcolon} = \EE$; in other words, $\bq_k = \mathcal D^{\mce}(\ee_k)$ for all $k=1,\ldots,K$.

\medskip
\noindent\textbf{Step 1.~}
First consider an ancestor attribute $k$. Under Condition D, suppose there are two row vectors in $\mathcal S^{\mce}(\QQ)$ indexed by $k$ and $j_k$ (where $j_k >K$) that both measure the attribute $k$; that is, $\bq_k = \bq_{j_k} = \mathcal D^{\mce}(\ee_k) = \ee_k$ since $k$ is an ancestor.
Define 
$$\ttt^\star = \bar\theta_k^- \ee_k + \bar\theta_{j_k}^+ \ee_{j_k} + \sum_{1\leq m\leq K,\, m\neq k} \theta_m^- \ee_m,
\quad
\rr^\star = \text{support}(\ttt^\star),
$$
then $\bar {\mathbf T}_{\rr^\star,\bcolon}(\bar\ttt^+ - \ttt^\star, ~\bar\ttt^- - \ttt^\star) \bar \pp = 0 = \mathbf T_{\rr^\star,\bcolon}(\ttt^+ - \ttt^\star, ~\ttt^- - \ttt^\star) \pp$ for all $\aaa$ due to the first two terms in the above $\ttt^\star$ and that $\bq_k = \bq_{j_k} = \ee_k$.
For any $\aaa\in\mca(\mce)$, if $\aaa\succeq \bq_m$ for all $m\in[K]\setminus\{k\}$, then there must be $\aaa\succeq \bq_k(=\ee_k)$ due to the fact that attribute $k$ is an ancestor attribute. This is because there must exist some $m\in[K]\setminus\{k\}$ such that $k\to m$, which further implies $\bq_m \succeq \bq_k$ since $\QQ_{1:K,\bcolon} = \EE$.
So
\begin{align*}
	  0 = {\mathbf T}_{\rr^\star,\bcolon}(\ttt^+ - \ttt^\star, ~\ttt^- - \ttt^\star)\pp
	= p_{\one_K} (\theta_k^+ - \bar\theta_k^-) (\theta_{j_k}^+ - \bar\theta_{j_k}^+) 
	\prod_{1\leq m\leq K,\, m\neq k} (\theta_m^+ - \theta_m^-),
\end{align*}
so we have $\theta_{j_k}^+ = \bar\theta_{j_k}^+$. Similarly we have $\theta_{k}^+ = \bar\theta_{k}^+$.
Still consider an ancestor attribute $k$ and define 
$$\ttt^\star = \bar\theta_k^- \ee_k + \theta_{j_k}^+ \ee_{j_K}.$$
Since we have shown $\theta_{j_k}^+ = \bar\theta_{j_k}^+$, there is $\bar {\mathbf T}_{\rr^\star,\aaa} (\bar\ttt^+ - \ttt^\star, ~\bar\ttt^- - \ttt^\star) =\zero$ for all $\aaa$ and hence $\bar {\mathbf T}_{\rr^\star,\bcolon}(\bar\ttt^+ - \ttt^\star, ~\bar\ttt^- - \ttt^\star) \bar \pp = 0 = \mathbf T_{\rr^\star,\bcolon}(\ttt^+ - \ttt^\star, ~\ttt^- - \ttt^\star) \pp$. For any $\aaa \in \mca(\mce)$, if $\aaa \nsucceq \bq_{j_k}$ there must be $\aaa \nsucceq \bq_{k}$, so
$$
 0 = \mathbf T _{\rr^\star,\bcolon}(\ttt^+ - \ttt^\star, ~\ttt^- - \ttt^\star)\pp
   = \Big(\sum_{\aaa:\, \aaa\nsucceq \bq_k} p_{\aaa} \Big) (\theta_{k}^- - \bar\theta_k^-) (\theta_{j_k}^- - \theta_{j_k}^+).
$$
This implies $\theta_{k}^- = \bar\theta_k^-$. By symmetry we can similarly obtain $\theta_{j_k}^- = \bar\theta_{j_k}^-$.
Now we have proved that for any ancestor attribute $k$, the item parameters associated with items $k$ and $j_k$ are identifiable.

\medskip
\noindent\textbf{Step 2.~}
Consider a leaf or an intermediate attribute $k$ and we next show $\theta_k^- = \bar\theta_k^-$.
Define
\begin{align}\label{eq-con-an}
	\ttt^\star=
	\bar\theta_k^-\ee_k + \sum_{m\text{ is ancestor}} \bar\theta_m^+ \ee_m,\quad
	\rr^\star = \text{support}(\ttt^\star).
\end{align}
%\darkred{Now that we are considering a connected graph hierarchy, enumerating all the ancestor attributes as in the above definition \eqref{eq-con-an} of $\ttt^\star$ must include a prerequisite for the attribute $k$.} 
Since attribute $k$ is a leaf or an intermediate attribute, enumerating all the ancestor attributes as in the above definition \eqref{eq-con-an} of $\ttt^\star$ must include a prerequisite for $k$.
In other words, there exists an ancestor attribute $m$ such that $m\to k$. Therefore for any $\aaa\in\{0,1\}^K$, if $\aaa\succeq \bar\bq_k$ (so $\bar\theta_{k,\aaa} = \bar \theta^+_k$) there must be $\aaa\succeq \bar\bq_m$ (so $\bar\theta_{m,\aaa} = \bar\theta^+_m$) for some ancestor attribute $m$. 
Since $\bar {T}_{\rr^\star,\aaa}(\bar\ttt^+ - \ttt^\star, ~\bar\ttt^- - \ttt^\star)$ contains a factor of $(\bar\theta_{k,\aaa} - \bar \theta^+_k) (\bar\theta_{m,\aaa} - \bar\theta^+_m)$, the above argument indeed proves that $\bar {T}_{\rr^\star,\aaa}(\bar\ttt^+ - \ttt^\star, ~\bar\ttt^- - \ttt^\star) = 0$ for any $\aaa\in\{0,1\}^K$ and hence $\bar {\mathbf T}(\bar\ttt^+ - \ttt^\star, ~\bar\ttt^- - \ttt^\star) \bar \pp = 0 = \mathbf T(\ttt^+ - \ttt^\star, ~\ttt^- - \ttt^\star) \pp$. 
Now consider $\aaa\in\mca(\mce)$ and examine which  $T_{\rr^\star, \aaa}$ is nonzero. For the $\bar\theta_m^+$ in \eqref{eq-con-an}, in Step 1 we have already shown $\theta_m^+ = \bar\theta_m^+$ if $m$ is an ancestor. 
For an allowable attribute pattern $\aaa\in\mca(\mce)$, in order to have 
$$T_{\rr^\star, \aaa}(\ttt^+ - \ttt^\star, ~\ttt^- - \ttt^\star) =  (\theta_{k,\aaa} - \bar\theta_k^-) \prod_{m\text{ is ancestor}} (\theta_{m,\aaa} - \bar\theta_m^+) \neq 0,$$ 
there must be $\theta_{k,\aaa} \neq \theta_m^+$ (i.e., $\theta_{k,\aaa} = \theta_m^-$) for every ancestor attribute $m$. {Since all the attributes are in a connected graph}, such argument implies  $T_{\rr^\star, \aaa} \neq 0$ happens only if $\aaa = \zero_K$. 
For $\aaa = \zero_K$, there is $\theta_{k,\aaa} = \theta_k^-$, so 
$$
0 = \mathbf T_{\rr^\star,\bcolon}(\ttt^+ - \ttt^\star, ~\ttt^- - \ttt^\star) \pp
= p_{\zero_K} (\theta_k^- - \bar \theta_k^-) \prod_{m\text{ is ancestor}} (\theta_m^- - \bar\theta_m^+),
$$
which gives that $\theta_k^- = \bar\theta_k^-$.
Note that thus far we have already shown $\theta_k^- = \bar\theta_k^-$ for all $k=1,\ldots,K$.

\medskip
\noindent\textbf{Step 3.~}
We next consider an intermediate attribute $k$ and prove $\theta_k^+ = \bar \theta_k^+$. Suppose there exists leaf attribute $\alpha_\ell$ such that $k \to \ell$.
Define 
$$
\ttt^\star = \theta_\ell^- \ee_\ell + \bar\theta_k^+ \ee_k,
\quad
\rr^\star = \text{support}(\ttt^\star).
$$
Note that in Step 2 we have shown $\theta_\ell^- = \bar \theta_\ell^-$ for a leaf attribute $\ell$.
So for any $\aaa\in\{0,1\}^K$ if $\aaa\succeq\bq_\ell$ then there must be $\aaa\succeq\bq_k$ so $\bar {\mathbf T}_{\rr^\star, \aaa}(\bar\ttt^+ - \ttt^\star, ~\bar\ttt^- - \ttt^\star) = 0$ for all $\aaa$. 
Therefore 
\begin{align*}
\bar {\mathbf T}_{\rr^\star,\bcolon}(\bar\ttt^+ - \ttt^\star, ~\bar\ttt^- - \ttt^\star) \bar \pp = 0 
= &~ {\mathbf T}(\ttt^+ - \ttt^\star, ~\ttt^- - \ttt^\star)_{\rr^\star,\bcolon} \pp
	\\
= &~ (\theta_\ell^+ - \theta_\ell^-)(\theta_k^+ - \bar\theta_k^+) \Big(\sum_{\aaa:\, \aaa\succeq \bq_\ell} p_{\aaa}\Big),
\end{align*}
which implies $\theta_k^+ = \bar\theta_k^+$.

\medskip
\noindent\textbf{Step 4.~} In this step we consider a leaf attribute and proceed in three separate steps, Step 4.1, 4.2, 4.3, as follows.

\vspace{1mm}
\noindent\textbf{Step 4.1.~}
If the attribute $k$ is a leaf attribute, then there exists $j_k > K$ such that $q^{\text{sparse}}_{j_k,k} = 1$ under Condition D, so $q_{j_k, \ell} = 1$. In this step we first prove $\theta_{j_k}^- = \bar\theta_{j_k}^-$. 
{Note that the sparsified row vector $\bo q^{\text{sparse}}_{j_k}$ can potentially contain multiple entries of ``1'', but $q^{\text{sparse}}_{j_k, \ell} = 1$ happens only if attribute $\ell$ is a leaf attribute (or a singleton attribute, which does not exist under the considered connected-graph hierarchy).} 
The above claim can be deducted from the definitions of attribute types and the sparsifying operation.
In this case we can replace the definition of $\ttt^\star$ in \eqref{eq-con-an} by 
$$
\ttt^\star=
	\bar\theta_{j_k}^-\ee_{j_k} + \sum_{m\text{ is ancestor}} \bar\theta_m^+ \ee_m,
\quad
\rr^\star = \text{support}(\ttt^\star).
$$
%Following similar rationale as the previous proof after \eqref{eq-con-an}, since we are considering a connected graph hierarchy, 
Since attribute $k$ is a leaf, enumerating all the ancestor attributes in the above definition of $\ttt^\star$ must include prerequisites for attribute $k$ and all the other leaf attributes measured by item $j_k$ (as indicated in  $\bo q^{\text{sparse}}_{j_k}$). Mathematically, for any $\ell\in[K]$ such that $q^{\text{sparse}}_{j_k, \ell} = 1$, there exists some ancestor attribute $m$ such that $m\to \ell$. 
Similarly as the argument in Step 2 after \eqref{eq-con-an}, this implies that $\bar {T}_{\rr^\star,\aaa}(\bar\ttt^+ - \ttt^\star, ~\bar\ttt^- - \ttt^\star) = 0$ for any $\aaa\in\{0,1\}^K$ and hence $\bar {\mathbf T}(\bar\ttt^+ - \ttt^\star, ~\bar\ttt^- - \ttt^\star) \bar \pp = 0 = \mathbf T(\ttt^+ - \ttt^\star, ~\ttt^- - \ttt^\star) \pp$. 
Now consider $\aaa\in\mca(\mce)$ and examine which  $T_{\rr^\star, \aaa}(\ttt^+ - \ttt^\star, ~\ttt^- - \ttt^\star)$ can potentially be nonzero.
In order to have $T_{\rr^\star, \aaa}(\ttt^+ - \ttt^\star, ~\ttt^- - \ttt^\star) \neq 0$, there must be $\theta_{j_k, \aaa} \neq \bar \theta_m^+ = \theta_m^+$ for every ancestor attribute $m$. This implies $T_{\rr^\star, \aaa}(\ttt^+ - \ttt^\star, ~\ttt^- - \ttt^\star) \neq 0$ only if $\aaa=\zero_K$ since we are considering a connected-graph hierarchy.
For  $\aaa = \zero_K$, there is $\theta_{j_k,\aaa} = \theta_{j_k}^-$, so 
$$
0 = \mathbf T_{\rr^\star,\bcolon}(\ttt^+ - \ttt^\star, ~\ttt^- - \ttt^\star) \pp
= p_{\zero_K} (\theta_{j_k}^- - \bar \theta_{j_k}^-) \prod_{m\text{ is ancestor}} (\theta_m^- - \bar\theta_m^+),
$$
which gives that $\theta_{j_k}^- = \bar\theta_{j_k}^-$.

\medskip
\noindent\textbf{Step 4.2.~}
Still consider a leaf attribute $k$. For any $j_k > K$ such that $q^{\text{sparse}}_{j_k,k} = 1$, we next show $\theta_{j_k}^+ = \bar\theta_{j_k}^+$.
%According to Condition D, suppose there are two items $k$ and $j_k$ ($j_k >K$) both measuring the leaf attribute $k$. 
Reasoning similarly as in the above Step 4.1, the $\bo q^{\text{sparse}}_{j_k}$ can potentially contain multiple entries of ``1'', but $q^{\text{sparse}}_{j_k, \ell} = 1$ only if attribute $\ell$ is a leaf under the considered connected-graph hierarchy.
Define
$$
\ttt^\star = \sum_{\ell\in[K]:\, q^{\text{sparse}}_{j_k, \ell} = 1} \theta_\ell^- \ee_\ell,
\quad
\rr^\star = \text{support}(\ttt^\star),
%\quad
%\ttt^{\#} = \ttt^\star  + \theta_{j_k}^- \ee_{j_k},
$$
and recall that we have shown $\theta_\ell^- = \bar\theta_\ell^-$ for any leaf attribute $\ell$ in Step 2.
Introduce a notation $\mc K = \{\ell\in[K]:\, q^{\text{sparse}}_{j_k, \ell} = 1\}$, then $k\in \mc K$ and $\mc K$ potentially contain some additional leaf attributes. Recall that $\bq_{\mc K} = \vee_{k\in\mc K}\, \bq_k$, then the element ${T}_{\rr^\star, \aaa} (\ttt^+ - \ttt^\star, ~\ttt^- - \ttt^\star) \neq 0$ only if $\aaa \succeq \bq_{\mc K}$.
Then there are 
\begin{align*}
	\bar{\mathbf T}_{\rr^\star,\bcolon}(\bar\ttt^+ - \ttt^\star, ~\bar\ttt^- - \ttt^\star)\bar\pp 
	= &~ {\mathbf T}_{\rr^\star,\bcolon}(\ttt^+ - \ttt^\star, ~ \ttt^- - \ttt^\star) \pp 
	\\
	= &~ \prod_{\ell \in \mc K} (\theta_\ell^+ - \theta_\ell^-)\Big(\sum_{\aaa:\,\aaa\succeq\bq_{\mc K}} p_{\aaa}\Big) \neq 0
\end{align*}  
%$\bar {\mathbf T}_{\rr^\star +\ee_{j_k},\bcolon}(\bar\ttt^+ - \ttt^{\star},~ \bar\ttt^- - \ttt^{\star})\bar\pp = {\mathbf T}_{\rr^\star +\ee_{j_k}, \bcolon}(\ttt^+ - \ttt^{\star},~ \ttt^- - \ttt^{\star}) \pp$ 
therefore
\begin{align}\label{eq-ejk}
\frac{\bar  {\mathbf T}_{\rr^\star +\ee_{j_k},\bcolon}(\bar\ttt^+ - \ttt^{\star},~ \bar\ttt^- - \ttt^{\star}) \bar\pp}
{\bar{\mathbf T}_{\rr^\star,\bcolon}(\bar\ttt^+ - \ttt^\star, ~\bar\ttt^- - \ttt^\star) \bar\pp}
=
\frac{{\mathbf T}_{\rr^\star + \ee_{j_k},\bcolon} (\ttt^+ - \ttt^{\star},~ \ttt^- - \ttt^{\star}) \pp}
{{\mathbf T}_{\rr^\star, \bcolon} (\ttt^+ - \ttt^\star, ~\ttt^- - \ttt^\star)\pp}.
\end{align}
Since $\theta_\ell^- = \bar\theta_\ell^-$ for all $\ell\in\mc K$, there also is ${\bar T}_{\rr^\star, \aaa} (\bar\ttt^+ - \ttt^\star, ~\bar\ttt^- - \ttt^\star) \neq 0$ only if $\aaa \succeq \bq_{\mc K}$.
Further, for any $\aaa \succeq \bq_{\mc K}$, there is $\theta_{j_k,\aaa} = \theta_{j_k}^+$ and $\bar\theta_{j_k,\aaa} = \bar\theta_{j_k}^+$.
Now we can write the equality in \eqref{eq-ejk} as
\begin{align*}
&~\frac{\bar\theta_{j_k}^+ 
\prod_{\ell \in \mc K}
(\bar\theta_\ell^+ - \theta_\ell^-)\left(\sum_{\aaa: \,\aaa\succeq\bq_{\mc K}}\bar p_{\aaa}\right)}
     {\prod_{\ell \in \mc K}  (\bar\theta_\ell^+ -  \theta_\ell^-)\left(\sum_{\aaa:\,\aaa\succeq\bq_{\mc K}}\bar p_{\aaa}\right)}
\\
=
&~
\frac{\theta_{j_k}^+
\prod_{\ell\in \mc K}
(\theta_\ell^+ - \theta_\ell^-)\left(\sum_{\aaa:\,\aaa\succeq\bq_{\mc K}} p_{\aaa}\right)}{\prod_{\ell \in \mc K} (\theta_\ell^+ - \theta_\ell^-)\left(\sum_{\aaa:\,\aaa\succeq\bq_{\mc K}} p_{\aaa}\right)},
\end{align*}
%Note that we have proved $\bar\theta^-_{j_k} = \theta^-_{j_k}$ for the leaf attribute $k$ in the previous step, so the above display 
which implies $\bar\theta_{j_k}^+ = \theta_{j_k}^+$.

\medskip
\noindent\textbf{Step 4.3.~}
Still consider a leaf attribute $k$ and we next prove $\bar\theta^+_{k} = \theta^+_{k}$.
Define
$$
\ttt^\star = 
    \theta^-_{j_k}\ee_{j_k},
$$
Recall that we have shown $\theta^-_{j_k} = \bar \theta^-_{j_k}$ in Step 4.1.
Therefore, 
\begin{equation*}
    \frac{ \bar  T_{\ee_{j_k}+\ee_k,\bcolon}(\bar\ttt^+ - \ttt^\star, ~\bar\ttt^- - \ttt^\star)  \bar\pp}{ \bar  T_{\ee_{j_k},\bcolon}(\bar\ttt^+ - \ttt^\star, ~\bar\ttt^- - \ttt^\star)  \bar\pp}=
    \frac{T_{\ee_{j_k}+\ee_k,\bcolon}(\ttt^+ - \ttt^\star, ~\ttt^- - \ttt^\star) \pp}{T_{\ee_{j_k},\bcolon} (\ttt^+ - \ttt^\star, ~\ttt^- - \ttt^\star) \pp},
\end{equation*}
which further gives
\begin{equation*}
    \frac{\bar\theta_{k}^+ \left(\bar\theta^+_{j_k} - \theta^-_{j_k}\right) \left(\sum_{\aaa:\, \aaa\succeq\bq_{j_k}} \bar p_{\aaa}\right)}
    {\left(\bar\theta^+_{j_k} - \theta^-_{j_k}\right) \left(\sum_{\aaa:\, \aaa\succeq\bq_{j_k}}\bar p_{\aaa}\right)}
    =
    \frac{\theta_{k}^+  \left(\bar\theta^+_{j_k} - \theta^-_{j_k}\right)  \left(\sum_{\aaa:\, \aaa\succeq\bq_{j_k}} p_{\aaa}\right)}
    { \left(\bar\theta^+_{j_k} - \theta^-_{j_k}\right) \left(\sum_{\aaa:\, \aaa\succeq\bq_{j_k}} p_{\aaa}\right)}.
    %\bar \theta^+_k = =\theta^+_k.
\end{equation*}
Since in Step 4.2 we have shown $\bar\theta^+_{j_k} = \theta^+_{j_k}$, the above display implies $\bar\theta_{k}^+ = \theta_{k}^+$.

\medskip
\noindent\textbf{Step 5.~}
Our proof has already shown that if assuming $\QQ_{1:K,\bcolon} = \EE$, then all the item parameters associated with the first $K$ items are identifiable.
Now consider an arbitrary item $j > K$.
For notational simplicity, we next write the row vector $\bar {\mathbf T}_{\rr^\star,\bcolon}(\bar\ttt^+ - \ttt^\star, ~\bar\ttt^- - \ttt^\star)$ simply as $\bar {\mathbf T}_{\rr^\star,\bcolon}$, and ${\mathbf T}_{\rr^\star,\bcolon}(\ttt^+ - \ttt^\star, ~\ttt^- - \ttt^\star)$ simply as ${\mathbf T}_{\rr^\star,\bcolon}$.
Define $\ttt^\star = \sum_{k=1}^K \theta_k^+ \ee_k$ and $\rr^\star = \sum_{k=1}^K \ee_k$.
Note that $\bar{\mathbf T}_{\rr^\star,\bcolon}\bar \pp = {\mathbf T}_{\rr^\star,\bcolon} \pp = \prod_{k=1}^K (\theta_k^- - \theta_k^+) p_{\zero_K} \neq 0$, so there is 
\begin{align*}
	\frac{ \bar\theta_j^- \prod_{k=1}^K (\theta_k^- - \theta_k^+) \bar p_{\zero_K}}
	     {\prod_{k=1}^K (\theta_k^- - \theta_k^+) \bar p_{\zero_K}}
	     =
	 \frac{\bar{\mathbf T}_{\rr^\star + \ee_j,\bcolon}\bar \pp}{\bar{\mathbf T}_{\rr^\star,\bcolon}\bar \pp}
     =
     \frac{{\mathbf T}_{\rr^\star + \ee_j,\bcolon}\bar \pp}{{\mathbf T}_{\rr^\star,\bcolon}\bar \pp}
     =
     \frac{ \theta_j^- \prod_{k=1}^K (\theta_k^- - \theta_k^+) p_{\zero_K}}
	     {\prod_{k=1}^K (\theta_k^- - \theta_k^+) p_{\zero_K}},
\end{align*}
which gives $\bar\theta_j^- = \theta_j^-$.  
Similarly define $\ttt^{'} = \sum_{k=1}^K \theta_k^- \ee_k$ and $\rr^{'} = \sum_{k=1}^K \ee_k$.
Since $\bar{\mathbf T}_{\rr',\bcolon}\bar \pp = {\mathbf T}_{\rr',\bcolon} \pp = \prod_{k=1}^K (\theta_k^+ - \theta_k^-) p_{\one_K} \neq 0$, there is
\begin{align*}
	\frac{ \bar\theta_j^+ \prod_{k=1}^K (\theta_k^+ - \theta_k^-) \bar p_{\one_K}}
	     {\prod_{k=1}^K (\theta_k^+ - \theta_k^-) \bar p_{\one_K}}
	     =
	 \frac{\bar{\mathbf T}_{\rr^{'} + \ee_j,\bcolon}\bar \pp}{\bar{\mathbf T}_{\rr^{'},\bcolon}\bar \pp}
     =
     \frac{{\mathbf T}_{\rr^{'} + \ee_j,\bcolon}\bar \pp}{{\mathbf T}_{\rr^{'},\bcolon}\bar \pp}
     =
     \frac{ \theta_j^+ \prod_{k=1}^K (\theta_k^+ - \theta_k^-) p_{\one_K}}
	     {\prod_{k=1}^K (\theta_k^+ - \theta_k^-) p_{\one_K}},
\end{align*}
which gives $\bar\theta_j^+ = \theta_j^+$. Thus far we have shown that all the item parameters are identifiable.
%$m\in[K]$ for which $k^\star \to m$.

%%
\medskip
\noindent\textbf{Step 6.~}
We introduce the following useful lemma.

\begin{lemma}\label{lem-prop}
	Consider a two-parameter HLAM under a fixed $\QQ$-matrix with $\QQ_{1:K,\bcolon} =\EE$.
	If item parameters $\ttt^+$ and $\ttt^-$ are fixed and known, then the proportion parameters $\pp$ and the attribute hierarchy $\mce$ are identifiable.
\end{lemma}

\begin{proof}[Proof of Lemma \ref{lem-prop}]

Next we show $\bar p_{\aaa} = p_{\aaa}$ for all $\aaa\in\mathcal A(\mathcal E)$. 
% This would also naturally establish the identifiability of the attribute hierarchy $\mathcal E$. 
First, define $\ttt^\star = \sum_{k=1}^K \theta^+_k\ee_k$ and $\rr^\star = \sum_{k=1}^K \ee_k$, then  $\bar {\mathbf T}_{\rr^\star,\bcolon}(\bar\ttt^+ - \ttt^\star, ~\bar\ttt^- - \ttt^\star) \bar \pp = \mathbf {\mathbf T}_{\rr^\star,\bcolon}(\ttt^+ - \ttt^\star, ~\ttt^- - \ttt^\star) \pp$ gives 
	$$
	\bar p_{\zero_K}\prod_{k=1}^K (\theta^-_k - \theta^+_k) =  p_{\zero_K}\prod_{k=1}^K (\theta^-_k - \theta^+_k),
	 %\quad\Longrightarrow\quad	\bar p_{\zero_K} = p_{\zero_K}.
	$$
	which implies $\bar p_{\zero_K} = p_{\zero_K}$.
Second, for any $\aaa\in\mca(\mce)$ we have $p_{\aaa}>0$. Define
$$
\ttt^\star = \sum_{1\leq k\leq K:\, \alpha_k=1} \theta^-_k\ee_k
+\sum_{1\leq k\leq K:\, \alpha_k=0} \theta^+_k\ee_k,
$$
and let $\rr^\star$ denote the support vector of $\ttt^\star$.
Then $\bar {\mathbf T}_{\rr^\star,\bcolon}(\bar\ttt^+ - \ttt^\star, ~\bar\ttt^- - \ttt^\star) \bar \pp = \mathbf {\mathbf T}_{\rr^\star,\bcolon}(\ttt^+ - \ttt^\star, ~\ttt^- - \ttt^\star) \pp$ gives
\begin{align*}
&~\bar p_{\aaa} \prod_{k\geq 1,\, \alpha_k=1}(\theta^+_k-\theta^-_k)\prod_{k\geq 1,\, \alpha_k=0}(\theta^-_k-\theta^+_k)\\
=&~
p_{\aaa} \prod_{k\geq 1,\, \alpha_k=1}(\theta^+_k-\theta^-_k)\prod_{k\geq 1,\, \alpha_k=0}(\theta^-_k-\theta^+_k),	
\end{align*}
which gives $\bar p_{\aaa} = p_{\aaa}$ for all $\aaa\in\mca(\mce)$. Thus far we established $\bar p_{\aaa} = p_{\aaa}$ for all $\aaa\in\mca(\mce)$. This implies $\sum_{\aaa\in\mca(\mce)} \bar p_{\aaa} = \sum_{\aaa\in\mca(\mce)} p_{\aaa} = 1$. Now from $\sum_{\aaa\in\{0,1\}^K} \bar p_{\aaa} = 1$ and $\bar p_{\aaa} \geq 0$ for any $\aaa\in\{0,1\}^K$, we obtain that $\bar p_{\aaa} =0$ for any $\aaa\in\{0,1\}^K\setminus \mca(\mce)$.
Now we have proved $\bar\pp = \pp$. This establishes the identifiability of  $\mce$ and $\pp$ and completes the proof of the lemma.
\end{proof}

Since the previous steps have already shown the identifiability of the item parameters $\ttt^+$ and $\ttt^-$, the conclusion of Lemma \ref{lem-prop} directly applies and we obtain the identifiability of $\pp$ and $\mce$. The proof of Corollary \ref{thm-connected} is now complete.
\end{proof}

%\color{blue!70!black}
\begin{proof}[Proof of Theorem \ref{thm-general}]
It suffices to prove the sufficiency of the conditions for identifiability, because the necessity of the conditions has already established in the previous propositions; in particular, the necessity of Condition A was proved in   Proposition \ref{prop-Q}, the necessity of B$^{\,\star}$ was proved in Propositions \ref{prop-nece-sing}-\ref{prop-nece-int}, and the necessity of  C$^{\,\star}$ was proved in Proposition \ref{thm-nec-dist}.

We prove the sufficiency of Conditions A, B$^\star$, and C$^\star$ in several steps. 
First denote the set of ancestor, intermediate, leaf, and singleton attributes by 
\begin{align}\label{eq-ancestor}
	\mc K^{\ancestor} &= \{k\in[K]:\, \alpha_k\text{ is an ancestor attribute}\};
	\\ \notag
	\mc K^{\interm} &= \{k\in[K]:\, \alpha_k\text{ is an intermediate attribute}\};
	\\ \notag
	\mc K^{\leaf} &= \{k\in[K]:\, \alpha_k\text{ is a leaf attribute}\};
	\\ \notag
	\mc K^{\single} &= \{k\in[K]:\, \alpha_k\text{ is a singleton attribute}\}.
\end{align}
Recall that we assume without loss of generality that $\QQ_{1:K,\bcolon} = \EE$ under Condition A.
In the following, we abuse notation a little by referring to each set above as both a set of latent attributes and also a set of integers potentially indexing the items. 
In the following, we prove the theorem in nine  steps, with the roadmap given in Table \ref{tab-proof}.

\begin{table}[h!]
\caption{{Roadmap of the proof of Theorem \ref{thm-general}, establishing the identifiability of item parameters $\ttt^+$ and $\ttt^-$. The identifiability of the proportion parameters $\pp$ and the attribute hierarchy graph $\mce$ are proved in the final Step 9.}}
   \label{tab-proof}
   \centering
  %\color{blue!70!black}
  \begin{tabular}{cccccccc}
    \toprule
    \multirow{2}{*}{Item parameters} & \multicolumn{4}{c}{For $1\leq j\leq K$ and the $j$th attribute is:} & &  \multirow{2}{*}{For $K+1 \leq j \leq J$:} & 
    \\
    \cmidrule(lr){2-5}
     & ancestor & intermediate & leaf & singleton &  & &
    \\
    \midrule
    $\theta_{j}^+$  & Step 1  & Step 3 & Step 5 & Step 7 & &  Step 4 &  \\[1mm]
    $\theta_{j}^-$  & Step 1  & Step 2 & Step 2 & Step 6 & &  Step 8 &  \\
    \bottomrule
   \end{tabular}
   
%   \captionsetup{format=plain}
\end{table}

\medskip
\noindent\textbf{Step 1.~}
First consider an ancestor attribute $k$.
This step proceeds in the same way as Step 1 in the proof of Corollary \ref{thm-connected} under Condition B$^\star$, so we omit the details. Just recall that Condition B$^\star$ implies there are two row vectors in $\mathcal S^{\mce}(\QQ)$ indexed by $k$ and $j_k$ (where $j_k >K$) that both measure the attribute $k$; that is, $\bq_k = \bq_{j_k} = \mathcal D^{\mce}(\ee_k) = \ee_k$ since $k$ is an ancestor.
Then the same argument as Step 1 in the proof of Corollary \ref{thm-connected} gives that for any ancestor attribute $k$, the item parameters associated with items $k$ and $j_k$ are identifiable.

\medskip
\noindent\textbf{Step 2.~}
Consider a leaf or an intermediate attribute $k$ and we next show $\theta_k^- = \bar\theta_k^-$.
This step is a modification of Step 2 in the proof of Corollary \ref{thm-connected}.
Define
\begin{align}\label{eq-con-an2}
	\ttt^\star=
	\bar\theta_k^-\ee_k + \sum_{m\in \mc K^{\ancestor}} \bar\theta_m^+ \ee_m,\quad
	\rr^\star = \text{support}(\ttt^\star).
\end{align}
Since attribute $k$ is a leaf or an intermediate, 
there must exist an ancestor attribute $m$ such that $m\to k$. 
Therefore for any $\aaa\in\{0,1\}^K$, if $\aaa\succeq \bar\bq_k$ (so $\bar\theta_{k,\aaa} = \bar \theta^+_k$) there must be $\aaa\succeq \bar\bq_m$ (so $\bar\theta_{m,\aaa} = \bar\theta^+_m$) for some ancestor attribute $m$. 
Since ${T}_{\rr^\star,\aaa}(\bar\ttt^+ - \ttt^\star, ~\bar\ttt^- - \ttt^\star)$ contains a factor of $(\bar\theta_{k,\aaa} - \bar \theta^+_k) (\bar\theta_{m,\aaa} - \bar\theta^+_m)$, the above argument indeed proves that ${T}_{\rr^\star,\aaa}(\bar\ttt^+ - \ttt^\star, ~\bar\ttt^- - \ttt^\star) = 0$ for any $\aaa\in\{0,1\}^K$ and hence ${\mathbf T}_{\rr^\star, \bcolon}(\bar\ttt^+ - \ttt^\star, ~\bar\ttt^- - \ttt^\star) \bar \pp = 0 = \mathbf T_{\rr^\star, \bcolon}(\ttt^+ - \ttt^\star, ~\ttt^- - \ttt^\star) \pp$. 
Now consider $\aaa\in\mca(\mce)$ and examine which  $T_{\rr^\star, \aaa}$ is nonzero. For the $\bar\theta_m^+$ in \eqref{eq-con-an}, in Step 1 we have already shown $\theta_m^+ = \bar\theta_m^+$ if $m$ is an ancestor. 
For an allowable attribute pattern $\aaa\in\mca(\mce)$, in order to have 
\begin{align}\label{eq-sinleaf}
	T_{\rr^\star, \aaa}:=
	T_{\rr^\star, \aaa}(\ttt^+ - \ttt^\star, ~\ttt^- - \ttt^\star) =  (\theta_{k,\aaa} - \bar\theta_k^-) \prod_{m\in \mc K^{\ancestor}} (\theta_{m,\aaa} - \bar\theta_m^+) \neq 0,
\end{align}
there need to be $\theta_{k,\aaa} \neq \theta_m^+$ (that is, $\theta_{k,\aaa} = \theta_m^-$) for every ancestor attribute $m$;
%\darkred{Since all the attributes are in a connected graph}, such argument implies  $T_{\rr^\star, \aaa} \neq 0$ happens only if $\aaa = \zero_K$. 
for such attribute pattern $\aaa$ with $\aaa\nsucceq\bq_m$ for all $m\in \mc K^{\ancestor}$, there must also be $\aaa\nsucceq\bq_k$ and hence $\theta_{k,\aaa} = \theta_k^-$. Therefore for such $\aaa$,  the $T_{\rr^\star, \aaa}$ in \eqref{eq-sinleaf} equals $(\theta_k^- - \bar \theta_k^-) \prod_{m \in \mc K^{\ancestor}} (\theta_m^- - \bar\theta_m^+)$.
Therefore we have
\begin{align*}
0 = &~ \mathbf T_{\rr^\star,\bcolon}(\ttt^+ - \ttt^\star, ~\ttt^- - \ttt^\star) \pp\\
= &~ \Big(\sum_{\aaa\in\mca^\star} p_{\aaa}\Big) (\theta_k^- - \bar \theta_k^-) \prod_{m \in \mc K^{\ancestor}} (\theta_m^- - \bar\theta_m^+),
\end{align*}
where
$
\mca^\star=\{\aaa\in\mca(\mce):\,
\aaa\nsucceq\bq_m~\forall~m\in \mc K^{\ancestor}\}.
$
The above equality gives that $\theta_k^- = \bar\theta_k^-$ for any $k\in \mc K^{\interm} \cup \mc K^{\leaf}$.

%%%%%%%%%%%%%%%%%%%%%%%%%
\medskip
\noindent\textbf{Step 3.~}
We next consider an intermediate attribute $k$ and prove $\theta_k^+ = \bar \theta_k^+$.
To show this we can use a similar argument as Step 3 in the proof of Corollary \ref{thm-connected}.
In short, consider a leaf attribute $\alpha_\ell$ such that $k \to \ell$, and define 
$\ttt^\star = \theta_\ell^- \ee_\ell + \bar\theta_k^+ \ee_k$ and
$\rr^\star = \text{support}(\ttt^\star).$
Since in Step 2 we have shown $\theta_\ell^- = \bar \theta_\ell^-$ for any leaf attribute $\ell$,
%So for any $\aaa\in\{0,1\}^K$ if $\aaa\succeq\bq_\ell$ then there must be $\aaa\succeq\bq_k$ so $\bar {\mathbf T}_{\rr^\star, \aaa}(\bar\ttt^+ - \ttt^\star, ~\bar\ttt^- - \ttt^\star) = 0$ for all $\aaa$. 
%Therefore 
%\begin{align*}
%\bar {\mathbf T}_{\rr^\star,\bcolon}(\bar\ttt^+ - \ttt^\star, ~\bar\ttt^- - \ttt^\star) \bar \pp = 0 
%= &~ {\mathbf T}(\ttt^+ - \ttt^\star, ~\ttt^- - \ttt^\star)_{\rr^\star,\bcolon} \pp
%	\\
%= &~ (\theta_\ell^+ - \theta_\ell^-)(\theta_k^+ - \bar\theta_k^+) \Big(\sum_{\aaa:\, \aaa\succeq \bq_\ell} p_{\aaa}\Big),
%\end{align*}
then we can obtain $\theta_k^+ = \bar\theta_k^+$ for any $k\in \mc K^{\interm}$.

%%%%%%%%%%%%%%%%%%%%%%%%%
\medskip
\noindent\textbf{Step 4.~}
In this step we prove $\theta_{j}^+ = \bar\theta_{j}^+$ for any $j\in\{K+1,\ldots, J\}$. First note that $q^{\sparse}_{j,k} = 1$ only if $k$ is a leaf attribute or a single attribute. 
%We next treat the following two cases (a) and (b) separately in Step 4.1 and Step 4.2, (a) there exists some leaf attribute $k$ such that $q^{\sparse}_{j, k} = 1$; and (b) $q^{\sparse}_{j, k} = 1$ only for those $k \in \mc K^{\single}$.
%
%%%%%%%%%%%%%%%%%%%%%%%%%%
%\medskip
%\noindent\textbf{Step 4.1.~}
%Consider case (a), $q^{\sparse}_{j, k} = 1$ only for those $k \in \mc K^{\single}$.
Define 
\begin{align}\label{eq-s42}
	\ttt^\star = \sum_{k=1}^K \bar\theta^-_k \ee_k + \sum_{h>K:\, h\neq j} \theta^-_h \ee_h,\quad
	\rr^\star = \support(\ttt^\star).
\end{align}
Under $\QQ_{1:K,\bcolon} = \EE$, the element $ T_{\rr^\star,\aaa}(\bar\ttt^+ - \ttt^\star, \bar\ttt^- - \ttt^\star)$ is potentially nonzero only if $\aaa = \one_K$ because of the first summation term in \eqref{eq-s42}.
Next consider $T_{\rr^\star,\aaa}(\ttt^+ - \ttt^\star, \ttt^- - \ttt^\star)$.
We now claim that the following quantity is potentially nonzero also only if $\aaa = \one_K$,
\begin{equation}
	\label{eq-s42t}
	T_{\rr^\star,\aaa} := T_{\rr^\star,\aaa}(\ttt^+ - \ttt^\star, \ttt^- - \ttt^\star) = 
	\prod_{k=1}^K (\theta_{k, \aaa} - \bar\theta^-_k) \prod_{h>K:\, h\neq j} (\theta_{h, \aaa} - \theta^-_h);
\end{equation}
 we next prove this claim.
\textit{First}, if an attribute pattern $\aaa$ lacks any leaf attribute $\ell \in \mc K^{\leaf}$, then $\theta_{\ell, \aaa} = \theta^-_\ell = \bar \theta^-_\ell$ (thanks to the conclusion obtained in Step 2). This results in $T_{\rr^\star,\aaa}$ defined in \eqref{eq-s42t} containing a factor of $(\theta^-_\ell - \theta^-_\ell)$ and hence it must be zero.
\textit{Second}, consider an attribute pattern $\aaa$ that lacks any singleton attribute $k \in \mc K^{\single}$.
Since in the current Step 4.2 we consider the case $q^{\sparse}_{j, k} = 1$ only for those $k \in \mc K^{\single}$. 
For any $k \in \mc K^{\single}$, Condition B$^\star$ states there are $\geq 3$ items measuring $k$ in the sparsified matrix $\mc S^{\mce}(\QQ)$. 
Since we assume $\QQ_{1:K, \bcolon} = \EE$ under Condition A, Condition B$^\star$ essentially means there are $\geq 2$ items measuring attribute $k$ in the sparsified submatrix $\mc S^{\mce}(\QQ_{(K+1):J, \bcolon})$.
Therefore other than the currently considered item $j$, there must exist some other item $j_k > K$, $j_k \neq j$ such that $q^{\sparse}_{j_k, k} = 1$.
%Therefore, for each attribute $k$ with $q^{\sparse}_{j, k} = 1$, there exists such an item $j_k \in \{K+1,\ldots,J\}\setminus \{j\}$ with $q^{\sparse}_{j_k, k} = 1$. 
For this $j_k$, since we are considering an attribute pattern $\aaa$ that lacks the attribute $k$, there must be $\aaa\nsucceq \bq_{j_k}$ and hence $\theta_{j_k, \aaa} = \theta^-_{j_k}$.
The above reasoning results in $T_{\rr^\star,\aaa}$ in \eqref{eq-s42t} containing a factor of $(\theta^-_{j_k} - \theta^-_{j_k})$ and hence it must be zero.
To summarize, we have shown that the $T_{\rr^\star,\aaa}$ defined in \eqref{eq-s42t} equals zero unless the attribute pattern $\aaa$ possesses all the leaf attributes and all the singleton attributes. Since the $T_{\rr^\star,\aaa}$ is meaningful only for attribute patterns respecting the attribute hierarchy $\aaa\in\mca(\mce)$, any such $\aaa$ possessing all the leaf and singleton attributes must be the all-one latent pattern $\aaa=\one_K$.
Thus far we have proved the earlier claim that the $T_{\rr^\star,\aaa}$ defined in \eqref{eq-s42t} is potentially nonzero only if $\aaa = \one_K$.

Further, for $\aaa = \one_K$ the element $T_{\rr^\star,\one_K} = \prod_{k=1}^K (\theta^+_k - \bar\theta^-_k) \prod_{h>K:\, h\neq j} (\theta^+_h - \theta^-_h)$ is indeed nonzero because each factor of it is nonzero.
Combined with the observation in the beginning of this Step 4.2, we have the following equality for the $\ttt^\star$ defined in \eqref{eq-s42}, 
\begin{align}\label{eq-s421}
    &~\prod_{k=1}^K (\bar\theta^+_k - \bar\theta^-_k) \prod_{h>K:\, h\neq j} (\bar\theta^+_h - \theta^-_h) \bar p_{\one_K}
    \\ \notag
	&~ 
	T_{\rr^\star,\one_K}(\bar\ttt^+ - \ttt^\star, \bar\ttt^- - \ttt^\star) \bar p_{\one_K}
	=T_{\rr^\star,\bcolon}(\bar\ttt^+ - \ttt^\star, \bar\ttt^- - \ttt^\star) \bar \pp 
	\\ \notag
	=&~ T_{\rr^\star,\bcolon}(\ttt^+ - \ttt^\star, \ttt^- - \ttt^\star)  \pp
	=
	T_{\rr^\star,\one_K}(\ttt^+ - \ttt^\star, \ttt^- - \ttt^\star) p_{\one_K} 
	\\ \label{eq-s422}
	=&~ \prod_{k=1}^K (\theta^+_k - \bar\theta^-_k) \prod_{h>K:\, h\neq j} (\theta^+_h - \theta^-_h) p_{\one_K}
	\neq 0.
\end{align}
Based on the above equality, we further consider the item $j$, then there is $\theta_{j, \one_K} = \theta^+_j$ and $\bar\theta_{j, \one_K} = \bar\theta^+_j$.
Therefore we can obtain
\begin{align*}
	\frac{T_{\rr^\star + \ee_j, \bcolon}(\bar\ttt^+ - \ttt^\star, \bar\ttt^- - \ttt^\star) \bar \pp}
	{T_{\rr^\star, \bcolon}(\bar\ttt^+ - \ttt^\star, \bar\ttt^- - \ttt^\star) \bar \pp}
	&=
	\frac{T_{\rr^\star + \ee_j, \bcolon}(\ttt^+ - \ttt^\star, \ttt^- - \ttt^\star) \bar \pp}
	{T_{\rr^\star, \bcolon}(\ttt^+ - \ttt^\star, \ttt^- - \ttt^\star) \bar \pp},
\end{align*}
which can be further written as
\begin{align*}
	\frac{\bar\theta^+_{j} \cdot \text{Eq.~}\eqref{eq-s421}}
	{\text{Eq.~}\eqref{eq-s421}}
	=
	\frac{\theta^+_{j} \cdot \text{Eq.~}\eqref{eq-s422}}
	{\text{Eq.~}\eqref{eq-s422}}.
\end{align*}
This gives $\bar\theta^+_{j} = \theta^+_{j}$ and completes the proof of Step 4.

%%%%%%%%%%%%%%%%%%%%%%%%%
\medskip
\noindent\textbf{Step 5.~}
In this step we aim to show $\theta_k^+ = \bar\theta_k^+$ for any leaf attribute $k \in \mc K^{\leaf}$.
To prove this conclusion, we first prove an intermediate result in Step 5.1 and then prove $\theta_k^+ = \bar\theta_k^+$ for   $k\in \mc K^{\leaf}$ in Step 5.2.

%%%%%%%%%%%%%%%%%%%%%%%%%
\vspace{1mm}
\noindent\textbf{Step 5.1.~}
Consider any item $j\in\{K+1,\ldots,J\}$ for which there exists some leaf attribute $k$ such that $q^{\sparse}_{j, k} = 1$. In this step we prove $\theta_j^- = \bar \theta_j^-$.
Define 
\begin{align*}
	\ttt^\star = 
	\bar\theta_{j}^- \ee_{j} + \sum_{m\in \mc K^{\ancestor}} \theta^+_m \ee_m,
	\quad \rr^\star = \support(\ttt^\star).
\end{align*}
Recall that we have shown $\theta^+_m = \bar\theta^+_m$ for any $m\in \mc K^{\ancestor}$ in Step 1.
For any $\aaa\in\{0,1\}^K$, the element $T_{\rr^\star,\aaa}(\bar\ttt^+ - \ttt^\star, \bar\ttt^- - \ttt^\star) = (\bar\theta_{j,\aaa} - \bar\theta^-_j) \prod_{m\in \mc K^{\ancestor}} (\bar\theta_{m,\aaa} - \theta^+_m)$ is potentially nonzero only if $\bar\theta_{m,\aaa} = \theta^-_m$ for all $m \in \mc K^{\ancestor}$, which holds only if $\aaa \nsucceq \bq_m$ for all $m \in \mc K^{\ancestor}$. Such an attribute pattern $\aaa$ must satisfy $\aaa \nsucceq \bq_j$ because of the following facts: $q^{\sparse}_{j,k} = 1$ holds for some leaf attribute $k$ so there is $m \to k$ for some $m\in\mc K^{\ancestor}$ and $\bq_k \succeq \bq_m$.
The above reasoning implies $\bar\theta_{j,\aaa} = \bar\theta_j^-$, therefore $T_{\rr^\star,\aaa}(\bar\ttt^+ - \ttt^\star, \bar\ttt^- - \ttt^\star) = 0$ always holds.
This means $T_{\rr^\star, \bcolon}(\bar\ttt^+ - \ttt^\star, \bar\ttt^- - \ttt^\star) \bar\pp  = 0 = T_{\rr^\star, \bcolon}(\ttt^+ - \ttt^\star, \ttt^- - \ttt^\star) \pp$.
Now consider the element $T_{\rr^\star, \aaa} := T_{\rr^\star, \aaa}(\ttt^+ - \ttt^\star, \ttt^- - \ttt^\star)$, by a similar argument there is $T_{\rr^\star, \aaa}$ is potentially nonzero only if $\aaa \nsucceq \bq_m$ for all $m \in \mc K^{\ancestor}$ and for such $\aaa$ there is also $\theta_{j,\aaa} = \theta_j^-$. Therefore
\begin{align*}
	0 = 
	T_{\rr^\star, \bcolon}(\ttt^+ - \ttt^\star, \ttt^- - \ttt^\star) \pp
	 = \Big( \sum_{\aaa\in\mca^\star} p_{\aaa} \Big)
	   (\theta^-_{j} - \bar\theta^-_j) \prod_{m\in \mc K^{\ancestor}} (\theta^-_{m} - \theta^+_m),
\end{align*}
where $\mca^\star = \{\aaa\in\mca(\mce):\, \aaa\nsucceq \bq_m$ for all $m \in \mc K^{\ancestor}\}$.
The above display gives $\theta^-_{j} = \bar\theta^-_j$, and this holds for any item $j > K$ that measures a leaf attribute in the sparsified matrix $\mc S^{\mce}(\QQ)$.

\medskip
\noindent\textbf{Step 5.2.~} We next show $\theta_k^+ = \bar\theta_k^+$ for   $k\in \mc K^{\leaf}$. For a leaf attribute $k$, by Condition B$^\star$ and with $\QQ_{1:K,\bcolon} = \EE$, there exists some item $j_k > K$ such that $q^{\sparse}_{j_k, k} = 1$. 
For this $j_k$ we define
$
\ttt^\star = 
    \theta^-_{j_k}\ee_{j_k},
$
then there is $T_{\ee_{j_k},\bcolon} (\ttt^+ - \ttt^\star, ~\ttt^- - \ttt^\star) \pp = (\theta_{j_k}^+ - \theta_{j_k}^-)\left(\sum_{\aaa:\, \aaa\succeq\bq_{j_k}} p_{\aaa}\right) \neq 0$. 
So there is also $T_{\ee_{j_k},\bcolon} (\bar\ttt^+ - \ttt^\star, ~\bar\ttt^- - \ttt^\star) \bar\pp = T_{\ee_{j_k},\bcolon} (\ttt^+ - \ttt^\star, ~\ttt^- - \ttt^\star) \pp \neq 0$.
Therefore, 
\begin{equation*}
    \frac{ \bar  T_{\ee_{j_k}+\ee_k,\bcolon}(\bar\ttt^+ - \ttt^\star, ~\bar\ttt^- - \ttt^\star)  \bar\pp}{ \bar  T_{\ee_{j_k},\bcolon}(\bar\ttt^+ - \ttt^\star, ~\bar\ttt^- - \ttt^\star)  \bar\pp}=
    \frac{T_{\ee_{j_k}+\ee_k,\bcolon}(\ttt^+ - \ttt^\star, ~\ttt^- - \ttt^\star) \pp}{T_{\ee_{j_k},\bcolon} (\ttt^+ - \ttt^\star, ~\ttt^- - \ttt^\star) \pp}.
\end{equation*}
Now note that the conclusion of Step 5.1 applies to this item $j_k$ and hence $\theta^-_{j_k} = \bar \theta^-_{j_k}$, so the above equality can be written as
\begin{equation*}
    \frac{\bar\theta_{k}^+ \left(\bar\theta^+_{j_k} - \theta^-_{j_k}\right) \left(\sum_{\aaa:\, \aaa\succeq\bq_{j_k}} \bar p_{\aaa}\right)}
    {\left(\bar\theta^+_{j_k} - \theta^-_{j_k}\right) \left(\sum_{\aaa:\, \aaa\succeq\bq_{j_k}}\bar p_{\aaa}\right)}
    =
    \frac{\theta_{k}^+  \left(\bar\theta^+_{j_k} - \theta^-_{j_k}\right)  \left(\sum_{\aaa:\, \aaa\succeq\bq_{j_k}} p_{\aaa}\right)}
    { \left(\bar\theta^+_{j_k} - \theta^-_{j_k}\right) \left(\sum_{\aaa:\, \aaa\succeq\bq_{j_k}} p_{\aaa}\right)}.
\end{equation*}
Further, recall that in Step 4 we have shown $\bar\theta^+_{j} = \theta^+_{j}$ for any $j>K$, so $\bar\theta^+_{j_k} = \theta^+_{j_k}$ and the above display gives $\bar\theta_{k}^+ = \theta_{k}^+$.
This completes the proof of Step 5.

%%%%%%%%%%%%%%%%%%%%%%%%%
\medskip
\noindent\textbf{Step 6.~}
In this step we prove $\theta^-_k = \bar \theta^-_k$ for any singleton attribute $k \in \mc K^{\single}$.
Denote the number of singleton attributes by $K_s := |\mc K^{\single}|$.
Without loss of generality, we can assume that the first $K_s$ attributes $1,\ldots,K_s$ are singletons and the remaining attributes $K_s+1, \ldots, K$ are not singletons.
% Assuming in such a way indeed causes no loss of generality, because permuting any singleton attribute and any other attribute will not change the requirement that $k\to \ell$ only if $k < \ell$.
Under Condition C$^\star$, the following $K_s$ column vectors corresponding to those singleton attributes are distinct binary vectors: $\QQ_{(K+1):J, \, 1}, \QQ_{(K+1):J, \, 2}, \ldots, \QQ_{(K+1):J, \, K_s}$. 
%We need to introduce a useful concept of the \textit{lexicographic order} between two binary vectors of the same length. For two binary vectors $\bo a=(a_1,\ldots,a_L)^\top$ and $\bo b=(b_1,\ldots,b_L)^\top$ both of length $L$, $\bo a$ is said to be of smaller lexicographic order than $\bo b$ (denoted by $\bo a\prec_{\text{lex}}\bo b$), if either $a_1<b_1$, or there exists a integer $l\in\{2,\ldots,L\}$ such that $a_l<b_l$ and $a_m=b_m$ for all $m=1,\ldots,l-1$. 
Recall the definition of the \textit{lexicographic order} between two binary vectors of the same length introduced in the proof of Theorem \ref{thm-main}.
It is not hard to see that under Condition C$^\star$, the $K_s$ distinct column vectors $\QQ_{(K+1):J, \, 1}$, $\QQ_{(K+1):J, \, 2}$, $\ldots$, $\QQ_{(K+1):J, \, K_s}$ can be arranged in an increasing lexicographic order.
Namely, there exists a permutation map $\sigma(\cdot):[K_s]\to[K_s]$ such that 
\begin{equation}\label{eq-qlex-gen}
\QQ_{(K+1):J, \,\sigma(1)}
\prec_{\text{lex}}
\QQ_{(K+1):J, \,\sigma(2)}
\prec_{\text{lex}}
\cdots
\prec_{\text{lex}}
\QQ_{(K+1):J, \,\sigma(K_s)}.
\end{equation}

In the following, we use an induction method to prove $\theta^-_{\sigma(k)} = \bar \theta^-_{\sigma(k)}$ for all $k=1,2,\ldots,K_s$ (equivalently, $\theta^-_k = \bar \theta^-_k$ for all $k=1,2,\ldots,K_s$).
First, for attribute $\sigma(1)$, define
\begin{align}\label{eq-s6}
	\ttt^\star = 
	\sum_{h=1}^K \bar \theta^-_h \ee_h 
	+ \sum_{j>K:\, q_{j, \sigma(1)} = 0}  \theta^-_j \ee_j
	+ \sum_{j>K:\, q_{j, \sigma(1)} = 1}  \theta^+_j \ee_j,
\end{align}
and $\rr^\star = \support(\ttt^\star) = \sum_{j=1}^J \ee_j$.
Then
\begin{align}\label{eq-s6tra}
	T_{\rr^\star, \aaa}(\ttt^+ - \ttt^\star,~ \ttt^- - \ttt^\star) 
	&= 
	\prod_{h=1}^K (\theta_{h,\aaa} - \bar \theta^-_h)
	\prod_{j>K:\atop q_{j, \sigma(1)} = 0}  (\theta_{j,\aaa} -  \theta^-_j)
	\prod_{j>K:\atop q_{j, \sigma(1)} = 1}  (\theta_{j,\aaa} -  \theta^+_j),
	\\ \notag
	T_{\rr^\star, \aaa}(\bar\ttt^+ - \ttt^\star,~ \bar\ttt^- - \ttt^\star) 
	&= 
	\prod_{h=1}^K (\bar\theta_{h,\aaa} - \bar \theta^-_h)
	\prod_{j>K:\atop q_{j, \sigma(1)} = 0}  (\bar\theta_{j,\aaa} - \theta^-_j)
	\prod_{j>K:\atop q_{j, \sigma(1)} = 1}  (\bar\theta_{j,\aaa} - \theta^+_j).
\end{align}
First, we claim that the row vector $T_{\rr^\star, \bcolon}(\bar\ttt^+ - \ttt^\star,~ \bar\ttt^- - \ttt^\star)$ equals the all-zero vector. 
This is because for any $\aaa\neq\one_K$, due to the first summation term in the definition of $\ttt^\star$ in \eqref{eq-s6}, there is $\bar\theta_{h,\aaa} = \bar \theta^-_h$ for some $h\in\{1,\ldots,K\}$ and hence $T_{\rr^\star, \aaa}(\bar\ttt^+ - \ttt^\star,~ \bar\ttt^- - \ttt^\star)$ contains a zero factor $(\bar \theta^-_h - \bar \theta^-_h)$.
Moreover, for $\aaa = \one_K$ the entry $\bar T_{\rr^\star, \one_K} := T_{\rr^\star, \one_K}(\bar\ttt^+ - \ttt^\star,~ \bar\ttt^- - \ttt^\star)$ also equals zero due to the following reason.
The $\bar T_{\rr^\star, \one_K}$ will equal zero if $\bar\theta_{j,\one_K} = \theta^+_j$ for some item $j > K$ with $q_{j, \sigma(1)} = 1$.
Now recall in Step 4 we proved that $\theta^+_j = \bar\theta^+_j$ for all $j > K$. 
So to show $\bar T_{\rr^\star, \one_K} = 0$, we only need that there exists some item $j > K$ that has $q_{j, \sigma(1)} = 1$. 
This indeed is true because the attribute $\sigma(1)$ is a singleton attribute and there exists some $j > K$ that has $q_{j, \sigma(1)} = 1$ under Condition B$^\star$.
Now we have shown that $T_{\rr^\star, \bcolon} (\bar\ttt^+ - \ttt^\star,~ \bar\ttt^- - \ttt^\star) = \zero$ and hence $T_{\rr^\star, \bcolon} (\bar\ttt^+ - \ttt^\star,~ \bar\ttt^- - \ttt^\star) \bar\pp = T_{\rr^\star, \bcolon} (\ttt^+ - \ttt^\star,~ \ttt^- - \ttt^\star) \pp = 0$.
Next consider the element $T_{\rr^\star, \aaa} (\ttt^+ - \ttt^\star,~ \ttt^- - \ttt^\star)$.
\textit{First}, note that for $\aaa = \one_K$ the entry $T_{\rr^\star, \one_K}(\ttt^+ - \ttt^\star,~ \ttt^- - \ttt^\star)$ also equals zero due to the similar reason as stated above for $T_{\rr^\star, \one_K}(\bar\ttt^+ - \ttt^\star,~ \bar\ttt^- - \ttt^\star)$.
\textit{Second}, as assumed in \eqref{eq-qlex-gen}, for the attribute $\sigma(1)$, the vector $\QQ_{(K+1):J, \, \sigma(1)}$ has the smallest lexicographic order among the $K_s$ columns of $\QQ_{(K+1):J, \; 1:K_s}$.
As a result,
\begin{equation}\label{eq-column}
\bo f^{\sigma(1)} := 
	\bigvee_{j>K:\atop q_{j, \sigma(1)} = 0}\, \bq_j = 
(\underbrace{1,\ldots,1, 
\overbrace{0,}^{\text{column } \sigma(1)} 
1,\, \ldots, \,1}_{\text{columns } 1,\, \ldots, \,K_s},
~
\underbrace{* ,~\ldots,~ *}_{\text{columns }K_s + 1,\, \ldots,\, K}),
\end{equation}
%%%%%%%%%%%%%%%%%%%%%
where the ``$*$'' above denotes unspecified values each of which can be either one or zero.
The above observation ensures the entry $T_{\rr^\star, \aaa}(\ttt^+ - \ttt^\star,~ \ttt^- - \ttt^\star)$ would equal zero if $\aaa$ lacks any singleton attribute other than attribute $\sigma(1)$. 
This is because $T_{\rr^\star, \aaa}(\ttt^+ - \ttt^\star,~ \ttt^- - \ttt^\star) \neq 0$ only if $\theta_{j, \aaa} \neq \theta^-_j$ and $\theta_{j, \aaa} = \theta^+_j$ for all $j>K$ with $q_{j, \sigma(1)}=0$, which holds only if $\aaa\succeq \bo f^{\sigma(1)}$ defined in \eqref{eq-column}.
Such $\aaa$ hence must possess all the singleton attributes other than the $\sigma(1)$th one.
Namely, we have shown that $T_{\rr^\star, \aaa}(\ttt^+ - \ttt^\star,~ \ttt^- - \ttt^\star) \neq 0$ only if $\aaa$ coincides with $\bo f^{\sigma(1)}$ in \eqref{eq-column} in the first $K_s$ entries.
We continue to show that $T_{\rr^\star, \aaa}(\ttt^+ - \ttt^\star,~ \ttt^- - \ttt^\star)$ will also equal zero if $\aaa$ lacks any non-singleton attribute.
The reason for this is that in Step 1 and Step 2 we have already shown $\theta^-_h = \bar\theta^-_h$ for item $h$ if attribute $h$ is an ancestor, intermediate, or leaf attribute; namely $\theta^-_h = \bar\theta^-_h$ already holds for any non-singleton  $h\in\{K_s+1,\ldots,K\}$.
Therefore, if $\aaa$ lacks any non-singleton attribute $h\in\{K_s+1,\ldots,K\}$, then $\theta_{h,\aaa} = \theta^-_h = \bar\theta^-_h$. As a result, the $T_{\rr^\star, \aaa}(\ttt^+ - \ttt^\star,~ \ttt^- - \ttt^\star)$  in \eqref{eq-s6tra} will contain a zero factor $(\bar\theta^-_h - \bar\theta^-_h)$ and hence must be zero.
Thus far we have shown that $T_{\rr^\star, \aaa}(\ttt^+ - \ttt^\star,~ \ttt^- - \ttt^\star)$ can potentially be nonzero only if $\aaa = \one_K - \ee_{\sigma(1)}$. Note that this particular $\aaa = \one_K - \ee_{\sigma(1)}$ indeed respects the attribute hierarchy $\mce$ and hence belongs to $\mca(\mce)$, because attribute $\sigma(1)$ is a singleton attribute. This means $p_{\one_K - \ee_{\sigma(1)}} > 0$.
Now $T_{\rr^\star, \bcolon} (\bar\ttt^+ - \ttt^\star,~ \bar\ttt^- - \ttt^\star) \bar\pp = T_{\rr^\star, \bcolon} (\ttt^+ - \ttt^\star,~ \ttt^- - \ttt^\star) \pp = 0$ indicates the following equality
\begin{align*}
0 = (\theta^-_{\sigma(1)} - \bar \theta^-_{\sigma(1)})
	\prod_{h\leq K:\atop h\neq \sigma(1)} (\theta^+_h - \bar \theta^-_h)
	\prod_{j>K:\atop q_{j, \sigma(1)} = 0}  (\theta^+_j -  \theta^-_j)
	\prod_{j>K:\atop q_{j, \sigma(1)} = 1}  (\theta^-_j -  \theta^+_j) p_{\one_K - \ee_{\sigma(1)}}.
\end{align*}
Since all the other factors above other than the first one are nonzero, we obtain $\theta^-_{\sigma(1)} = \bar \theta^-_{\sigma(1)}$.
This completes the first step of the induction method of proof.

Now as the inductive hypothesis, suppose $\theta_{\sigma(\ell)} = \bar\theta_{\sigma(\ell)}$ holds for $\ell = 1,\ldots, m-1$, where $m \leq K_s$.
We next show $\theta_{\sigma(m)} = \bar\theta_{\sigma(m)}$. Define
\begin{align}\label{eq-s62}
\ttt^\star = 
	\sum_{h=1}^K \bar \theta^-_h \ee_h 
	+ \sum_{j>K:\, q_{j, \sigma(m)} = 0}  \theta^-_j \ee_j
	+ \sum_{j>K:\, q_{j, \sigma(m)} = 1}  \theta^+_j \ee_m,
\end{align}
and $\rr^\star = \support(\ttt^\star) = \sum_{j=1}^J \ee_j$.
Then similarly as before,
\begin{align}\label{eq-s6tra21}
	T_{\rr^\star, \aaa}(\ttt^+ - \ttt^\star,~ \ttt^- - \ttt^\star) 
	&= 
	\prod_{h=1}^K (\theta_{h,\aaa} - \bar \theta^-_h)
	\prod_{j>K:\atop q_{j, \sigma(m)} = 0}  (\theta_{j,\aaa} -  \theta^-_j)
	\prod_{j>K:\atop q_{j, \sigma(m)} = 1}  (\theta_{j,\aaa} -  \theta^+_j),
	\\ \label{eq-s6tra22}
	T_{\rr^\star, \aaa}(\bar\ttt^+ - \ttt^\star,~ \bar\ttt^- - \ttt^\star) 
	&= 
	\prod_{h=1}^K (\bar\theta_{h,\aaa} - \bar \theta^-_h)
	\prod_{j>K:\atop q_{j, \sigma(m)} = 0}  (\bar\theta_{j,\aaa} - \theta^-_j)
	\prod_{j>K:\atop q_{j, \sigma(m)} = 1}  (\bar\theta_{j,\aaa} - \theta^+_j).
\end{align}
Still, we claim that the row vector $T_{\rr^\star, \bcolon}(\bar\ttt^+ - \ttt^\star,~ \bar\ttt^- - \ttt^\star)$ equals the all-zero vector.  
This is because for any $\aaa\neq\one_K$, due to the first summation term in the definition of $\ttt^\star$ in \eqref{eq-s62}, $T_{\rr^\star, \aaa}(\bar\ttt^+ - \ttt^\star,~ \bar\ttt^- - \ttt^\star)$ contains a zero factor $(\bar \theta^-_h - \bar \theta^-_h)$.
Moreover, for $\aaa = \one_K$ the entry $T_{\rr^\star, \one_K}(\bar\ttt^+ - \ttt^\star,~ \bar\ttt^- - \ttt^\star)$ in \eqref{eq-s6tra22} also equals zero due to the following reason.
The entry in \eqref{eq-s6tra22} will equal zero if $\bar\theta_{j,\one_K} = \theta^+_j$ for some item $j > K$ with $q_{j, \sigma(m)} = 1$.
In Step 4 we proved that $\theta^+_j = \bar\theta^+_j$ for all $j > K$. 
To show $T_{\rr^\star, \one_K}(\bar\ttt^+ - \ttt^\star,~ \bar\ttt^- - \ttt^\star)=0$, we only need that there exists some item $j > K$ that has $q_{j, \sigma(m)} = 1$; for such $j$ there would be $\bar\theta_{j,\one_K} = \bar\theta^+_j = \theta^+_j$. 
This indeed is true because the attribute $\sigma(m)$ is a singleton attribute and there exists some $j > K$ that has $q_{j, \sigma(m)} = 1$ under Condition B$^\star$.
Now we have shown that $T_{\rr^\star, \bcolon} (\bar\ttt^+ - \ttt^\star,~ \bar\ttt^- - \ttt^\star) = \zero$ and hence $T_{\rr^\star, \bcolon} (\bar\ttt^+ - \ttt^\star,~ \bar\ttt^- - \ttt^\star) \bar\pp = T_{\rr^\star, \bcolon} (\ttt^+ - \ttt^\star,~ \ttt^- - \ttt^\star) \pp = 0$.
%%%
Next consider the element $T_{\rr^\star, \aaa} (\ttt^+ - \ttt^\star,~ \ttt^- - \ttt^\star)$ in \eqref{eq-s6tra21}.
\textit{First}, for $\aaa = \one_K$ the entry $T_{\rr^\star, \one_K}(\ttt^+ - \ttt^\star,~ \ttt^- - \ttt^\star)$ also equals zero due to the similar reason as stated above for $T_{\rr^\star, \one_K}(\bar\ttt^+ - \ttt^\star,~ \bar\ttt^- - \ttt^\star)=0$.
\textit{Second}, as assumed in \eqref{eq-qlex-gen}, for the attribute $\sigma(m)$, the vector $\QQ_{(K+1):J, \, \sigma(m)}$ has the smallest lexicographic order among the following $K_s - (m-1)$ column vectors: $\QQ_{(K+1):J, \, \sigma(m)}$, $\QQ_{(K+1):J, \, \sigma(m+1)}$, $\ldots$, $\QQ_{(K+1):J, \, \sigma(K_s)}$.
Next we consider the following binary vector $\bo f^{\sigma(m)}$ defined similarly as the previous $\bo f^{\sigma(1)}$,
\begin{equation}\label{eq-column-m}
\bo f^{\sigma(m)} := 
	\bigvee_{j>K:\atop q_{j, \sigma(m)} = 0}\, \bq_j = 
(\underbrace{*,\ldots,*, 
\overbrace{0,}^{\text{column } \sigma(m)} 
*,\, \ldots, \,*}_{\text{columns } 1,\, \ldots, \,K_s},
~
\underbrace{* ,~\ldots,~ *}_{\text{columns }K_s + 1,\, \ldots,\, K}).
\end{equation}
An important observation is that for any $\ell=m+1,\ldots,K_s$, the $\sigma(\ell)$th entry of the vector $\bo f^{\sigma(m)} = (f^{\sigma(m)}_1, \ldots, f^{\sigma(m)}_K)$ equals one; namely $f^{\sigma(m)}_{\sigma(\ell)} = 1$ for $m+1 \leq \ell < K_s$.
%where the ``$*$'' above denotes unspecified values each of which can be either one or zero.
This is due to the assumption of the increasing lexicographic order among the vectors $\QQ_{(K+1):J, \, \sigma(m)}$, $\QQ_{(K+1):J, \, \sigma(m+1)}$, $\ldots$, $\QQ_{(K+1):J, \, \sigma(K_s)}$ previously specified in \eqref{eq-qlex-gen}.
Based on this observation, the entry $T_{\rr^\star, \aaa}(\ttt^+ - \ttt^\star,~ \ttt^- - \ttt^\star)$ will equal zero if $\aaa$ lacks any attribute $\sigma(\ell)$ for $\ell=m+1,\ldots,K_s$; because in that case, $\aaa\nsucceq \bo f^{\sigma(m)}$ so $\theta_{j,\aaa} = \theta^-_j$ would hold for some item $j$ with $q_{j,\sigma(m)} = 0$ and $T_{\rr^\star, \aaa}(\ttt^+ - \ttt^\star,~ \ttt^- - \ttt^\star)$ would contain a zero factor $(\theta^-_j - \theta^-_j)$.
Next consider any attribute pattern $\aaa$ lacks some attribute $\sigma(\ell)$ for $\ell=1,\ldots,m-1$. For such $\aaa$, according to our inductive hypothesis, for the corresponding item $\sigma(\ell)$ there is $\theta_{\sigma(\ell), \aaa} = \theta_{\sigma(\ell), \aaa} = \bar \theta^-_{\sigma(\ell), \aaa}$. 
This means for such $\aaa$ the entry $T_{\rr^\star, \aaa}(\ttt^+ - \ttt^\star,~ \ttt^- - \ttt^\star)$ contains a zero factor $(\bar \theta^-_{\sigma(\ell), \aaa} - \bar \theta^-_{\sigma(\ell), \aaa})$ and hence must be zero.
To summarize, now we have shown that $T_{\rr^\star, \aaa}(\ttt^+ - \ttt^\star,~ \ttt^- - \ttt^\star)$ is potentially nonzero only if $\aaa\neq\one_K$ and $\aaa$ possess all the singleton attributes other than attribute $\sigma(\ell)$.

We continue to show that $T(\ttt^+ - \ttt^\star,~ \ttt^- - \ttt^\star)$ will also equal zero if $\aaa$ lacks any non-singleton attribute.
The reason is that Step 1 and Step 2 have already shown $\theta^-_h = \bar\theta^-_h$ for item $h$ if attribute $h$ is an ancestor, intermediate, or leaf attribute; namely $\theta^-_h = \bar\theta^-_h$ already holds for any non-singleton  $h\in\{K_s+1,\ldots,K\}$.
Therefore, if $\aaa$ lacks any non-singleton attribute $h\in\{K_s+1,\ldots,K\}$, then $\theta_{h,\aaa} = \theta^-_h = \bar\theta^-_h$. As a result, the $T_{\rr^\star, \aaa}(\ttt^+ - \ttt^\star,~ \ttt^- - \ttt^\star)$  in \eqref{eq-s6tra22} will contain a zero factor $(\bar\theta^-_h - \bar\theta^-_h)$ and hence must be zero.
Thus far we have shown that $T_{\rr^\star, \aaa}(\ttt^+ - \ttt^\star,~ \ttt^- - \ttt^\star)$ can potentially be nonzero only if $\aaa = \one_K - \ee_{\sigma(m)}$. This particular $\aaa = \one_K - \ee_{\sigma(m)}$ indeed respects the attribute hierarchy $\mce$ and hence belongs to $\mca(\mce)$, because attribute $\sigma(m)$ is a singleton attribute. This means $p_{\one_K - \ee_{\sigma(m)}} > 0$.
Now $T_{\rr^\star, \bcolon} (\bar\ttt^+ - \ttt^\star,~ \bar\ttt^- - \ttt^\star) \bar\pp = T_{\rr^\star, \bcolon} (\ttt^+ - \ttt^\star,~ \ttt^- - \ttt^\star) \pp = 0$ indicates the following equality
\begin{align*}
0 = (\theta^-_{\sigma(m)} - \bar \theta^-_{\sigma(m)})
	\prod_{h\leq K:\atop h\neq \sigma(m)} (\theta^+_h - \bar \theta^-_h)
	\prod_{j>K:\atop q_{j, \sigma(m)} = 0}  (\theta^+_j -  \theta^-_j)
	\prod_{j>K:\atop q_{j, \sigma(m)} = 1}  (\theta^-_j -  \theta^+_j) p_{\one_K - \ee_{\sigma(m)}}.
\end{align*}
Since all the other factors above other than the first one are nonzero, we obtain $\theta^-_{\sigma(m)} = \bar \theta^-_{\sigma(m)}$ for the currently considered attribute $\sigma(m)$.
Therefore, using the induction method we have shown that $\theta^-_{\sigma(k)} = \bar \theta^-_{\sigma(k)}$ holds for every $k\in\{1,\ldots,K_s\}$; equivalently, $\theta^-_{k} = \bar \theta^-_{k}$ holds for every $k\in\{1,\ldots,K_s\}$. This completes the proof of Step 6.

%%%%%%%%%%%%%%%%%%%%%%%%%
\medskip
\noindent\textbf{Step 7.~}
In this step we prove $\theta^+_k = \bar \theta^+_k$ if the $k$th attribute is a singleton attribute.
By Condition B$^\star$, there exist two different items $j_1, j_2>K$ such that $q^{\sparse}_{j_1,k} = q^{\sparse}_{j_2,k} =1$ (so $q_{j_1,k} = q_{j_2,k} =1$). Define
\begin{equation}
    \label{eq-step7}
    \ttt^\star = \sum_{h\leq K:\, h\neq k}\bar \theta^-_h\ee_h
    + \bar \theta^-_{j_1}\ee_{j_1}
    + \theta^-_{j_2}\ee_{j_2},
\end{equation}
then $T_{\rr^\star,\aaa}(\bar\ttt^+ - \ttt^\star, ~ \bar\ttt^- - \ttt^\star) \neq 0$ if and only if $\aaa=\one_K$ due to the first two terms in the above definition of $\ttt^\star$. This is because considering the second term $\bar\theta^-_{j_1}\ee_{j_1}$, there is $T_{\rr^\star,\aaa}(\bar\ttt^+ - \ttt^\star, ~ \bar\ttt^- - \ttt^\star) \neq 0$ only if $\aaa \succeq \bq_{j_1}$ and $q_{j_1, k} = 1$.
Since in the previous steps we have already shown $\bar \theta^-_h = \theta^-_h$ for any $h=1,\ldots,K$,
there also is $T_{\rr^\star,\aaa}(\ttt^+ - \ttt^\star, ~\ttt^- - \ttt^\star) \neq 0$ if and only if $\aaa=\one_K$.
This is because considering the term $\theta^-_{j_2}\ee_{j_2}$ in the definition of $\ttt^\star$ in \eqref{eq-step7}, we have $T_{\rr^\star,\aaa}(\ttt^+ - \ttt^\star, ~ \ttt^- - \ttt^\star) \neq 0$ only if $\aaa \succeq \bq_{j_2}$ and $q_{j_2, k} = 1$.
Specifically, 
$$
T_{\rr^\star,\one_K}(\ttt^+ - \ttt^\star, ~ \ttt^- - \ttt^\star) = \prod_{h\leq K:\, h\neq k}(\theta^+_h-\bar \theta^-_h)
(\theta^+_{j_1} - \bar \theta^-_{j_1})
(\theta^+_{j_2} - \theta^-_{j_2}) \neq 0,
$$
and there is $T_{\rr^\star,\bcolon}(\bar\ttt^+ - \ttt^\star, ~ \bar\ttt^- - \ttt^\star) \bar\pp  = T_{\rr^\star,\bcolon}(\ttt^+ - \ttt^\star, ~ \ttt^- - \ttt^\star)\pp = T_{\rr^\star,\one_K}p_{\one_K}\neq 0$.
Therefore,
\begin{equation*}
    \bar \theta^+_k = 
    \frac{ T_{\rr^\star + \ee_k,\bcolon}(\bar\ttt^+ - \ttt^\star, ~ \bar\ttt^- - \ttt^\star) \bar\pp}{ T_{\rr^\star,\bcolon}(\bar\ttt^+ - \ttt^\star, ~ \bar\ttt^- - \ttt^\star) \bar\pp}
    =
    \frac{ T_{\rr^\star + \ee_k,\bcolon}(\ttt^+ - \ttt^\star, ~ \ttt^- - \ttt^\star) \pp}{ T_{\rr^\star,\bcolon}(\ttt^+ - \ttt^\star, ~ \ttt^- - \ttt^\star) \pp}
    =\theta^+_k.
\end{equation*}

%%%%%%%%%%%%%%%%%%%%%%%%%
\medskip
\noindent\textbf{Step 8.~}
In this step we prove $\theta^-_j = \bar\theta^-_j$ for all $j\in\{K+1,\ldots,J\}$.
This step proceeds similarly as Step 5 in the proof of Corollary \ref{thm-connected}.
For notational simplicity, we next write the row vector $\bar {\mathbf T}_{\rr^\star,\bcolon}(\bar\ttt^+ - \ttt^\star, ~\bar\ttt^- - \ttt^\star)$  as $\bar {\mathbf T}_{\rr^\star,\bcolon}$, and ${\mathbf T}_{\rr^\star,\bcolon}(\ttt^+ - \ttt^\star, ~\ttt^- - \ttt^\star)$  as ${\mathbf T}_{\rr^\star,\bcolon}$.
Define $\ttt^\star = \sum_{k=1}^K \theta_k^+ \ee_k$ and $\rr^\star = \sum_{k=1}^K \ee_k$.
Note that $\bar{\mathbf T}_{\rr^\star,\bcolon}\bar \pp = {\mathbf T}_{\rr^\star,\bcolon} \pp = \prod_{k=1}^K (\theta_k^- - \theta_k^+) p_{\zero_K} \neq 0$, so there is
\begin{align*}
	\frac{ \bar\theta_j^- \prod_{k=1}^K (\theta_k^- - \theta_k^+) \bar p_{\zero_K}}
	     {\prod_{k=1}^K (\theta_k^- - \theta_k^+) \bar p_{\zero_K}}
	     =
	 \frac{\bar{\mathbf T}_{\rr^\star + \ee_j,\bcolon}\bar \pp}{\bar{\mathbf T}_{\rr^\star,\bcolon}\bar \pp}
     =
     \frac{{\mathbf T}_{\rr^\star + \ee_j,\bcolon}\bar \pp}{{\mathbf T}_{\rr^\star,\bcolon}\bar \pp}
     =
     \frac{ \theta_j^- \prod_{k=1}^K (\theta_k^- - \theta_k^+) p_{\zero_K}}
	     {\prod_{k=1}^K (\theta_k^- - \theta_k^+) p_{\zero_K}},
\end{align*}
which gives $\bar\theta_j^- = \theta_j^-$.

%%%%%%%%%%%%%%%%%%%%%%%%%
\medskip
\noindent\textbf{Step 9.~}
Recalling the roadmap of the proof presented in Table \ref{tab-proof}, the previous eight steps have already shown the identifiability of the item parameters $\ttt^+$ and $\ttt^-$. 
Now the conclusion of Lemma \ref{lem-prop} directly applies and we also obtain the identifiability of the proportion parameters $\pp$ and the attribute hierarchy $\mce$.
This completes the proof of Theorem \ref{thm-general}.
\end{proof}

%%%%%%%%%
\color{black}

%%%%%%%%%%%%%%%%%%%%%%%%%%%%%%%%%%

\section{Proofs of Results in Section 5}\label{pf-multiple}

\begin{proof}[Proof of Proposition \ref{prop-dino}]
We first prove part (a) of the proposition.
Under an arbitrary attribute hierarchy $\mce$, by the definition of the set of allowable attribute patterns $\mca(\mce)$, for any $\aaa\in\{0,1\}^K$, $\aaa\in\mca(\mce)$ if and only if the following statement S1 holds, 
\begin{itemize}
	\item[S1.] \textit{if $k\to\ell$ under $\mce$, then $\alpha_k = 0$ implies $\alpha_\ell = 0$ for $\aaa$.}
\end{itemize}
Next consider another attribute pattern $\aaa' = \one_K - \aaa$, we have $\aaa' \in \mca(\mce^{\rev})$ holds if and only if the following statement S2 holds,
\begin{itemize}
	\item[S2.] \textit{if $k\to\ell$ under $\mce^{\rev}$, then $\alpha'_k = 0$ implies $\alpha'_\ell = 0$ for $\aaa'$,}
\end{itemize}
which is equivalent to S3 below because of $\aaa' = \one_K - \aaa$,
\begin{itemize}
	\item[S3.] \textit{if $k\to\ell$ under $\mce^{\rev}$, then $\alpha_k = 1$ implies $\alpha_\ell = 1$ for $\aaa$.}
\end{itemize}
Next, since each $\alpha_k$ can only take two possible values, zero or one, the statement S3 is equivalent to S4 below,
\begin{itemize}
	\item[S4.] \textit{if $k\to\ell$ under $\mce^{\rev}$, then $\alpha_\ell = 0$ implies $\alpha_k = 0$ for $\aaa$.}
\end{itemize}
Finally, according to the definition of the reversed hierarchy $\mce^{\rev}$ in \eqref{eq-reverse} in the proposition, there is $k\to\ell$ under $\mce^{\rev}$ if and only if $\ell\to k$ under $\mce$. So statement S4 is further equivalent to S5 below,
\begin{itemize}
	\item[S5.] \textit{if $\ell\to k$ under $\mce$, then $\alpha_\ell = 0$ implies $\alpha_k = 0$ for $\aaa$.}
\end{itemize}
This statement S5 holds if and only if $\aaa \in\mca(\mce)$. 
Now that all the above four statements S2, S3, S4, and S5 are equivalent, we obtain that $\aaa' = \one_K - \aaa \in \mca(\mce^{\rev})$ holds if and only if $\aaa \in\mca(\mce)$. 
This proves part (a) of the proposition.

\medskip
We next prove part (b) of the proposition.
Based on the relationship that $\Gamma^{\myor}_{\bq_j, \aaa} = 1 - \Gamma^{\myand}_{\bq_j, \one_K - \aaa}$, for the same set of per-item Bernoulli parameters $\ttt^+$ and $\ttt^-$, there is
	\begin{align}\notag
		&~ \mathbb P(\RR = \rr \mid \text{DINO}, \QQ,\mce,\ttt^+,\ttt^-, \pp)
		\\ \notag
		= &~
		\sum_{\aaa\in\mca(\mce)} p_{\aaa} 
		 \prod_{j=1}^J 
	     [\Gamma^{\myor}_{\bq_j,\aaa}\theta_{j}^+ + 
	     (1-\Gamma^{\myor}_{\bq_j,\aaa})\theta_{j}^-]^{r_j}
	    \\ \notag
	    & \qquad \qquad \qquad  
	    \times 
	     [\Gamma^{\myor}_{\bq_j,\aaa}(1-\theta_{j}^+) +
	     (1-\Gamma^{\myor}_{\bq_j,\aaa})(1-\theta_{j}^-)]^{1-r_j}
	    \\ \notag
	    %%%
		= &~
		\sum_{\aaa\in\mca(\mce)} p_{\aaa} 
		 \prod_{j=1}^J 
	     [(1-\Gamma^{\myand}_{\bq_j,\one-\aaa})\theta_{j}^+ + 
	     \Gamma^{\myand}_{\bq_j,\one-\aaa}\theta_{j}^-]^{r_j} 
	    \\ \notag
	    &\qquad\qquad\qquad  \times 
	     [(1-\Gamma^{\myand}_{\bq_j,\one-\aaa})(1-\theta_{j}^+) +
	     \Gamma^{\myand}_{\bq_j,\one-\aaa}(1-\theta_{j}^-)]^{1-r_j}
	    \\ \notag
	    %%%
	    &~ (\text{let }\aaa' = \one-\aaa)
	    \\ \notag
		= &~
		\sum_{\one-\aaa' \in \mca(\mce)} p_{\one-\aaa'} 
		 \prod_{j=1}^J 
	     [(1-\Gamma^{\myand}_{\bq_j,\aaa'})\theta_{j}^+ + 
	     \Gamma^{\myand}_{\bq_j,\aaa'}\theta_{j}^-]^{r_j} 
	    \\ \notag
	    &\qquad\qquad\qquad 
	    \times 
	     [(1-\Gamma^{\myand}_{\bq_j,\aaa'})(1-\theta_{j}^+) +
	     \Gamma^{\myand}_{\bq_j,\aaa'}(1-\theta_{j}^-)]^{1-r_j} 
	    \\ \label{eq-dinatodino}
	    %%%
		\stackrel{(\star)}{=} &~
		\sum_{\aaa' \in \mca(\mce^{\rev})} p_{\one - \aaa'} 
		 \prod_{j=1}^J 
	     [(1-\Gamma^{\myand}_{\bq_j,\aaa'})\theta_{j}^+ + 
	     \Gamma^{\myand}_{\bq_j,\aaa'}\theta_{j}^-]^{r_j} 
	    \\ \notag
	    &\qquad\qquad\qquad  \times 
	     [(1-\Gamma^{\myand}_{\bq_j,\aaa'})(1-\theta_{j}^+) +
	     \Gamma^{\myand}_{\bq_j,\aaa'}(1-\theta_{j}^-)]^{1-r_j}.
	\end{align}
The last equality ``$(\star)$'' above follows from part (a) of the proposition.
We introduce a new notation $\widetilde{\pp} = (\widetilde p_{\aaa'}:\, \aaa'\in\mca(\mce^{\rev}))$ with the following definition,
\begin{equation}\label{eq-ptilde}
\widetilde p_{\aaa'} = p_{\one - \aaa'} ~\text{ for any }~ \aaa' \in \mca(\mce^{\rev}).
\end{equation}
Then the above display in \eqref{eq-dinatodino} can be further written as
\begin{align*}
    &~\mathbb P(\RR = \rr \mid \text{DINO},~ \QQ,~ \mce,~ \ttt^+,~ \ttt^-,~ \pp)
    \\ \notag
	    %%%
	= &~
	\sum_{\aaa' \in \mca(\mce^{\rev})} \widetilde p_{\aaa'} 
		 \prod_{j=1}^J 
		 [ \Gamma^{\myand}_{\bq_j,\aaa'}(1-\theta_{j}^-)
		 + (1-\Gamma^{\myand}_{\bq_j,\aaa'})(1-\theta_{j}^+)
	      ]^{1-r_j}
	\\ \notag
	& \qquad\qquad\qquad\quad \times 
	     [\Gamma^{\myand}_{\bq_j,\aaa'}\theta_{j}^- 
	     + 
	     (1-\Gamma^{\myand}_{\bq_j,\aaa'})\theta_{j}^+ ]^{r_j}
	     \\ \notag
	= &~
    \mathbb P(\RR = \one - \rr \mid \text{DINA},~ \QQ,~ \mce^{\text{new}} = \mce^{\rev},~ 
        \ttt^{+,\text{new}} = \one - \ttt^-,~
        \ttt^{-,\text{new}} = \one - \ttt^+,~ \pp^{\text{new}} = \widetilde{\pp}),	     
\end{align*}
where the equality on the last line above just follows from the definition of the DINA-based HLAM.
Based on the above equality, given an arbitrary $\QQ$-matrix, the quantities $(\mce,~ \ttt^+,~ \ttt^-,~ \pp)$ under a DINO-based HLAM are identifiable if and only if the quantities $(\mce^{\text{new}} = \mce^{\rev},~ 
        \ttt^{+,\text{new}} = \one - \ttt^-,~
        \ttt^{-,\text{new}} = \one - \ttt^+,~ \pp^{\text{new}} = \widetilde{\pp})$ under a DINA-based HLAM are identifiable.
Note that the continuous parameters $\ttt^{+,\text{new}}$, $\ttt^{-,\text{new}}$, and $\pp^{\text{new}}$ are just explicit transformations of the original parameters. 
Therefore, the original parameters $(\ttt^+,~ \ttt^-,~ \pp)$ and the original attribute hierarchy $\mce$ under a DINO-based HLAM are identifiable if and only if the new parameters and the reversed hierarchy $\mce^{\rev}$ under a DINA-based HLAM are identifiable.
This completes the proof of the proposition.
\end{proof}

%%%%% corollary 
\begin{proof}[Proof of Corollary \ref{cor-dino}]
Proposition \ref{prop-dino} implies that $\mce^{\rev}$ and $(\ttt^+, \ttt^-, \pp)$ are identifiable if the following three conditions hold: 
\begin{itemize}\setlength\itemsep{0.5em}
	\item[(a)] the $\mce^{\rev}$-densified matrix $\mc D^{\mce^{\rev}}(\QQ)$ contains a submatrix which is the reachability matrix under the reversed hierarchy $\mce^{\rev}$ (denoted by $\EE(\mce^{\rev})$); 
	\item[(b)] in the $\mce^{\rev}$-sparsified matrix $\mc S^{\mce^{\rev}}(\QQ)$,  any intermediate attribute is each measured by $\geq 1$ items, any ancestor attribute and any leaf attribute is each measured by $\geq 2$ items, and any singleton attribute is each measured by $\geq 3$ items; 
	\item[(c)] for any two singleton attributes $\alpha_k$ and $\alpha_\ell$ under the reversed hierarchy $\mce^{\rev}$, there is $\QQ_{(K+1):J,\, k} \neq \QQ_{(K+1):J,\, \ell}$ (assuming without loss of generality that $\QQ_{1:K,\, \bcolon}$ equals the $\EE(\mce^{\rev})$).
\end{itemize}
The above (a) is exactly Condition A$^\star$ in the theorem.
Note that for the same set of $K$ attributes, any ancestor attribute in $\mce$ becomes a leaf attribute in $\mce^{\rev}$,  any leaf attribute in $\mce_0$ becomes ancestor attribute in $\mce^{\rev}$, and any intermediate attribute or singleton attribute remain the same type when $\mce$ is reversed to be $\mce^{\rev}$.
It is not hard to see that the $\mce$-sparsified matrix $\mc S^{\mce}(\QQ)$ satisfies the requirement in Condition B$^\star$ if and only if the $\mce^{\rev}$-sparsified matrix $\mc S^{\mce^{\rev}}(\QQ)$ satisfies the requirement in Condition B$^\star$. So the above bullet point (b) on $\mce^{\rev}$ is equivalent to the original Condition B$^\star$ which is about $\mce$.
Finally, the above bullet point (c) is a condition on the singleton attributes under the reversed hierarchy $\mce^{\rev}$, which holds if and only if the same condition holds for all the singleton attributes under the original hierarchy $\mce$, i.e., Condition C$^\star$.
This proves the corollary.
\end{proof}

%%%%% main-effect models
\begin{proof}[Proof of Theorem \ref{thm-mult}]
Because the main-effect-based HLAMs have a different algebraic structure from the two-parameter HLAMs, we use a different proof technique to establish identifiability.
We need a useful concept, the Kruskal rank of matrix.
A matrix $\mathbf M$'s Kruskal rank is the maximal number $r$ such that every $r$ columns of it are linear independent; denote the Kruskal rank of $\mathbf M$ by $\rank_{K}(\mathbf M)$. 
Denote by ``$\odot$'' the Khatri-Rao product (i.e., the column-wise Kronecker product) of matrices.
That is, for two matrices $\mathbf A = (\bo a_1 \mid \bo a_2 \mid \cdots \mid \bo a_r)$ and $\mathbf B = (\bo b_1 \mid \bo b_2 \mid \cdots \mid \bo b_r)$ that both contain $r$ columns, there is $\mathbf A \odot \mathbf B = (\bo a_1 \otimes \bo b_1 \mid \bo a_2 \otimes \bo b_2 \mid \cdots \mid \bo a_r \otimes \bo b_r)$ which also contains $r$ columns.
The following lemma restates a useful variation of the Kruskal's theorem on three-way tensor decomposition.
More discussion on how this theorem can be invoked to show identifiability for a variety of latent variable models can be found in \cite{allman2009}.

\begin{lemma}[Kruskal's Theorem \cite{kruskal1977}]\label{lem-kruskal}
	Suppose $M_1, M_2, M_3$ are three matrices each of size $a_i\times r$, $N_1, N_2, N_3$ are three matrices each with $r$ columns, and they satisfy $\odot_{i=1}^3 M_i \cdot\one = \odot_{i=1}^3 N_i \cdot\one$. 
	If $\rank_{K}( M_1) + \rank_{K}( M_2) + \rank_{K}( M_3) \geq 2r + 2$, then there exists a $r\times r$ permutation matrix $P$ and three $r\times r$ invertible diagonal matrices $D_i$ such that $D_1 D_2 D_3 =I_r$ and $N_i = M_i D_i P$.
\end{lemma}

We also need a technical lemma established in \cite{partial} for the so-called restricted latent class models with binary responses.
The following result is adapted from Lemma A.1 in \cite{partial} into the current context of HLAMs.
Generally speaking, the following lemma is useful because an HLAM can be viewed as a restricted latent class model with the matrix $\Gamma(\QQ, \mce)$ imposing the following equality and inequality constraints
\begin{align}\label{eq-rlcm}
	\theta_{j, \aaa} = \theta_{j, \aaa'}~~\text{if}~~ \Gamma_{\bq_j, \aaa} = \Gamma_{\bq_j, \aaa'} = 1;
	\quad
	\text{and}~~ \theta_{j, \aaa} \neq \theta_{j, \aaa'} ~~\text{if}~~ \Gamma_{\bq_j, \aaa} \neq \Gamma_{\bq_j, \aaa'}.
\end{align}
For any subset of items $S \subseteq [J]$, define a $2^{|S|} \times |\mca(\mce)|$ matrix $T(\QQ, \bo\Theta_{S})$ with rows indexed by response pattern $\bo r\in\{0,1\}^{|S|}$ and columns by allowable attribute patterns $\aaa\in\mca(\mce)$.
Similar to the $T$-matrix defined in \eqref{eq-def-T} in the proof of Theorem \ref{thm-main} for the two-parameter HLAMs, the $(\rr, \aaa)$th entry of this $T(\QQ, \bo\Theta_{S})$ is defined as
%\begin{align*}
$T_{\rr, \aaa}(\QQ, \bo\Theta_{S}) = \prod_{j\in S}\theta_{j,\aaa}^{r_j}$.
%\end{align*}

\begin{lemma}[Adapted from Lemma A.1 in \citep{partial}]\label{lem-rank}
Consider a main-effect-based HLAM with structural matrix $\QQ$, attribute hierarchy $\mce$, and $J$ items with binary responses.
%Under a natural inequality constraint that $\theta_{j,\aaa} \neq \theta_{j,\aaa'}$ if $\Gamma_{\bq_j, \aaa} \neq \Gamma_{\bq_j, \aaa'}$, 
Under the equality constraint and inequality constraint in \eqref{eq-rlcm}, 
if the matrix $\Gamma(\QQ_{S,\bcolon}\, , \, \mce)$ contains distinct column vectors, then $T(\QQ, \bo\Theta_{S})$ has full column rank.
\end{lemma}

Under Condition E in the theorem, we apply Lemma \ref{lem-rank} to obtain that $\Gamma(\QQ_{S_1,\bcolon}\, , \, \mce)$ and $\Gamma(\QQ_{S_2,\bcolon}\, , \, \mce)$ each has full column rank $|\mca(\mce)|$, so $\Gamma(\QQ_{S_1,\bcolon}\, , \, \mce)$ and $\Gamma(\QQ_{S_2,\bcolon}\, , \, \mce)$ each also has Kruskal rank equal to $|\mca(\mce)|$.
Denote $S_3 = [J] \setminus(S_1\cup S_2)$.
We next show that under Condition F in the theorem, $T(\QQ, \bo\Theta_{S_3})$ has Kruskal rank at least two.
Consider arbitrary two different allowable attribute patterns $\aaa \neq \aaa' \in\mca(\mce)$, Condition F states that there exists $j\in S_3$ such that $\Gamma_{\bq_j,\aaa} \neq \Gamma_{\bq_j,\aaa'}$, which further ensures $\theta_{j,\aaa} \neq \theta_{j,\aaa'}$ under the natural inequality constraint in \eqref{eq-rlcm}.
We claim that columns $T_{\bcolon, \aaa}(\QQ, \bo\Theta_{S_3})$ and $T_{\bcolon, \aaa'}(\QQ, \bo\Theta_{S_3})$ are not equal, nor is one a scalar multiple of the other.
Such a statement, if true, would prove the earlier claim that $T(\QQ, \bo\Theta_{S_3})$ has Kruskal rank at least two.
Specifically, consider response pattern $\rr^0 = \zero_{|S_3|}$, then $T_{\rr^0, \aaa}(\QQ, \bo\Theta_{S_3}) = T_{\rr^0, \aaa'}(\QQ, \bo\Theta_{S_3}) = 1$; also consider response pattern $\rr^1 = \ee_{j}$ for the item $j$ with $\Gamma_{\bq_j,\aaa} \neq \Gamma_{\bq_j,\aaa'}$, then $T_{\rr^0, \aaa}(\QQ, \bo\Theta_{S_3}) = \theta_{j,\aaa} \neq   \theta_{j,\aaa'} = T_{\rr^0, \aaa'}(\QQ, \bo\Theta_{S_3})$.
This shows there do not exist any nonzero scalars $c_1, c_2 \neq 0$ such that $c_1\cdot T_{\bcolon, \aaa}(\QQ, \bo\Theta_{S_3}) - c_2 \cdot  T_{\bcolon, \aaa'}(\QQ, \bo\Theta_{S_3}) = 0$, therefore $\rank_K(T(\QQ, \bo\Theta_{S_3})) \geq 2$.
Now we can apply the Kruskal's theorem in Lemma \ref{lem-kruskal} to obtain that 
\begin{align*}
	T(\bar\QQ, \bar{\bo\Theta}_{S_i})
	= 
	T(\QQ, \bo\Theta_{S_i}) D_i P,\quad i=1,2,3,
\end{align*}
for diagonal matrices $D_1, D_2, D_3$ and permutation matrix $P$.
Indeed, we next show that each $D_i$ equals the identity matrix. This is because $T_{\zero, \aaa}(\QQ, \bo\Theta_{S_i}) = T_{\zero, \aaa'}(\QQ, \bo\Theta_{S_i}) = 1$ for each $\aaa$ and each $i$ always holds by the definition of the T-matrix.
Therefore, we have $T(\bar\QQ, \bar{\bo\Theta}_{S_i})$ equals $T(\QQ, \bo\Theta_{S_i})$ up to a permutation $P$ of the columns (label swapping of the allowable latent patterns).
This implies  $(\bo\Theta, \pp)$ and $\Gamma(\QQ, \mce)$  are identifiable. %and completes the proof of Theorem \ref{thm-mult}.

Next consider the case where in addition to the above three conditions, the $\QQ$ is known in part to contain an identity submatrix $I_K$.
Since the definition of the current $\Gamma(\QQ, \mce)$ is the same as that for the DINA-based HLAM, a similar argument as the proof of Proposition \ref{prop} leads to that $\mce$ is also identifiable.
\end{proof}

\color{black}

\section{Proofs of Statement in Examples 7-8 and  Technical Lemmas}\label{pf-example}

\begin{proof}[Proof of the Identifiability Statement in Examples \ref{exp-num-attr}  and \ref{exp-diff}]
	We prove that under the hierarchy $\mce=\{2\to 3\to 4\}$ among $K=4$ attributes, the following $\QQ$-matrix gives an identifiable model.
	\begin{align}\label{eq-diff2}
		\QQ=\begin{pmatrix}
			1 & 0 & 0 & 0 \\
			0 & 1 & 0 & 0 \\
			0 & 1 & 1 & 0 \\
			0 & 1 & 1 & 1 \\
			\hline
			1 & 0 & 0 & 0 \\
			1 & 0 & 0 & 0 \\
			0 & 1 & 0 & 0 \\
			0 & 1 & 1 & 1 \\
		\end{pmatrix}.
	\end{align}
	We only briefly outline the proof procedures. The concrete steps follow similarly as the proof of Proposition  \ref{prop-two-e}.
	First define 
	$$\ttt^\star=\theta^-_5\ee_5 + \bar \theta^-_6 \ee_6,$$ 
	then 
	$$
	\frac{\bar T_{\rr^\star+\ee_1}(\bar\ttt^+-\ttt^\star,\bar\ttt^--\ttt^\star)\bar\pp}{\bar T_{\rr^\star}(\bar \ttt^+-\ttt^\star,\bar \ttt^--\ttt^\star)\bar\pp} = 
	\frac{T_{\rr^\star+\ee_1}(\ttt^+-\ttt^\star,\ttt^--\ttt^\star)\pp}{ T_{\rr^\star}(\ttt^+-\ttt^\star,\ttt^--\ttt^\star)\pp}
	$$
	and it yields $\theta^+_1=\bar \theta^+_1$. By symmetry, we also obtain $\theta^+_5=\bar \theta^+_5$ and $\theta^+_6=\bar \theta^+_6$.
	
	Next define $$\ttt^\star=\bar\theta_2^- + \theta_7^+ \ee_7,$$ then $\bar T_{\rr^\star,\aaa}=0$ for all $\aaa\in\{0,1\}^K$. Therefore 
	$$\bar T_{\rr^\star,\cdot}\bar\pp =  T_{\rr^\star,\cdot}\pp = 0 = \left(\sum_{\aaa\in\mca(\mce),\, \alpha_2=0} p_{\aaa}\right)(\theta_2^- - \bar \theta_2^-)(\theta_7^- - \theta_7^+),
	$$
	which gives $\theta_2^- = \bar \theta_2^-$. By symmetry we also obtain $\theta_7^- = \bar \theta_7^-$.
	
	Next define $$\ttt^\star=\theta^+_1\ee_1+\theta^-_4\ee_4+\bar \theta^-_8\ee_8,$$ then 
	$$\frac{\bar T_{\rr^\star+\ee_2}(\bar\ttt^+-\ttt^\star,\bar\ttt^--\ttt^\star)\bar\pp}{\bar T_{\rr^\star}(\bar \ttt^+-\ttt^\star,\bar \ttt^--\ttt^\star)\bar\pp} = 
	\frac{T_{\rr^\star+\ee_2}(\ttt^+-\ttt^\star,\ttt^--\ttt^\star)\pp}{ T_{\rr^\star}(\ttt^+-\ttt^\star,\ttt^--\ttt^\star)\pp}$$
	gives $\theta^+_2=\bar \theta^+_2$; similarly $\theta^+_3=\bar \theta^+_3$ and $\theta^+_7=\bar \theta^+_7$.
	
	Next for $j=3$, 4, or 8, define $$\ttt^\star=\bar \theta^-_2\ee_2 + \theta^+_j\ee_7,$$ then Eq.~$\eqref{eq-algebra}=0$ and it yields $\theta^-_3=\bar \theta^-_3$, $\theta^-_4=\bar \theta^-_4$ and $\theta^-_8=\bar \theta^-_8$. 
	Define $\ttt^\star=\theta^-_4\ee_4$ (or $\ttt^\star=\theta^-_8\ee_8$) gives $\theta^+_4=\bar \theta^+_4$ (or $\theta^+_8=\bar \theta^+_8$).
	
	Next define $$\ttt^\star=\theta^+_5\ee_5 + \theta^+_6\ee_6,$$ then 
	$$\frac{\bar T_{\rr^\star+\ee_j}\bar\pp}{\bar T_{\rr^\star}\bar\pp} = \frac{
	T_{\rr^\star+\ee_j}\pp} {T_{\rr^\star}\pp},$$
	 for $j=1,5,6$, gives $\bar \theta^-_1=\theta^-_1$,  $\bar \theta^-_5=\theta^-_5$, and $\bar \theta^-_6=\theta^-_6$. 
	Now that we have obtained $\ttt^+=\bar \ttt^+$ and $\ttt^-=\bar\ttt^-$.
	Finally, proceeding similarly as that in the end of Proposition  \ref{prop-two-e} gives $\bar\pp=\pp$ and $\bar\mce=\mce$. 
	Identifiability is established and the proof is complete.
\end{proof}

\begin{proof}[{\bf Proof of Lemma \ref{lem-basic}}]
	%The proof of statements (a) and (b) are straightforward and hence is omitted. We next prove part (c) of the lemma. 
	
	For part (a), %we call the type of modification of $\QQ$ described in Condition $A$ by ``Operation'' $A$, which sets every $q_{j,k}$ to zero if $q_{j,h}=1$ and $k\to h$. Denote the resulting matrix by $\QQ^A$. 
	if there exists some $\bq_k\succeq\bq_h$ for some $k\not\to h$, then the sparsifying operation would not set $q_{k,h}$ to zero, and the first rows of $\mathcal S^{\mce}(\QQ)$ would not be an $I_K$. So  $\bq_k\succeq\bq_h$ happens only if $k\to h$.
	For part (b), %the $\QQ^B$ satisfies that if $k\to h$, then 
	    if $k\to h$, then  under the densifying operation there is $\QQ^{\star,C}_{\bcolon,k}\succeq \QQ^{\star,C}_{\bcolon,h}$. Since Condition $C$ states that $\QQ^{\star,C}$ has distinct columns, there must be  $\QQ^{\star,C}_{\bcolon,k}\succ \QQ^{\star,C}_{\bcolon,h}$.
\end{proof}

\begin{proof}[{\bf Proof of Lemma \ref{lem-order}}]
We use proof by  contradiction. Assume there exists attribute $h\in[K]$ and a set of attributes $\mathcal J\subseteq[K]\setminus\{h\}$, such that $\vee_{j\in\mathcal J}\,\bar\bq_j\nsucceq\bar\bq_h$; and that there exists  $S\subseteq\{K+1,\ldots,J\}$ such that $\max_{m\in S}q^{\dense}_{m,h}=0$ and $\max_{m\in S} q^{\dense}_{m,j}=1$. 
Define
$$
\ttt^\star = \bar \theta^+_h\ee_{h} + \sum_{j\in\mt J}\bar \theta^-_j\ee_{j} 
+ \sum_{m=K+1}^J \theta^-_{m}\ee_m,\quad \rr^\star = \ee_{h} + \sum_{j\in\mt J}\ee_{j} + \sum_{m=K+1}^J \ee_m,
$$
and we claim that $T_{\rr^\star,\bcolon}(\bar \QQ,\bar\ttt^+-\ttt^\star,\bar\ttt^--\ttt^\star)$ is an all-zero vector. This is because for any $\aaa\in\{0,1\}^K$, the corresponding element in $T_{\rr^\star,\aaa}(\bar \QQ,\bar\ttt^+-\ttt^\star,\bar\ttt^--\ttt^\star)$ contains a factor $F_{\aaa} = (\bar \theta_{h,\aaa} - \bar \theta^+_{h})\prod_{j\in\mt J}(\bar \theta_{j,\aaa} - \bar \theta^-_{j})$. While this factor $F_{\aaa}\neq 0$ only if $\bar\theta_{h,\aaa}=\bar \theta^-_h$ and $\bar \theta_{j,\aaa} = \bar \theta^+_{j}$ for all $j\in\mt J$, which happens if and only if $\aaa\nsucceq\bar \bq_{h}$ and $\aaa\succeq\bar \bq_{j}$  for all $j\in\mt J$, which is impossible because $\vee_{j\in\mt J}\bar \bq_{j}\succeq\bar  \bq_{h}$ by our assumption. So the claim $T_{\rr^\star,\bcolon}(\bar \QQ,\bar\ttt^+-\ttt^\star,\bar\ttt^--\ttt^\star)=\zero$ is proved, and further $T_{\rr^\star,\bcolon}(\bar \QQ,\bar\ttt^+-\ttt^\star,\bar\ttt^--\ttt^\star)\bar\pp=0$. %Equation \eqref{eq-def-T} and Lemma \ref{lem} gives 
Equality \eqref{eq-tra} becomes
$$
T_{\rr^\star,\bcolon}( Q,\ttt^+-\ttt^\star,\ttt^--\ttt^\star)\bar\pp = T_{\rr^\star,\bcolon}(\bar\QQ,\bar\ttt^+-\ttt^\star,\bar\ttt^--\ttt^\star)\bar\pp=0,$$
which leads to 
\begin{align*}%\label{eq-t11}
0=T_{\rr^\star,\bcolon}( \QQ,\ttt^+-\ttt^\star,\ttt^--\ttt^\star)\pp
=
%\Big(\sum_{\aaa:\alpha_{j}=0,\atop \alpha_{h}=1}p_{\aaa}\Big) 
p_{\one}(\theta^+_{h} - \bar \theta^+_{h})\prod_{j\in\mt J} (\theta^+_{j} - \bar \theta^-_{j})\prod_{m>K}(\theta^+_m-\theta^-_m),
\end{align*}
which is because for any $\aaa\neq\one$, we must have $\aaa\nsucceq\bq_m$ for some $m>K$ under Condition $B$, and hence the element $T_{\rr^\star,\aaa}( \QQ,\ttt^+-\ttt^\star,\ttt^--\ttt^\star)$ contains a factor $(\theta^-_m - \theta^-_m)=0$. 
Since $\theta^+_m-\theta^-_m>0$ for $m>K$ and $\theta^+_{j} - \bar \theta^-_{j}\neq 0$, we obtain $\theta^+_{h} = \bar \theta^+_{h}$. 

We remark here that $\theta^+_{h} = \bar \theta^+_{h}$ also implies $\bar\bq_h\neq\zero$, because otherwise we would have $\bar\theta_h=\bar \theta^+_h=\theta^+_h$, which contradicts the $\theta^-_h<\bar\theta_h<\theta^+_h$ proved before the current Step 1. 
This indicates the $\bar\QQ_{1:K,\bcolon}$ can not contain any all-zero row vector, because otherwise $\bar\bq_j\succeq\bar\bq_h$ for the all-zero row vector $\bar\bq_h$, which we showed is impossible.

Consider the item set $S$ in the lemma that satisfies $S\subseteq\{K+1,\ldots,J\}$ such that $\max_{m\in S}q^{\dense}_{m,h}=0$ and $\max_{m\in S}q^{\dense}_{m,j}=1$ for all $j\in\mt J$. Define
$$
\ttt^\star = \bar \theta^+_h\ee_{h} + \sum_{j\in\mt J}\bar \theta^-_j\ee_{j} +\sum_{m\in S} \theta^-_m\ee_{m}.
$$
Note that $\theta^+_h = \bar \theta^+_h$.
The RHS of \eqref{eq-tra} is zero, and so is the LHS of it. The row vector $T_{\rr^\star,\bcolon}(\QQ,\ttt^+-\ttt^\star,\ttt^--\ttt^\star)$ has the following property
\begin{align*}
&T_{\rr^\star,\aaa}(\QQ,\ttt^+-\ttt^\star,\ttt^--\ttt^\star)\\
=
&\begin{cases}
(\theta^-_{h} - \bar \theta^+_{h})
\underset{j\in\mt J} \prod(\theta^+_j-\bar \theta^-_j)
\underset{m\in S} \prod(\theta^+_m-\theta^-_m), 
& \aaa\nsucceq\bq_h,\,\aaa\succeq \bq_{\mt J},\, \aaa\succeq\bq_S;\\
0, & \text{otherwise}.
\end{cases}
\end{align*}
Note that 
\begin{align*}
	&~\{\aaa\in\{0,1\}^K:\,\aaa\nsucceq \bq^{\dense}_h,\,\aaa\succeq \bq^{\dense}_{\mathcal J}, \,\aaa\succeq\bq^{\dense}_S\}\\
	=&~\{\aaa:\,\aaa\nsucceq\bq^{\dense}_h,\,\aaa\succeq\bq^{\dense}_S\}
	=\mt A_1 \neq\varnothing,
\end{align*}
because $q_{S,\ell}=0$ and $q_{S,k}=1$ hold. Furthermore, we claim that $\sum_{\aaa\in\ma_1} p_{\aaa}>0$ under the specified attribute hierarchy. 
This is because Lemma \ref{lem-basic} ensures $\bq^{\dense}_m\in\ma$ for the considered $m>K$, and hence the attribute pattern $\aaa^\star=\bq^{\dense}_m$ belongs to the set $\ma_1$ and also belongs to the set $\ma$. This ensures $p_{\aaa^\star}>0$ and $\sum_{\aaa\in\ma_1} p_{\aaa}\geq p_{\aaa^\star}>0$.
Therefore we have
\begin{align*}
&~T_{\rr^\star,\bcolon}(\QQ,\ttt^+-\ttt^\star,\ttt^--\ttt^\star)\pp\\
=
&~(\theta^-_{\ell} - \bar \theta^+_{\ell})(\theta^+_k-\bar \theta^-_k)(\theta^+_m-\theta^-_m)\left(\sum_{\aaa\in\mt A_1}p_{\aaa} \right)=0,
\end{align*}
which leads to a contradiction since $\theta^-_\ell- \bar \theta^+_\ell\neq 0$, $\theta^+_k-\bar \theta^-_k\neq 0$, $\theta^+_m-\theta^-_m\neq 0$ and $\sum_{\aaa\in\mt A_1}p_{\aaa}>0$, i.e., every factor in the above product is nonzero. This completes the proof of Lemma \ref{lem-order}.
\end{proof} % end of proof of Lemma 2

% Lemma 3
\begin{proof}[{\bf Proof of Lemma \ref{lem-ch}}]
Define
$$
\ttt^\star = \sum_{h\in\mathcal K}\bar \theta^-_h\ee_h + \bar \theta^+_m\ee_m + \sum_{l>K:\, l\not\in \mathcal K\cup \{m\}}\theta^-_l\ee_l,
$$
then $ \bar  T_{\rr^\star,\aaa}$ contains a factor $\bar f_{\aaa}:=\prod_{h\in\mathcal K}(\bar\theta_{h,\aaa}-\bar \theta^-_h)(\bar \theta_{m,\aaa} - \bar \theta^+_m)$ because of the first two terms in the above display. The $\bar f_{\aaa}\neq 0$ only if $\aaa\succeq\vee_{h\in\mathcal K}\,\bar \bq_h$ and $\aaa\nsucceq\bar\bq_m$.
However, since  $\vee_{h\in\mathcal K}\,\bar \bq_h\succeq \bar\bq_m$, such $\aaa$ does not exist and $\bar f_{\aaa}=0$ for all $\aaa\in\{0,1\}^K$.
Therefore $ \bar  T_{\rr^\star,\bcolon}=\zero$ and $ \bar  T_{\rr^\star,\bcolon}\bar\pp=0$, so the RHS of \eqref{eq-tra} is zero.
Hence the LHS of \eqref{eq-tra} is also zero.
Condition $B$ implies $\sum_{j=K+1}^J q_{j,k}\geq 2$ for all attribute $k$.
Under Condition $B$ and the condition $\big|(\mathcal K\cup \{m\})\cap \{K+1,\ldots,J\}\big|\leq 1$, the attributes required by the items in the set $\{l>K: l\not\in \mathcal K\cup \{m\}\}$ must cover all the $K$ attributes.
because of the term $\sum_{l>K:\atop l\not\in \mathcal K\cup \{m\}} \theta^-_l\ee_l$ in the defined $\ttt^\star$, we have $T_{\rr^\star,\aaa}\neq 0$ only if $\aaa=\one_K$. So
\begin{align*}
    0=&~\text{RHS of }\eqref{eq-tra}=\text{LHS of }\eqref{eq-tra}\\
     =&~ \prod_{h\in\mathcal K}(\theta^+_h-\bar \theta^-_h)(\theta^+_m-\bar \theta^+_m)\prod_{l>K:\, l\not\in \mathcal K\cup \{m\}} (\theta^+_l-\theta^-_l)p_{\one_K},
\end{align*}
which implies $\theta^+_m-\bar \theta^+_m=0$ since any other factor in the above display is nonzero. This completes the proof of the lemma.
\end{proof}

\end{document}